\documentclass[11pt]{article}

\usepackage[margin=1in]{geometry}
\usepackage{amsmath,amssymb,amsthm,mathtools}
\usepackage{bm}
\usepackage{bbm}
\usepackage{booktabs}
\usepackage{microtype}
\usepackage{graphicx}
\usepackage{subcaption}
\usepackage{enumitem}
\usepackage{algorithm}
\usepackage[noend]{algpseudocode}
\usepackage{xcolor}
\usepackage[colorlinks=true,linkcolor=blue,citecolor=blue,urlcolor=blue]{hyperref}
\usepackage{natbib}
\usepackage[nameinlink,capitalize]{cleveref}

\Crefname{assumption}{Assumption}{Assumptions}

\usepackage{shortcutsg}

\setlist[itemize]{leftmargin=*}
\setlist[enumerate]{leftmargin=*}

\renewcommand{\E}{\mathbb{E}}
\newcommand{\Pbb}{\PP}
\renewcommand{\R}{\mathbb{R}}
\newcommand{\Var}{\mathrm{Var}}
\newcommand{\norm}[1]{\left\lVert #1 \right\rVert}
\newcommand{\cX}{\mathcal{X}}

\newcommand{\AUC}{\mathrm{AUC}}
\newcommand{\Errhatg}{\Delta L(\hat g)}

\newcommand{\piref}{\pi_{\rm ref}}
\newcommand{\pis}{\pi^\star}
\newcommand{\gs}{g^\star}
\newcommand{\tst}{t^\star}
\newcommand{\Psis}{\Psi^\star}

\newcommand{\thetas}{\theta^\star}

\newtheorem{theorem}{Theorem}
\newtheorem{proposition}{Proposition}
\newtheorem{lemma}{Lemma}
\theoremstyle{definition}
\newtheorem{assumption}{Assumption}
\newtheorem{definition}{Definition}
\newtheorem{remark}{Remark}

\title{
{\bf Semiparametric Preference Optimization:}\\
Your Language Model is Secretly a \emph{Single-Index Model}}
\author{Nathan Kallus\\\small Netflix \& Cornell University}
\date{}

\begin{document}
\maketitle

\begin{abstract}
Policy alignment to preference data typically assumes a known link function between observed preferences and latent rewards (e.g., Bradley-Terry model / logistic link). Misspecification of this link can bias inferred rewards and misalign learned policies. We study policy alignment under an unknown and unrestricted link function. We formulate an $f$-divergence-constrained reward maximization problem and show that realizability in a policy class induces a semiparametric single-index binary choice model, where a scalar policy-induced index captures all dependence on demonstrations and the remaining preference distribution is unrestricted. Rather than impose identifiability of structural parameters of such a model and estimate them, as in econometrics, we develop methods that directly learn policies, with the reward function implicit, analyzing error to the optimal policy and allowing for unidentifiable and nonparametric indices. We prove link-agnostic convergence guarantees in terms of generic function complexity measures and validate the methods and theory empirically. Code is available at \url{https://github.com/causalml/spo/}.
\end{abstract}

\section{Introduction}
\label{sec:intro}
Modern large language models (LLMs) are tuned using human or AI feedback (RLHF/RLAIF) to better align outputs with user preferences and/or safety desiderata \citep{Christiano2017,Ziegler2019,Stiennon2020,Ouyang2022,Bai2022,bai2022training,nakano2021webgpt,wu2021recursively,Rafailov2023,zhao2023slic,yuan2023rrhf,gheshlaghiAzar2024general,hong2024orpo,meng2024simpo}. A common setup interprets pairwise preferences as discrete choice under a latent reward, which is inferred and optimized, while constraining or penalizing the deviation from a reference model, balancing quality improvements with preservation of language and other abilities.

Linking preferences to rewards is usually done by assuming a particular choice model, such as logistic (Bradley-Terry) or probit, so that, given demonstrations (query and response) and a reward function, the distribution of preferences is fully specified \citep{Rafailov2023,zhan2023provable,glaese2022improving,Ziegler2019,ibarz2018reward}.
A prominent example is Direct Preference Optimization (DPO) \citep{Rafailov2023}, which uses the Bradley-Terry choice model.
This, however, imposes a lot of structure on choice behavior; recent work has studied both the consequences of misspecifying the link, preference model, or reference policy and methods robust to such misspecification \citep{hong2025robustness,xu2024strong,sun2024rethinking,golz2025distortion,chidambaram2025direct,zhang2025provableunknownlink,xu2025doublyrobust}.
Alternative approaches to alignment depart altogether from a structural/generative modeling of observed choice by using prospect-theoretic losses, robustifying against corrupted labels, or optimizing a game-theoretic equilibrium \citep{ethayarajh2024kto,liang2024ropo,kong2024perplexity,munos2024nash}.

Specifying a known reward-preference link is, however, not actually necessary for a structural/generative interpretation, meaning one that assumes preferences optimize a random utility or equivalently that preferences are generated by some conditional probability distribution given demonstrations and reward. Letting the link be arbitrary and unknown gives rise to \emph{semiparametric} discrete choice models, which have been studied extensively in the econometrics literature \citep{McFadden1974,Cosslett1983,KleinSpady1993,Manski1975,Manski1985,Horowitz1992,KimPollard1990,Han1987,Sherman1993}. In that literature, however, the focus is typically structural estimation of finite-dimensional utility parameters under stringent conditions that make those parameters identifiable.

In this paper, we bring this classical econometric idea to bear on the modern RLHF problem by shifting the target from structural estimation to policy learning. We study policy alignment from preference data generated by a completely unknown link function and unknown reward function, where the target is the optimal reward-maximizing divergence-constrained policy. A key insight is that mere realizability of this optimal policy in our policy class implies a semiparametric \emph{single-index model} for preferences, with an arbitrary unknown link function. The policy class induces a class of indices, which we let be an arbitrary nonparametric class, and the true index (corresponding to the true reward) is generally not identifiable. Methodologically, this requires extending econometric semiparametric ideas from low-dimensional structural models to policy classes rich enough to include neural policies and transformer LLMs, and developing guarantees that depend on generic functional complexity rather than parametric dimension. We develop the Semiparametric Preference Optimization (SPO) approach that optimizes the policy by fitting its implied index and develop and analyze various instantiations.

\subsubsection*{Summary of contributions}
\begin{itemize}
    \item \textbf{From econometric binary choice to RLHF policy learning.} We show that divergence-constrained preference alignment naturally induces a semiparametric single-index binary choice model, connecting RLHF to classical econometric approaches while shifting the target from structural estimation to policy learning.
    \item \textbf{Semiparametric policy learners.} We develop Profiled SPO (PSPO), Orthogonalized SPO (OSPO), and Ranking SPO (RSPO) as three ways to optimize over policies without specifying the link or explicitly fitting a reward model.
    \item \textbf{Policy-learning theory.} We establish policy convergence guarantees for our methods under generic functional complexity conditions that admit flexible nonparametric policy and reward-index classes.
    \item \textbf{Empirical validation.} We use both synthetic and Qwen3/UltraFeedback experiments as proof-of-concept demonstrations of the proposed learners, illustrating the robustness predicted by the theory under link misspecification.
\end{itemize}

\subsubsection*{Roadmap}
\emph{Section~\ref{sec:reduction}} formalizes the divergence-constrained policy target, the preference-data model, and the resulting semiparametric single-index reduction, including the policy induced by an index and the identification condition used throughout. \emph{Sections~\ref{sec:pspo}--\ref{sec:rspo}} develop the three SPO learners: PSPO via profiled likelihood and consistency, OSPO via an orthogonalized quasi-likelihood with plug-in regression and rates, and RSPO via bipartite ranking and rates. \emph{Section~\ref{sec:empirical}} presents experiments both synthetic and using Qwen3/UltraFeedback illustrating link robustness. \emph{Section~\ref{sec:optimization}} gives first-order implementations of the learners. \emph{Section~\ref{sec:empirical-calib}} analyzes empirical recalibration to the divergence budget. \emph{Section~\ref{sec:extensions}} discusses extensions around localized OSPO, max-score estimators, and approximate realizability. \emph{Section~\ref{sec:related}} surveys related work, and \emph{Section~\ref{sec:conclusion}} concludes. Full proofs and additional experimental details appear in \cref{sec:proofs,sec:additional_detail_on_experiments}, respectively.

\section{Preference Alignment as a Semiparametric Single-Index Model}
\label{sec:reduction}
This section formalizes the policy objective and preference data, derives the single-index model, and defines an index-error metric that controls error of the index-induced divergence-calibrated policy. 

\subsection{The policy optimization target}
\label{sec:target}

We target the policy maximizing average reward subject to an $f$-divergence constraint from a reference policy. This captures alignment while staying close to a pre-trained model.

We let $x\in\Xcal$ denote context and $y\in\Ycal$ action, with unknown mean reward $r^\star(x,y)$. For simplicity we assume $|\Ycal|<\infty$. Given convex $f:\RR^+\to\RR$ with $f(1)=0$, define $D_f(p\|q)=\infty$ if $\mathrm{supp}(p)\nsubseteq \mathrm{supp}(q)$ and otherwise $D_f(p\|q)=\sum_{y:q(y)>0} q(y)\, f(p(y)/q(y))$. We assume $f$ is twice continuously differentiable, strictly convex, and $f'(0^+)=-\infty$; KL corresponds to $f(u)=u\log u$.

Given a reference policy $\piref:\Xcal\to\Delta^{\Ycal}$, we are interested in the divergence-constrained reward-maximizing policy: with expectations taken with respect to a context distribution $x\sim P_x$, define
\begin{equation}
\label{eq:target}
\pis\in\argmax_{\pi:\Xcal\to\Delta^{\Ycal}}\E_x
\sum_{y\in\Ycal}\pi(y\mid x) r^\star(x,y)
\quad\text{s.t.}\quad\E_x D_f(\pi(\cdot\mid x)\|\piref(\cdot\mid x))\leq\kappa,
\end{equation}

A convexity argument and the results of \citet{wang2023beyond} together yield a closed form for $\pis$. \Cref{asmp:div} below rules out the trivial case where a pure reward-maximizing policy already satisfies the divergence budget.

\begin{assumption}\label{asmp:div}
Let $\omega(x)=\sum_{y\in\argmax_y r^\star(x,y)}\piref(y\mid x)$. Suppose $\E_x\omega(x)f(1/\omega(x))>\kappa$.
\end{assumption}
\begin{theorem}[Closed form for $\pis$]
\label{thm:closed-form}
Under \cref{asmp:div},
there exist $\beta^\star>0$ and $\lambda^\star:\cX\to\R$ such that
\[
\pi^\star(y\mid x)=\pi_{\rm ref}(y\mid x)\,(f')^{-1}\!\left(\beta^{\star-1}\big(r^\star(x,y)-\lambda^\star(x)\big)\right),\quad \E_x D_f(\pis(\cdot\mid x)\|\piref(\cdot\mid x))=\kappa.
\]
\end{theorem}

\begin{remark}[The case of KL] The solution simplifies a lot for $f(u)=u\log u$, in which case $\pi^\star(y\mid x)\propto \piref(y\mid x)\exp(\beta^{\star-1}r^\star(x,y))$ for some $\beta^\star$.\end{remark}

\subsection{Preference data and their distribution}
\label{sec:data}

To optimize \cref{eq:target}, we observe iid $(w_i,z_i)\sim P$ with $w_i=(x_i,y_{i0},y_{i1})$ and $z_i\in\{0,1\}$ indicating whether $y_{i1}\succ y_{i0}$. 
We assume the $x$-marginal of $P$ matches $P_x$ (otherwise we may use importance weighting) %
and 
that preferences are generated according to
\begin{equation}
\label{eq:pairwise}
z\mid x,y_0,y_1 \sim \mathrm{Bern}\prns{\Phi^\star\prns{r^\star(x,y_1)-r^\star(x,y_0)}},
\end{equation}
where $\Phi^\star$ is an unknown cumulative distribution function (CDF). We treat $\Phi^\star$ as completely \textit{unknown} (rather than fixing it to logistic/probit), since our goal is policy optimization rather than fully modeling preference noise. For simplicity, we assume $P$ is symmetric in $(y_0,y_1)$ (\eg, shuffle before scoring) so that $\E[h(x,y_1)-h(x,y_0)]=0$ for any function $h(x,y)$.

\begin{remark}[Random utility interpretation] If $\epsilon_0,\epsilon_1$ are independent of $x,y_0,y_1$ with $\PP(\epsilon_0-\epsilon_1\leq t)=\Phi^\star(t)$, then choosing $z=\argmax_{i=0,1}\{r^\star(x,y_i)+\epsilon_i\}$ yields \cref{eq:pairwise}. The shocks $\epsilon_0,\epsilon_1$ capture idiosyncratic noise in individuals' valuations of responses around a common mean $r^\star(x,y)$.
\end{remark}

\subsection{Reformulation as a single-index model}
\label{sec:single-index}

We search over a policy class $\{\pi_\theta:\theta\in\Theta\}$, allowing $\theta$ to be an abstract (possibly infinite-dimensional) parameter. We will now show that assuming realizability of this class for our learning task in \cref{eq:target} induces what is known as a \textit{single-index model} \citep{ichimura1987estimation} on the preference data.

Let us posit the following model for preference data: for $\theta\in\Theta$ and $\Psi$ a CDF,
\begin{align}
z\mid x,y_0,y_1 &\sim \mathrm{Bern}\prns{\Psi\prns{t_\theta(x,y_0,y_1)}},\label{eq:sim}
\\\notag%
t_\theta(x,y_0,y_1)&=h_\theta(x,y_1)-h_\theta(x,y_0),~~h_\theta(x,y)=f'\prns{{\pi_\theta(y\mid x)}/{\pi_{\rm ref}(y\mid x)}}.
\end{align}

Combining \Cref{thm:closed-form} with \cref{eq:pairwise} and rearranging shows that, under policy realizability of $\pis$, this model is correctly specified, i.e., it contains the true data-generating process (DGP) for some $\theta,\Psi$.
\begin{assumption}[Realizability]\label{asmp:realizability}$\pi^\star=\pi_{\theta^\star}$ for $\theta^\star\in\Theta$.\end{assumption}
\begin{proposition}[Policy realizability implies a single-index model]\label{prop:sim}
Under \Cref{asmp:div,asmp:realizability} we have that \cref{eq:sim} holds for $\theta=\theta^\star$, $\Psi=\Psis$, where we set $\Psi^\star(u)=\Phi^\star(\beta^\star u)$ with $\beta^\star$ as in \cref{thm:closed-form}.
\end{proposition}
Specifically, \cref{thm:closed-form} can be rearranged as $r^\star(x,y)=\beta^\star h_{\theta^\star}(x,y)+\lambda^\star(x)$. Then $\lambda^\star$ is canceled by differencing the rewards of $y_1$ and $y_0$, and $\beta^\star$ is swallowed into the CDF, as it is unrestricted. The implications of the failure of \cref{asmp:realizability} and of \emph{approximate} realizability are discussed in \cref{sec:approx-realizability}.

\Cref{eq:sim} is a single-index model because it posits that all dependence on $(x,y_0,y_1)$ flows through the scalar index $t_\theta(w)$. The remaining dependence is a completely unrestricted nonparametric on this univariate quantity via $\Psi$. We refer to $h_\theta$ as a potential and $t_\theta$ as an index.
We henceforth impose \Cref{asmp:div,asmp:realizability} so that, by \cref{prop:sim}, $\theta^\star$, $\tst$, and $\Psis$ are well defined and characterize the DGP.

Since $\theta$ is just an abstract parameter, the class of indices it induces,
\[
\Tcal=\braces{
f'\prns{\pi_\theta(y_1\mid x)/\pi_{\rm ref}(y_1\mid x)}
-
f'\prns{\pi_\theta(y_0\mid x)/\pi_{\rm ref}(y_0\mid x)}
:\theta\in\Theta},
\]
may be parametric or nonparametric and is as rich as the policy class $\{\pi_\theta:\theta\in\Theta\}$. As such, we will characterize $\Tcal$ in terms of its functional complexity, rather than some finite parametric dimension.

\subsection{From an index to a policy}
\label{sec:policy-from-index}

While policies are the ones parameterizing the index, via the parameterization by $\theta$, not every policy is feasible in the divergence constraint. Policies are invariant to additive transforms of rewards but not scaling, as that determines the temperature or divergence. Nonetheless, once we fix the divergence constraint we obtain scaling invariance as well.

Given $\theta$, which gives a potential function $h_\theta$, define the corresponding $\beta$-regularized policy
$$
\pi_{\beta,\theta}(y\mid x)=\pi_{\rm ref}(y\mid x)\,(f')^{-1}\!\left(\beta^{-1}\big(h_\theta(x,y)-\lambda_{\beta,\theta}(x)\big)\right),
$$
where $\lambda_{\beta,\theta}(x)$ is chosen so that $\pi_{\beta,\theta}(y\mid x)$ is a valid conditional probability mass function (sums to 1). Overloading notation let us further define the population-divergence-calibrated policy as
$$
\pi_{\kappa,\theta}(y\mid x)\;=\;
\pi_{\beta_{\kappa,\theta},\theta}(y\mid x),\quad
\text{where $\beta_{\kappa,\theta}>0$ is s.t. $\E_x D_f(\pi_{\kappa,\theta}(\cdot\mid x)\|\piref(\cdot\mid x))=\kappa$}.$$ 
Under \cref{asmp:div,asmp:realizability}, $\pis=\pi_{\kappa,\theta^\star}$. 
We will later use an empirical version of this calibration (\cref{sec:empirical-calib}) and show it only incurs an additive error, so we can focus on the convergence of $\pi_{\kappa,\hat\theta}$ for a learned $\hat\theta$. 

Since $\pi_{\kappa,\theta}$ is unchanged when $h_\theta$ undergoes an $x$-dependent shift and a global positive scaling, we only need to learn $h_{\theta^\star}$ up to this equivalence to learn good policies:
\begin{definition}[Scale- and $x$-location-invariant metric]
\label{def:rho}
Define
\[
\rho(h)\;=\;\inf_{a>0} \bigl(\E_{x}\inf_{b(x)}\sum_{y}\pi_{\rm ref}(y\mid x)\,\big(a\,h(x,y)-b(x)-h_{\theta^\star}(x,y)\big)^2\bigr)^{1/2}.
\]
\end{definition}

Clearly, if $\rho(h_\theta)=0$ then $\pi_{\kappa,\theta}=\pi_{\kappa,\thetas}$ as we just need to scale $\lambda_{\kappa,\thetas}(x),\beta_{\kappa,\thetas}$, and $\pi_{\kappa,\thetas}=\pis$ under \cref{asmp:div,asmp:realizability}. We would further like nonzero but small $\rho(h_\theta)$ to translate to small policy error. Before establishing this, we collect some of the basic regularity conditions we will use throughout.

\begin{definition}[Norms]\label{def:norms}
Throughout, for $w$-functions $f(w)=f(x,y_0,y_1)$ we define $\|f\|_p$ as the $L_p$ norm over $w\sim P$, and for $(x,y)$-functions $g(x,y)$ we define $\|g\|_p$ as $L_p$ norm over $P_x\times\piref$. For $p=2$ we omit the subscript.
\end{definition}

\begin{assumption}[Boundedness, smoothness, and coverage]
\label{asmp:basic}
$M=\sup_{\theta\in\Theta}\|t_\theta\|_\infty<\infty$, $\sup_{\theta\in\Theta}\|\log f''(\pi_\theta(y\mid x)/\piref(y\mid x))\|_\infty<\infty$, and $\operatorname{esssup}_{\substack{x\sim P_x\\y_0,y_1\sim\piref}}\frac{\pi_{\rm ref}(y_0\mid x)\pi_{\rm ref}(y_1\mid x)}{P(y_0,y_1\mid x)}<\infty$.
\end{assumption}

The first two conditions ensure boundedness and smoothness; the last is a coverage condition allowing us to translate fit under $P$ to small $\rho$ defined under $\piref$. It is trivially satisfied when demonstrations are generated by drawing twice independently from $\piref$. 
That is, provided we are satisfied with regularizing divergence to $\piref$ as defined in \cref{eq:target} (which forces the optimal solution to be covered by $\piref$), there is no concern regarding coverage as long as our data is \emph{also} generated by $\piref$.

\begin{theorem}[Affine-invariant index error controls policy error]\label{thm:policy-stability-metric}
Suppose \cref{asmp:div,asmp:realizability,asmp:basic} hold, that $\pi_{\rm ref}(y\mid x)$ takes values in $\{0\}\cup[\epsilon,1]$ for $\epsilon>0$, and that, for each $\theta\in\Theta$ and $x\in\mathcal X$, define $
\omega_\theta(x)=\sum_{y\in \argmax_{y\in\mathcal Y} h_\theta(x,y)}\pi_{\mathrm{ref}}(y\mid x),
$
and assume $\mathbb E_x\Bigl[\omega_\theta(x)\,f\!\bigl(1/\omega_\theta(x)\bigr)\Bigr]>\kappa$ for every $\theta\in\Theta$.
Then there exists a constant $C<\infty$
such that
\[
\|\pi_{\kappa,\theta}-\pis\|_1\ \le\ C\,\rho(h_\theta)
\qquad\text{for all $\theta\in\Theta$}.
\]
\end{theorem}

Thus, bounding $\rho(h_{\hat\theta})$ bounds policy error for $\pi_{\kappa,\hat\theta}$. Henceforth we focus on bounding $\rho(h_{\hat\theta})$.

\subsection{Identification of the target policy}
\label{sec:aip-id}

Our policy target is invariant to affine transforms of $h_\theta$, while the single-index model is invariant to monotone transforms of $t_\theta$. To connect fitting preferences to learning policies, we assume that observational equivalence of indices implies policy equivalence.

\begin{assumption}[Observational equivalence implies policy equivalence]
\label{asm:aip-id}
For any $\theta\in\Theta$,
if there exists a monotonic non-decreasing $m:\R\to\R$ with
\[
t_{\theta^\star}(x;y_1,y_0)=m\!\left(t_{\theta}(x;y_1,y_0)\right)\quad \text{$P$-a.s.},
\]
then there exist $a\geq0$ and measurable $b:\cX\to\R$ such that
\[
h_{\theta^\star}(x,y)=a\,h_{\theta}(x,y)+b(x)\qquad \text{$P_x\times\piref$-a.s.}
\]
 \end{assumption}
In fact, it holds naturally when $|\Ycal|\ge3$ and generations are sufficiently rich because the additivity of indices imposes Cauchy's functional equation on $m$ above (see \cref{app:aip-proof}).

\begin{theorem}\label{thm:aip-id-sufficient}
Suppose \cref{asmp:basic} holds and there exists $\delta>0$ such that for every $\theta\in\Theta$, $\abs u<\|t_{\theta}\|_\infty$, $\abs v<\delta$ with $\abs{u+v}<\|t_{\theta}\|_\infty$, there exist $x,y_1,y_2,y_3$ with $(x,y_1,y_2),(x,y_2,y_3),(x,y_1,y_3)\in \mathrm{supp}(P)$ and $u=t_{\theta}(x;y_1,y_2)$, $v=t_{\theta}(x;y_2,y_3)$. Then \cref{asm:aip-id} holds.
\end{theorem}

\section{Semiparametric Policy Optimization by Profiling the Link (PSPO)}
\label{sec:pspo}

We now consider policy learners inspired by semiparametric estimation of single-index models. We begin by adapting the ideas of \citet{Cosslett1983}, which \emph{profiles} the link by maximizing over monotone $\Psi$ for each candidate $\theta$. We then establish existence and (rate-free) policy consistency.

\subsection{The PSPO policy learner}
\label{sec:pspo-estimator}
Define the binary cross-entropy given both an index $t$ and a candidate link function $\Psi$:
\[
\ell(t,\Psi;w,z)= z\log \Psi\big(t(w)\big)+(1-z)\log\big(1-\Psi\big(t(w)\big)\big).
\]
Recall that $w=(x,y_{0},y_{1})$.

If we knew $\Psi^\star$ we would plug it in and optimize the empirical average of the above likelihood over $t_\theta$; this is the approach of DPO using the logistic link function. Instead, we optimize over both $t_\theta$ and monotone $\Psi$, profiling out the link:
\begin{equation}
\label{eq:pspo-obj}
L_n^{\rm PSPO}(\theta)=\sup_{\Psi\uparrow}\;
\frac{1}{n}\sum_{i=1}^n\ell(t_\theta,\Psi;w_i,z_i)
,\qquad
\hat\theta_{\rm PSPO}\in\argmax_{\theta\in\Theta} L_n^{\rm PSPO}(\theta),
\end{equation}
where $\Psi\uparrow$ denotes non-decreasing functions $\R\to[0,1]$.

Similar to DPO, the approach here is to align the policy by maximizing the likelihood of the implicit reward function implied by the aligned policy. Recall, the policy is wrapped inside $t_\theta(w)=f'\prns{\frac{\pi_\theta(y_1\mid x)}{\pi_{\rm ref}(y_1\mid x)}}
-
f'\prns{\frac{\pi_\theta(y_0\mid x)}{\pi_{\rm ref}(y_0\mid x)}}$. Here the likelihood is \cref{eq:pspo-obj}, which profiles the link rather than assuming it is known and plugging it in.

\subsection{Policy consistency of PSPO}
\label{sec:pspo-theory}
We now establish rate-free consistency of PSPO, allowing for a nonparametric index class $\Tcal$.

\begin{theorem}[PSPO policy consistency]
\label{thm:pspo}
Suppose \cref{asmp:div,asmp:realizability,asmp:basic,asm:aip-id} hold so that $M=\sup_{t\in\Tcal}\|t\|_\infty<\infty$ and suppose that the class $\{\mathbb I[t(w)\leq\tau]:t\in\mathcal T,~\tau\in[-M,M]\}$ is $P$-Glivenko-Cantelli. Moreover, suppose that $\Var(\E[z\mid w])\neq0$, that $\{t/\|t\|:t\in\Tcal\backslash\{0\}\}$ is $L_2$-compact, and that $t(w)$ has no atoms for each $t\in\Tcal\backslash\{0\}$.
Then any sequence of profiled maximizers $\hat\theta_{\rm PSPO}$ satisfies
\[
\rho(h_{\hat\theta_{\rm PSPO}})\to0\quad\text{in probability.}
\]
\end{theorem}

The functional complexity condition here is that thresholded indices form a $P$-Glivenko--Cantelli class (\eg, $\Tcal$ is VC-subgraph). The assumption $\Var(\E[z\mid w])\neq0$ is equivalent to $\tst\neq0$. And, the $L_2$-compactness assumption ensures we cannot take $t\to0$ while making it overly complex at small scales, since we must be invariant to scale.

\section{Semiparametric Policy Optimization by Orthogonalizing the Link (OSPO)}
\label{sec:ospo}

PSPO yields consistency but can be ill-behaved because profiling can induce a rough loss surface. Following \citet{KleinSpady1993}, rather than profile the link, we consider how perturbing $\theta$ changes the likelihood, \ie, the score. The score has two orthogonal components: the distribution of the index and the distribution of $z$ given the index. At $\thetas$, the latter is the true link $\Psi^\star$, but away from it this is some conditional probability, which we can estimate by univariate regression and plug in. Orthogonality yields quadratic sensitivity to plug-in error, akin to Double/Debiased Machine Learning \citep{chernozhukov2024applied}. 

\subsection{The idealized quasi-likelihood and OSPO policy learner with nuisance plug-in}
\label{sec:ospo-est}

Solving the score equation in $\theta$ is equivalent to maximizing a quasi-likelihood:
\begin{equation}\label{eq:ospo-ideal}
L_n^{\rm ideal}(\theta)
= \frac{1}{n}\sum_{i=1}^n \ell(t_\theta,g^\star_\theta;w_i,z_i),
\qquad
g^\star_\theta(u)=P(z=1\mid t_\theta(w)=u).
\end{equation}
At $\theta=\thetas$, $g^\star_{\thetas}=\Psi^\star$ is the true link, but for $\theta\neq\thetas$, $g^\star_\theta$ need not even be monotone.

Unfortunately, we do not know $\gs_\theta$. We therefore estimate it by $\hat g_\theta$ as a univariate regression of $z$ on $t_\theta(w)$ and optimize the plug-in quasi-likelihood
\begin{equation}
\label{eq:ospo-obj}
L_n^{\rm OSPO}(\theta)=\frac{1}{n}
\sum_{i=1}^n \ell(t_\theta,\hat g_\theta;w_i,z_i)
,\qquad
\hat\theta_{\rm OSPO}\in\argmax_{\theta\in\Theta} L_n^{\rm OSPO}(\theta).
\end{equation}

One simple choice for $\hat g_\theta$ is kernel regression on a dataset $\Dcal$:
\begin{equation}\label{eq:nwreg}
\hat g_\theta(u)=
\sum_{(w,z)\in\Dcal}K\big((u-t_\theta(w))/(h\,\hat\sigma(t_\theta(w)))\big)\,z
\Big/
\sum_{(w,z)\in\Dcal}K\big((u-t_\theta(w))/(h\,\hat\sigma(t_\theta(w)))\big),
\end{equation}
where $K$ is a kernel, $h$ is a bandwidth, and $\hat\sigma(t_\theta(w))$ is the empirical standard deviation of $t_\theta$ on $\Dcal$, scaling by which preserves the scale invariance of $\gs$. We recommend leave-one-out estimation ($\Dcal=\{(w_j,z_j):j\neq i\}$ when plugging into the $i\thh$ term). For rates, we analyze an independent $\Dcal$, since data splitting only affects constants. See \cref{sub:local_polynomial_regression_plug_in_error} for examples and analyses of $\hat g_\theta$ estimators.

\subsubsection{Identification up to scale and sign}

The ideal OSPO objective only sees the conditional mean $g_t^\star(t(w))$. Thus, if $\tst$ is a function of $t$, the ideal objective cannot distinguish the two indices. Unlike classical semiparametric choice work, we do not need to identify the underlying parameter; we need only identify the calibrated policy, so identifying the index up to scale and sign is enough. The MSE margin assumption in the main text turns near observational equivalence into a quantitative bound on distance to the best affine projection.

\begin{lemma}[Curvature up to scale and sign]\label{lem:margin}
Under \cref{asmp:div,asmp:realizability,asmp:basic,asmp:ospo-margin,asm:ospo-regularity}, 
there exists $\lambda_0>0$ such that for all $\theta\in\Theta$,
\[
L^{\rm ideal}(\tst)-L^{\rm ideal}(t_\theta)\geq \lambda_0
\inf_{f\in\Fcal_{\rm lin}}\|\tst-f(t_\theta)\|^2.
\]
\end{lemma}

This still leaves the sign ambiguous. If the optimizer returns something close to $-\tst$, the induced policy would point in the reward-minimizing direction. The next lemma shows that once an index is sufficiently close to either signed scaling of $\tst$, an independent sample can recover the sign by estimating the correlation between $t_\theta(w)$ and $z$.

\begin{lemma}[Fixing index sign]\label{lem:fix-sign} 
Suppose \cref{asmp:div,asmp:realizability,asmp:basic,asmp:ospo-margin,asm:ospo-regularity} hold and that $\tst\neq0$. 
Then $\E[t^\star(w) z]>0$ is positive.
Fix any $\gamma\in(0,1)$ and any $\theta\in\Theta$ with $t_\theta\neq0$. Let
\[
s=\E[t_\theta(w)\tst(w)],
\qquad
r=\inf_{a\in\R}\|\tst-a t_\theta\|.
\]
Suppose $r\leq\gamma\mu$.
Given a dataset $\Dcal$ of iid draws of $(w,z)$,
let $\hat s=\frac1{\abs{\Dcal}}\sum_{(w,z)\in\Dcal} t_\theta(w)z$.
Then
\[
\mathbb P\left(
\hat s\,s\;\leq\;0
\right)\leq
\frac{1}{(1-\gamma)^2\abs{\Dcal}}\,
\frac{\mathbb E[\tst(w)^2]}{\E[t^\star(w) z]^2}.
\]
\end{lemma}

\subsection{Policy error guarantees with a generic plug-in}
\label{sec:ospo-guarantees}

To obtain learning rates,
we will need a slight strengthening of our up-to-scaling identification condition (\cref{asm:aip-id}) so that \emph{near}-observational-equivalence translates to \emph{near} index alignment.
After rescaling by $\Var(t^\star(w))$, \cref{asmp:ospo-margin} compares nonparametric and linear $R^2$: if $t_\theta$ nearly explains $\tst$ nonparametrically, then it must also nearly explain it linearly (possibly with a negative slope).
To ensure we get the right sign, we do a post-hoc adjustment to ensure positive correlation with $z$.
\begin{assumption}[MSE margin]
\label{asmp:ospo-margin}
With $\Fcal_{\rm lin}=\{u\mapsto \alpha+\beta u:\alpha,\beta\in\Rl\}$,
exists $c_{\rm MSE}>0$ such that
\[
\inf_{f:\Rl\to\Rl}\|\tst-f(t_\theta)\|\geq
c_{\rm MSE}\inf_{f\in\Fcal_{\rm lin}}\|\tst-f(t_\theta)\|\quad\text{for all $\theta\in\Theta$}.
\]
\end{assumption}

We control the complexity of the index class $\Tcal$ in terms of its covering number $N_2(\epsilon,\mathcal T)$, the minimal number of $\epsilon$-radius $L_2$-balls needed to cover $\Tcal$. \cref{asm:indexentropy} covers both the nonparametric case such as H\"older classes ($p>0$) and the ``simple'' case such as VC classes ($p=0$) \citep[thm. 2.6.7]{vdVaartWellner1996}.

\begin{assumption}[Regularity]\label{asm:ospo-regularity}
$\Psis$ has a derivative bounded away from $0$ on $[-M,M]$.
$g_\theta^\star$ has continuous first and second derivatives on $[-M,M]$ that are uniformly bounded 
over $\theta\in\Theta$.
\end{assumption}

\begin{assumption}[Complexity of indices]\label{asm:indexentropy}
Exist constants $p\geq0,A>0$ such that, if $p>0$ we have $\log N_2(\epsilon,\mathcal T)\leq A\epsilon^{-p}$ and if $p=0$ we have $N_2(\epsilon,\mathcal T)\leq (A/\epsilon)^{v}$ for some constant $v>0$.
\end{assumption}

Next we need to control the plug-in $\hat g_\theta$. First we define the error terms that enter the bound. Second, we ensure the plug-in has a similar range to $\gs_\theta$.

\begin{definition}[Plug-in errors]\label{def:ospo-plugin-rate}
With $\Delta q_\theta(w)=\hat g_\theta(t_\theta(w))-g^\star_\theta(t_\theta(w))$, define
$\zeta=\sup_{\theta\in\Theta}\|\Delta q_\theta\|_4$ and $\Gamma=\sup_{\theta\in\Theta:\inf_{a\in\RR}\|\tst-at_\theta\|\leq\iota}\|\Delta q_\theta-\Delta q_{\thetas}\|_2/\inf_{a\in\RR}\|\tst-at_\theta\|$ for some $\iota>0$.
\end{definition}

Here $\zeta$ controls the size of the regression error, while $\Gamma$ controls how quickly that error changes as the index moves locally around $\tst$, using a scale invariant distance $\inf_{a\in\RR}\|\tst-a t_\theta\|$ because OSPO only identifies the index up to scale. This matches the nuisance: for any nonzero scalar $a$, $g^\star_{a t}(a u)=g^\star_t(u)$, and our plug-in estimators (such as \cref{sec:ospo-est}) are similarly scale invariant. 

\begin{assumption}[Plug-in range]\label{asm:ospo-plugin-range} 
$\frac{\inf_{u\in[-M,M]}\min\{\hat g_\theta(u),1-\hat g_\theta(u)\}}{\inf_{u\in[-M,M]}\min\{g^\star_\theta(u),1-g^\star_\theta(u)\}}\geq c_{\rm range}$ for some $c_{\rm range}>0$.
\end{assumption}

\begin{theorem}[Nonparametric rate for OSPO]
\label{thm:ospo-rate}
Suppose \cref{asmp:div,asmp:realizability,asmp:basic,asmp:ospo-margin,asm:ospo-regularity,asm:indexentropy,asm:ospo-plugin-range} hold. Let $\hat s=\frac1{\abs{\Dcal}}\sum_{(w,z)\in\Dcal} t_{\hat\theta_{\rm OSPO}}(w)z$ on independent data.
Then there exists a constant $c>0$ such that, for all $\delta\in(0,0.5)$ and all sufficiently large $n$ and $\abs{\Dcal}$, with probability at least $1-\delta$,
\begin{align*}
\rho(\hat s\,h_{\hat\theta_{\rm OSPO}})
&\leq c\prns{(1+\Gamma)\prns{\varrho_n+\sqrt{\frac{\log(1/\delta)}{n}}}
+\Gamma\zeta+\zeta^2},~~~
\varrho_n=\begin{cases}
\sqrt{{v\log(n)}/{n}}, & p=0, \\
n^{-1/(2+p)}, & 0<p<2, \\
n^{-1/4}\sqrt{\log n}, & p=2, \\
n^{-1/(2p)}, & p>2 .
\end{cases}
\end{align*}
\end{theorem}

Crucially, the dependence on the error in $\hat g_\theta$ is \textit{orthogonal}: the contribution is quadratic in its convergence rates. \Cref{thm:local-poly-plugin} gives an instantiation: with $m\asymp n$ 
a local-polynomial plug-in can give $\Gamma\zeta+\zeta^2=\tilde o(n^{-1/2})$, which is always $o(\varrho_n)$ so plug-in error has no first-order effect on OSPO rates.

\subsection{Local-polynomial regression as plug-in for OSPO}\label{sub:local_polynomial_regression_plug_in_error}

Local-polynomial regression is a direct generalization of the kernel regression plug-in in \cref{eq:nwreg}: using $0$-order polynomials (constants) recovers kernel regression, while higher orders can reduce bias when the regression functions being estimated are smoother. This lets us treat kernel regression and its smoother variants in one framework.

Let $\eta(w)=\E[z\mid w]$ and, for any bounded index $t$, define $g_t^\star(u)=\E[\eta(w)\mid t(w)=u]$. Fix an integer order $r\geq0$, let $R_r(v)=(1,v,\dots,v^r)^\top$, and set the scale-free bandwidth $b_t=h\sigma(t)$. The untruncated local-polynomial fit is
\[
\tilde a_t(u)\in\argmin_{a\in\RR^{r+1}}
\sum_{(w_i,z_i)\in\Dcal}
K\!\left(\frac{t(w_i)-u}{b_t}\right)
\left\{z_i-a^\top R_r\!\left(\frac{t(w_i)-u}{b_t}\right)\right\}^2,
\qquad
\tilde g_t(u)=e_1^\top\tilde a_t(u),
\]
where $e_1=(1,0,\dots,0)^\top$. Let $\varepsilon_g=\min\{\Psis(-M),1-\Psis(M),1/4\}$ and fix $\varepsilon_0\in(0,\varepsilon_g)$. The proposed plug-in estimator is the following clipped estimator
\[
\hat g_t(u)=\bigl(\varepsilon_0\vee \tilde g_t(u)\bigr)\wedge(1-\varepsilon_0).
\]

We next show that this estimator satisfies the conditions of \cref{thm:ospo-rate} and establish rates that are generally of negligible order in the final bound. The key condition is that $g_t^\star(u)$ is smooth enough in that it is $s$-H\"older: it has $r$ derivatives and its $r\thh$ derivative is $(s-r)$-H\"older continuous. We then set $r=\lceil s\rceil-1$ in the estimator above. We defer the specific technical conditions to \cref{sec:analysis_of_local_polynomial_regression_as_plug_in_for_ospo} as they involve standard conditions of local-polynomial regression \citep[Ch.~3]{fan1996local} with an added uniformity (because OSPO optimizes over $\theta$ and we have one local polynomial regression per $\theta$).

\begin{theorem}[Local-polynomial plug-in error and stability]\label{thm:local-poly-plugin}
Suppose Assumptions~\ref{asmp:basic}, \ref{asm:local-poly-regression}, and \ref{asm:local-poly-uniform} hold and let
$
A_m(\delta)=v\log(c m/h)+\log(1/\delta),
$
for a large enough constant $c$. There are constants $C_0,C<\infty$ and $c_0>0$ such that, if $mh^3\ge C_0 A_m(\delta)$ and $\zeta_m\leq c_0$, then with probability at least $1-\delta$, the plug-in $\hat g_\theta=\hat g_{t_\theta}$ satisfies \cref{asm:ospo-plugin-range}, and the error terms in \cref{def:ospo-plugin-rate} obey
\[
\zeta\le C\prns{h^s+\sqrt{\frac{A_m(\delta)}{mh}}+\frac{A_m(\delta)}{mh}},\qquad \Gamma\le C\prns{h^{s-1}+\sqrt{\frac{A_m(\delta)}{mh^3}}+\frac{A_m(\delta)}{mh^2}}.
\]
\end{theorem}

If $h\asymp(A_m/m)^{1/(2s+1)}$, then $\zeta_m\asymp(A_m/m)^{s/(2s+1)}$ and $\gamma_m\asymp(A_m/m)^{(s-1)/(2s+1)}$, so
\[
\Gamma\zeta+\zeta^2
\lesssim
\gamma_m\zeta_m+\zeta_m^2
\asymp
\left(\frac{A_m}{m}\right)^{(2s-1)/(2s+1)}.
\]
At the local-constant endpoint, corresponding to kernel regression with only Lipschitz smoothness ($s=1$), this calculation gives the familiar $m^{-1/3}$ contribution. With $m\asymp n$, that can be too slow to be negligible in \cref{thm:ospo-rate} when $p<1$. In contrast, any $s>3/2$ gives a nuisance contribution $o(m^{-1/2})$, hence with $m\asymp n$ it is faster than every critical radius $\varrho_n$ in \cref{thm:ospo-rate}. Also, $\gamma_m\to0$ whenever $s>1$, so the multiplier $(1+\Gamma)$ in \cref{thm:ospo-rate} stays asymptotically constant for these plug-ins.

\section{Semiparametric Policy Optimization via Bipartite Ranking (RSPO)}
\label{sec:rspo}

Both PSPO and OSPO were based on maximizing something akin to a likelihood. We now consider an alternative based on optimizing a ranking objective. \citet{Han1987} proposed a max rank correlation approach for learning an identifiable \emph{linear} index in a family of observation models that includes binary choice, in which case the estimator maximizes the area under the curve (AUC). \citet{Sherman1993} showed the estimated linear coefficients converge at a $n^{-1/2}$ rate and are asymptotically normally distributed. However, in the single-index binary choice setting, the asymptotic variance is generally \emph{larger} than the efficiency bound derived by \citet{Cosslett1983} and achieved by \citet{KleinSpady1993}, meaning the max rank correlation estimator is not efficient and does not behave asymptotically like maximum likelihood with the true link function, as if it were actually known.

Nonetheless, a ranking objective for RLHF is appealing as it involves no additional nuisances (like PSPO) while still attaining a rate (like OSPO). In this section, we study how to leverage it for our alignment problem with \emph{nonlinear} index classes implied by flexible policies, characterizing \emph{policy error} in terms of \emph{general function complexity}. To do so, we first establish an identification margin similar to \cref{lem:margin}, showing that small excess AUC translates to learning the index up to scale. Then we can leverage empirical risk minimization results of $U$-statistics \citep{clemenccon2008ranking} to establish a bound on the excess AUC for nonparametric index classes. Chaining the two together with \cref{thm:policy-stability-metric}, we obtain policy convergence.

\subsection{From AUC maximization to a policy}

Define the AUC of an index as 
\begin{align}\label{eq:auc}
\mathrm{AUC}(t)=&~\PP_{(w_0,z_0),(w_1,z_1)\sim P^2}(t(w_1)> t(w_0)\mid z_0=0,z_1=1)\\&+\frac12\PP_{(w_0,z_0),(w_1,z_1)\sim P^2}(t(w_1)=t(w_0)\mid z_0=0,z_1=1),\notag
\end{align}
where $(w_0,z_0),(w_1,z_1)$ are two independent draws from $P$. This is another way to express the area under the receiver operating characteristic curve given by thresholding $t(w)$ to classify $z$ \citep{menon2016bipartite}. 

Notice this is \emph{distinct} from averaging over single demonstration pairs $(x,y_0,y_1)$ how often the ranking between $h_\theta(x,y_1)$ and $h_\theta(x,y_0)$ matches the label $z$. This amounts to a 0-1 loss for the sign of $t_\theta(w)$ in classifying the label $z$, maximizing which corresponds to max-score estimators (discussed in \cref{sec:max_score_estimators}). 
Importantly here we average over \emph{pairs of pairs}.

In particular, taking the pairs-of-pair point of view further, notice that, since $P$ is symmetric in $y_0,y_1$, if $\Phi^\star$ were also symmetric ($\Phi^\star(u)=1-\Phi^\star(-u)$), then we can rewrite \cref{eq:auc} for our alignment-from-preferences setting \emph{without any conditioning}:
\begin{align}\label{eq:aucsym}
\mathrm{AUC}(t)=&~\PP_{(x_0,y_{00},y_{01},z_0),(x_1,y_{10},y_{11},z_1)\sim P^2}(t(x_1,y_{1,(1-z_1)},y_{1,z_1})> t(x_0,y_{0,z_0},y_{0,1-z_0}))\\&+\frac12\PP_{(x_0,y_{00},y_{01},z_0),(x_1,y_{10},y_{11},z_1)\sim P^2}(t(x_1,y_{1,(1-z_1)},y_{1,z_1})= t(x_0,y_{0,z_0},y_{0,1-z_0})).\notag
\end{align}
Essentially, to get a draw from the conditional on $z=1$, take any observation $(x,y_0,y_1,z)$ and just make sure that the winning demonstration of the two per the observed label $z$ is placed in the $y_1$ position, and vice versa for $z=0$.

We next show that, in our setting, the set of maximizers of $\mathrm{AUC}(t)$ over $t\in\Tcal$ are exactly $\tst$, up to scale.

\begin{proposition}[Maximizers of AUC are $\tst$ up to scale]\label{prop:auc-max}
$\mathrm{AUC}(\tst)\geq\mathrm{AUC}(t)$ for all indices $t:\Wcal\to\RR$, with equality if and only if $t^\star(w)=m(t(w))$ almost surely for a monotone nondecreasing function $m:\RR\to\RR$. Consequently, under \cref{asm:aip-id}, for any $t\in\Tcal$, we have equality if and only if $t^\star(w)=a t(w)$ almost surely for some $a\geq0$.
\end{proposition}

In order to obtain rates, however, we need to show that small AUC suboptimality corresponds to being \emph{nearly} equal to $\tst$ up to scale, which correspondingly controls policy error, based on our results. We will again leverage \cref{asmp:ospo-margin} to quantify near-observational equivalence. We will also need some additional regularity similar but distinct from \cref{asm:ospo-regularity}.

\begin{assumption}[Regularity]\label{asm:rspo-regularity}
$\Psis$ has a derivative bounded away from $0$ and $\infty$ on $[-M,M]$, and $\tst$ is supported on an interval, on which it has a density bounded away from 0 and $\infty$.
\end{assumption}

The latter part of \cref{asm:rspo-regularity} regarding the density of $\tst$ will be used both to establish the next result this section as well as to establish a noise (margin) condition in the next section (for which only the upper bound is needed).

\begin{theorem}[AUC suboptimality of policy-implied index controls policy error]\label{thm:auc-margin}
Suppose \cref{asmp:div,asmp:realizability,asmp:basic,asmp:ospo-margin,asm:rspo-regularity} hold. Then there exist constants $\iota>0,\lambda_0>0$ such that for every $\theta\in\Theta$ satisfying $\mathrm{AUC}(\tst)-\mathrm{AUC}(t_\theta)\leq\iota$, we have
\[
\mathrm{AUC}(\tst)-\mathrm{AUC}(t_\theta)\geq\lambda_0
\inf_{a\geq0}\|\tst-a t_\theta\|^2.
\]
\end{theorem}

Note that unlike \cref{lem:margin}, here we can restrict to \emph{nonnegative} scalings. Indeed, AUC is not invariant to flipping the sign -- that will yield the complement of the AUC. Our regularity conditions above ensure that the optimal AUC (attained by $\tst$) is strictly above $0.5$ and that the AUC is $\frac{2}{3}$-H\"older continuous in the score function. Since AUC \emph{is} invariant to positive scalings, this means that for $\iota$ small enough, the covariance of $\tst$ and $t_\theta$ cannot be negative.

\subsection{The RSPO policy learner}

Given \cref{thm:auc-margin}, to align the policy to be quite close to the optimal policy, it suffices to find a policy whose implied index function attains near-optimal AUC. A reasonable approach to doing so is to maximize the empirical AUC, giving rise to the Ranking SPO (RSPO) learner:
\begin{equation}\label{eq:rspo}
\widehat{AUC}(\theta)=\frac1{n_1n_0}\sum_{i,j:z_i=0,z_j=1}\phi(t_\theta(w_j)-t_\theta(w_i)),
\qquad
\hat\theta_{\rm RSPO}\in\argmax_{\theta\in\Theta} \widehat{AUC}(\theta).
\end{equation}
where $\phi(u)=\indic{u>0}+\tfrac12\indic{u=0}$ and $n_1=\sum_{i=1}^n z_i$, $n_0=n-n_1$.

Under the assumption that $\Phi^\star$ is symmetric, we can alternatively use an empirical version of \cref{eq:aucsym}:
\begin{equation}\label{eq:rsposym}
\widetilde{AUC}(\theta)=\frac1{n^2}\sum_{i,j}\phi(t_\theta(x_j,y_{j,1-z_j},y_{j,z_j})-t_\theta(x_i,y_{i,z_i},y_{i,1-z_i})),
\qquad
\tilde\theta_{\rm RSPO}\in\argmax_{\theta\in\Theta} \widetilde{AUC}(\theta).
\end{equation}
Note that while \cref{eq:auc} and \cref{eq:aucsym} are exactly equal under the assumption that $\Phi^\star$ is symmetric, their sample versions (\cref{eq:rspo} and \cref{eq:rsposym}, respectively) are generally distinct. In practice, since the symmetric assumption is quite reasonable, we suggest to use \cref{eq:rsposym} as it essentially gets double the data for free. In theory, in order to easily leverage existing results \citep{clemenccon2008ranking} and since the rates would be unaffected, we analyze \cref{eq:rspo}.

\begin{remark}[Pairs-of-Pairs DPO]\label{remark:popdpo}
Using the symmetric-link AUC objective in \cref{eq:rsposym}, replacing the 0-1 loss $\phi$ with a logistic surrogate loss ($\log(1+e^{-u})$), and instantiating with KL divergence ($f(u)=u\log u$) yields a \emph{pairs-of-pairs} variant of DPO (PoP-DPO):
\[
\hat\theta_{\rm PoPDPO}\in\argmin_{\theta\in\Theta}\sum_{i,j}\log(1+e^{-s_i(\theta)-s_j(\theta)}),~~
s_i(\theta)=\log\prns{\frac{\pi_\theta(y_{i,z_i}\mid x_i)}{\piref(y_{i,z_i}\mid x_i)}
\frac{\piref(y_{i,1-z_i}\mid x_i)}{\pi_\theta(y_{i,1-z_i}\mid x_i)}}.
\]
This is similar to DPO's objective, which can be written as minimizing $\sum_i\log(1+\exp(-s_i(\theta)))$. However, PoP-DPO compares a pair against another pair, rather than only within-pair comparisons.
\end{remark}

\subsection{Policy error guarantees}

Since the AUC objective involves a step function, to appropriately control the complexity we need to consider the sets of pairs that ranked in order or anti-order. To measure the complexity of a class $\mathcal S$ of subsets of $\mathcal W\times\mathcal W$, define $N_\triangle(\epsilon,\mathcal S)$ as the smallest number of subsets $S_1,\dots,S_N\subset\mathcal W\times\mathcal W$ such that for every $S\in\mathcal S$ there exists $S_i$ with $\PP_{w,w'\sim P^2}((w,w')\in S\triangle S_i)\le\epsilon$, where $A\triangle B=(A\setminus B)\cup(B\setminus A)$ is the symmetric difference.
\begin{assumption}[Complexity of ranking sets]\label{asm:symmdiffentropy}
Define $\mathcal R_{\mathcal T}=\{\{(w,w'):t(w)-t(w')\geq0\}:t\in\mathcal T\}$. Suppose there exist constants $p\geq0,A>0$ such that, if $p>0$ we have $\log N_\triangle(\epsilon,\mathcal R_{\mathcal T})\leq A\epsilon^{-p}$ and if $p=0$ we have $N_\triangle(\epsilon,\mathcal R_{\mathcal T})\leq (A/\epsilon)^{v}$ for some constant $v>0$.
\end{assumption}
This again neatly packages two cases: the truly nonparametric case $p>0$ and the ``simple" case $p=0$. In particular, if $\mathcal R_{\mathcal T}$ were a VC class of dimension $v$ then we would have \cref{asm:symmdiffentropy} with $p=0$ and the given $v$ and some universal constant $A$ \citep[theorem 2.6.4]{vdVaartWellner1996}.

We can now establish our main result for RSPO, where we permit flexible indices under \cref{asm:symmdiffentropy}.

\begin{theorem}\label{thm:rspo-rate}
Suppose \cref{asmp:div,asmp:realizability,asmp:basic,asmp:ospo-margin,asm:rspo-regularity,asm:symmdiffentropy} hold. Then for a constant $c>0$, for all $\delta\in(0,0.5)$ and sufficiently large $n$, with probability at least $1-\delta$,
\[
\rho(h_{\hat\theta_{\rm RSPO}})\,\leq\,
c\;\sqrt{\log(1/\delta)}\;\left\{\begin{array}{lr}
\sqrt{{v\log(n)}/{n}} & p=0 \\
n^{-1/(2+2p)} & 0<p<1 \\
n^{-1/4}\sqrt{\log n} & p=1 \\
n^{-1/4p} & p>1
\end{array}\right.
\]
\end{theorem}

RSPO avoids nuisance estimation but can suffer slower rates. While \cref{asm:indexentropy,asm:symmdiffentropy} are generally incomparable, the former implies the latter under certain conditions (\cref{lemma:entropy}): with bounded densities, \cref{asm:indexentropy} with $p=2$ implies a $\tilde O(n^{-1/4})$ rate for OSPO but only a $\tilde O(n^{-1/12})$ rate for RSPO.

\subsubsection{Rate comparison}

Importantly, we are able to attain policy error rates for RSPO without having to estimate any nuisance parameters as we did in OSPO, while PSPO had no nuisances but also no rates. The rates we obtain for RSPO and OSPO are generally incomparable since \cref{asm:indexentropy,asm:symmdiffentropy} are generally incomparable. The next lemma establishes on situation in which the former does imply the latter so that we can compare rates implied by the former.

\begin{lemma}\label{lemma:entropy}
Suppose for some $c_{\rm margin}>0$ and $\alpha\in(0,1]$
we have $\PP_{w,w'\sim P^2}(\abs{t(w)-t(w')}\leq u)\leq c_{\rm margin} u^\alpha$ for each $t\in\Tcal$. Then for each $q\in[1,\infty]$, for a constant $c_q>0$, we have $N_\triangle(\epsilon,\mathcal R_{\mathcal T})\leq N_q((\epsilon/c_q)^{1/\beta},\mathcal T)$ with $\beta=\frac{\alpha q}{\alpha+q}$ for $q<\infty$ and otherwise $\beta=\alpha$.
\end{lemma}

If, for example, indices have bounded densities then the margin condition above would hold with $\alpha=1$. This would then mean that \cref{asm:indexentropy} with $p>0$ implies \cref{asm:symmdiffentropy} with larger exponent $3p/2$, and \cref{asm:indexentropy} with $p=0$ and constant $v$ implies \cref{asm:symmdiffentropy} with $p=0$ and the larger constant $3v/2$. For example, if $p=2$, OSPO would attain a rate of $\tilde O(n^{-1/4})$ compared to a much worse rate of $\tilde O(n^{-1/12})$ for RSPO.

\section{Empirical Illustration of Link-Robustness}
\label{sec:empirical}

We next demonstrate how SPO performs when the link is unknown and varies. We focus on settings with known rewards so we can vary only the preference-reward link and evaluate using the same reward, isolating the effect of the link. We use KL divergence ($f(u)=u\log u$) throughout. Note that while our theory analyzes exact empirical optimization, here we use first-order implementations as in \cref{sec:optimization} to apply to neural policies: PSPO alternates gradient steps in $\theta$ with PAVA updates to $\Psi$, OSPO uses gradient steps with kernel-regression plug-ins (leave-one-out in the synthetic experiment and a rolling memory buffer in the Qwen3 experiment), and RSPO is implemented as PoP-DPO (\cref{remark:popdpo}).
Replication code is provided at \url{https://github.com/causalml/spo/}.

\begin{figure}[t!]
\centering%
\begin{subfigure}[t]{0.32\textwidth}
\centering
\includegraphics[width=\linewidth]{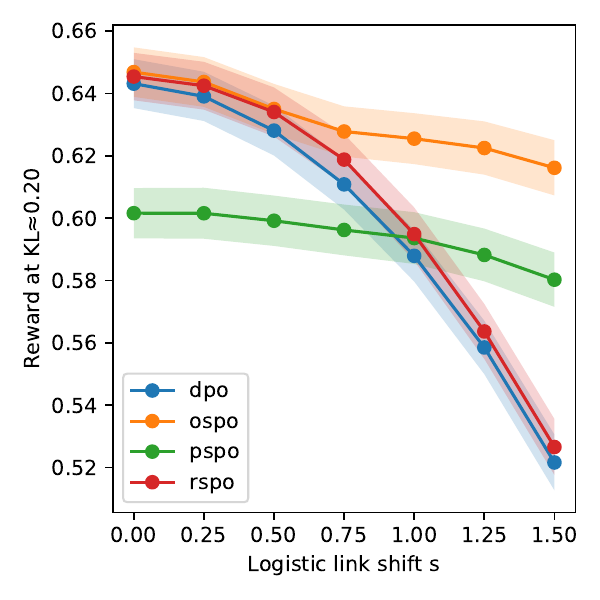}
\caption{Reward at a given KL budget.}
\label{fig:threepanel:a}
\end{subfigure}\hfill%
\begin{subfigure}[t]{0.32\textwidth}
\centering
\includegraphics[width=\linewidth]{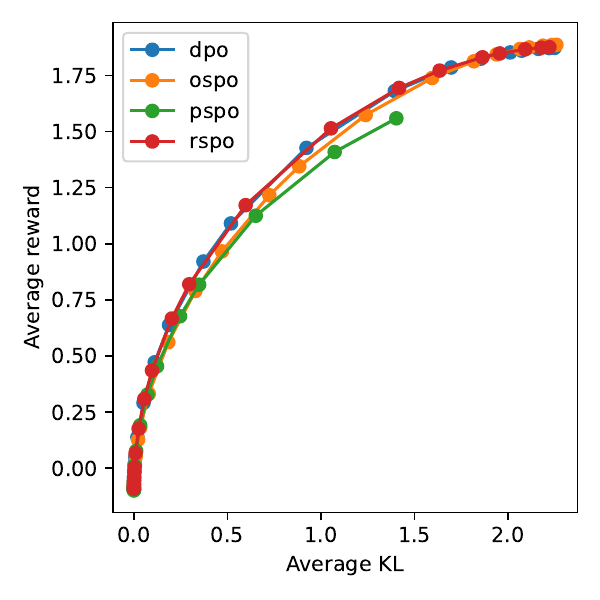}
\caption{Achievable reward-KL, $s=0$.}
\label{fig:threepanel:b}
\end{subfigure}\hfill%
\begin{subfigure}[t]{0.32\textwidth}
\centering
\includegraphics[width=\linewidth]{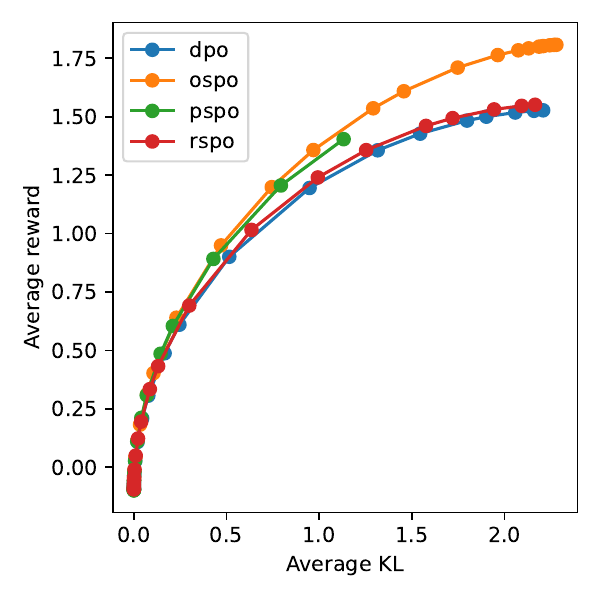}
\caption{Achievable reward-KL, $s=1.5$.}
\label{fig:threepanel:c}
\end{subfigure}%
\caption{Synthetic preference optimization experiment.}
\label{fig:synth_exp}
\end{figure}

\subsection{Synthetic preference optimization experiment}\label{sub:synthetic}

We simulate preferences with 20-dimensional Gaussian contexts $x$ and $|\Ycal|=10$ actions, using a two-hidden-layer MLP policy class and a uniform reference policy. We fix $\theta^\star$ and set $r^\star(x,y)=10 \log(\pi_{\theta^\star}(y\mid x)/\pi_{\mathrm{ref}}(y\mid x))$ and generate preferences via a mixture of shifted logistic links $0.5\,\sigma(4(u-s))+0.5\,\sigma(4(u+s))$, where $s$ controls misspecification ($s=0$ is logistic). We train on $n=1000$ preference pairs and replicate 1000 times. Additional detail is given in \cref{apx:synth_exp}.

\Cref{fig:synth_exp} shows the average reward (and 90\%-confidence intervals over runs) at KL budget $\kappa=0.2$ as $s\in[0,1.5]$ varies and the average reward-divergence curves for $\pi_{\beta,\hat\theta}$ as $\beta$ varies, at $s=0$ and $s=1.5$. All methods perform well when logistic is well-specified ($s=0$). DPO performance deteriorates as $s$ grows, and so does RSPO due to its logistic surrogate loss, albeit providing some robustness at low misspecification. PSPO performs less well, consistent with its rough loss surface, but is robust to misspecified link, with performance roughly flat as $s$ varies. There is some light deterioration due to the lower SNR given the link is very flat around 0 when $s$ is large (rather than due to misspecification). OSPO consistently performs the best across all $s$, matching the link-agnostic efficiency intuition behind OSPO in this controlled setting.

\subsection{Aligning Qwen3 on UltraFeedback}\label{sub:qwen}

We next run a controlled LLM experiment with Qwen3 \citep{yang2025qwen3}, UltraFeedback \citep{cuiultrafeedback}, and Skywork-Reward-V2 \citep{liu2025skywork}. We take as our reference model Qwen3-0.6B after light supervised fine-tuning on UltraFeedback responses to ensure format adherence, and treat Skywork-Reward-V2-Qwen3-1.7B as a proxy reward function. For $n=5000$ prompts, we generate two responses from the reference policy, score them with the reward model, and take the difference, normalizing so differences have variance 1. We generate preference observations using the same shifted-logistic family with varying shift. Thus the experiment tests link misspecification and optimization in an LLM pipeline, not the validity of reward-model feedback itself or human-alignment quality. We run each method on preference observations alone (no reward) to train a rank-16 LoRA. We then vary $\beta$, use the approximate token-level generation described in \cref{sec:realizing-calibrated-policy} to generate responses from the aligned and scaled autoregressive policy, and score using the reward model. Additional detail is given in \cref{apx:llm_exp}.

\begin{figure}[t!]
\centering%
\includegraphics[width=0.6\linewidth]{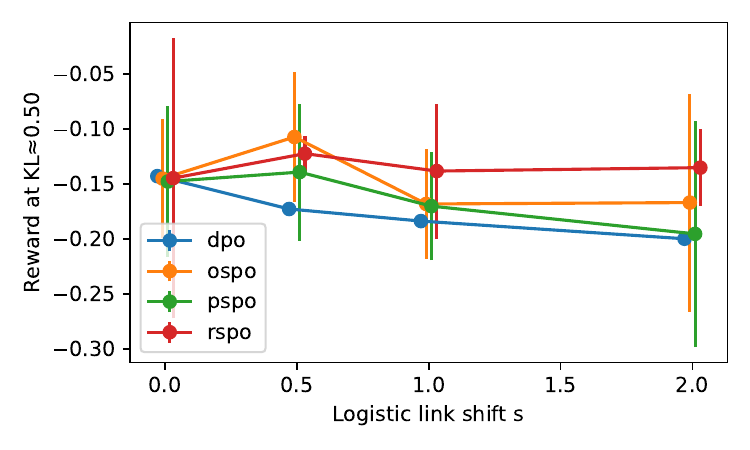}%
\caption{Aligning Qwen3 on UltraFeedback.}%
\label{fig:llm_exp}%
\end{figure}

\Cref{fig:llm_exp} depicts the resulting average reward over 4 replications at $\kappa=0.5$ (mean rewards, with 90\% confidence intervals for the paired differences from DPO shifted to absolute scale). The experiment is small, so the figure should be read qualitatively. It nevertheless mirrors the synthetic experiment results: methods are similar under a well-specified logistic link, DPO deteriorates as the link becomes more misspecified, and SPO methods remain more robust to non-logistic links. In particular, in this more large-scale setting, OSPO and PSPO are more unstable noisy, consistent with their more intricate training procedures. RSPO (as PoP-DPO) is the most stable SPO variant here, using extra pairwise rank information while retaining a familiar logistic surrogate that is easy to optimize.

\section{Implementation of Learners Using First-Order Optimization}
\label{sec:optimization}

The learners in \cref{sec:pspo,sec:ospo,sec:rspo} are defined as fully optimizing a sample-average objective over policies. To use them with neural policies, including LLMs, we need differentiable objectives or alternating updates that can be optimized with standard first-order methods and minibatches. This section provides suggested recipes for practical implementations amenable to such first-order optimization, which we also use in our proof-of-concept experiments. For each method, after fitting $\hat\theta$, we empirically calibrate the divergence to the given constraint as in \cref{sec:empirical-calib}.

The theory should be read as analyzing the statistical -- rather than computational -- behavior of the ideal procedures that fully optimize the empirical objective. The practical algorithms below replace these ideal optimizers with first-order algorithms and, in a few places, smooth or clipped approximations. These choices make the methods implementable in practice and in particular with neural networks.

\subsection{PSPO}\label{sec:alg-pspo}

For fixed $\theta$, the inner optimization over monotone links $\Psi$ in \cref{eq:pspo-obj} is a one-dimensional isotonic Bernoulli likelihood problem. Sorting the observations by $t_\theta(w_i)$, the maximizing $\Psi$ is the isotonic fit of $z_i$ on this sorted index and can be computed by the Pool Adjacent Violators Algorithm (PAVA). This allows us to efficiently compute the objective function in practice. This suggests a practical alternating scheme, akin to the Isotron algorithm \citep{KalaiSastry2009,KakadeKalaiKanade2011}: compute current indices $t_\theta(w_i)$, fit $\Psi$ by PAVA, hold this fitted link fixed (with interpolation and clipping for numerical stability), take several minibatch gradient steps in $\theta$ on the resulting cross-entropy loss, and refresh the isotonic fit periodically.

The PAVA step is exact for the current $\theta$, so the method is best viewed as a block-coordinate approximation to empirical PSPO. If the fitted link were refreshed after every infinitesimal update and the outer problem were globally optimized, it would recover the fully optimized empirical profiled objective analyzed by \cref{thm:pspo}. The interpolation or smoothing, clipping, finite refresh schedule, and finite SGD horizon are of course a departure from the full empirical optimization analyzed in the theory. Differentiating through PAVA or replacing $\Psi$ by a monotone neural network with shape constraints are alternative implementations of the same profiled-link idea.

In the synthetic experiments in (\cref{sub:synthetic}), we use the clipped PAVA fit with linear interpolation and no extra logistic smoothing. In the Qwen3 experiments (\cref{sub:qwen}), where a piecewise-constant link gives little useful gradient, we use soft nearest-neighbor interpolation over the PAVA knots, mix in a 5\% logistic link, clip probabilities, and refresh the PAVA fit every 50 gradient steps.

\subsection{OSPO}\label{sec:ospo_optimization}

For first-order optimization, the issue is how to evaluate and update the plug-in objective in \cref{eq:ospo-obj} as $\theta$ changes. At a current iterate, we compute the indices, fit $\hat g_\theta$ by one-dimensional regression of $z$ on $t_\theta(w)$, and take minibatch gradient steps in $\theta$ using the resulting plug-in loss. Kernel regression is the default recommendation: standardize the stored index values, use a bandwidth on this normalized scale, and clip the estimated probabilities away from $0$ and $1$ to keep the log loss finite. The kernel weights themselves need not be clipped.

When the same observations are used for regression and policy updates, we recommend leave-one-out estimates for the current point. For LLM minibatches, a rolling memory buffer of recent $(t_\theta(w),z)$ values is a practical substitute for all-data leave-one-out. In automatic differentiation, we detach the stored index values but allow gradients through the current query index in the kernel regression; this keeps the regression fit from becoming a second model to optimize through while preserving useful gradients for $\theta$. A small mixture with a smooth baseline link can also be used as numerical regularization when the kernel fit gives very flat gradients.

In the synthetic experiments (\cref{sub:synthetic}), OSPO uses the full training set as the kernel memory, masks each observation out of its own regression, uses a Gaussian kernel with bandwidth scaled as $n^{-1/5}$ after standardizing by the index standard deviation, detaches the stored values between refreshes, and clips $\hat g_\theta$ to $[10^{-6},1-10^{-6}]$ without any sigmoid mixture. In the Qwen3 experiments (\cref{sub:qwen}), OSPO uses a Gaussian kernel with bandwidth $0.5$ over a FIFO memory buffer of at most 4096 detached past index-label pairs, normalizes by the memory standard deviation, clips $\hat g_\theta$ to $[10^{-4},1-10^{-4}]$, mixes in 5\% of the logistic sigmoid $\sigma(t_\theta)$, and updates the memory after each minibatch step.

Finally, the OSPO loss is invariant to flipping the sign of the learned index. In implementation we align the sign after training using the empirical covariance between $t_\theta(w)$ and $z$.

\subsection{RSPO}\label{sec:rspo_optimization}

The empirical objective in \cref{eq:rspo} (or, preferably, \cref{eq:rsposym}) is hard to optimize since $\phi$ is a step function. A practical implementation replaces $\phi$ by a nonincreasing margin surrogate $\ell:\RR\to\RR_+$, such as the logistic loss $\ell(u)=\log(1+\exp(-u))$, the exponential loss, or a squared hinge loss. For the conditional AUC objective in \cref{eq:rspo}, this gives
\[
\hat L_{\ell,\rm rank}(\theta)
=\frac1{n_1n_0}\sum_{i:z_i=0,j:z_j=1}
\ell\prns{t_\theta(w_j)-t_\theta(w_i)} ,
\]
which we minimize over $\theta$. For the symmetric objective in \cref{eq:rsposym}, write
\[
a_i(\theta)=t_\theta(x_i,y_{i,1-z_i},y_{i,z_i}),
\]
the policy-implied index with the preferred response placed second. Since our indices are antisymmetric in the two responses, the corresponding surrogate objective is
\[
\tilde L_{\ell,\rm sym}(\theta)
=\frac1{n^2}\sum_{i,j}
\ell\prns{a_i(\theta)+a_j(\theta)} .
\]
With KL divergence and logistic $\ell$, this is exactly the PoP-DPO objective in \cref{remark:popdpo}. It is differentiable and can be optimized by standard stochastic first-order methods. In minibatch training, we take a batch of $B$ preference observations and form all $B^2$ within-batch terms; this estimates the full $U$-statistic gradient while reusing each sampled example many times. In all experiments (across \cref{sub:synthetic,sub:qwen}), we use this symmetric PoP-DPO implementation directly with minibatch gradients.

Using a surrogate changes the object being optimized, so the exact RSPO theorem in the main text should not be read as a theorem for PoP-DPO. The theorem analyzes empirical AUC maximization with the step loss. Surrogate losses are justified by classification-calibration results: \citet{bartlett2006convexity} show that classification-calibrated losses control 0--1 classification regret via a calibration function, and for convex margin losses calibration is equivalent to differentiability at zero with negative derivative. The logistic loss satisfies this condition. For bipartite ranking specifically, \citet{agarwal2014surrogate} show that strongly proper composite losses control AUC/ranking regret; logistic, exponential, squared, and squared-hinge losses are examples, with logistic yielding an explicit square-root ranking-regret bound in terms of logistic surrogate regret. These results can be chained with ours by using surrogate regret to bound AUC regret, which per our \cref{thm:auc-margin} controls policy error. This chain generally incurs a calibration penalty: using the generic strongly-proper-loss bound, surrogate excess $\epsilon$ gives AUC excess of order $\sqrt{\epsilon}$, and hence policy-index error of order $\epsilon^{1/4}$ after \cref{thm:auc-margin}. This is a conservative worst-case translation, however; low-noise conditions can improve the calibration exponent \citep{agarwal2014surrogate}. The exact-AUC RSPO theorem avoids this surrogate-calibration penalty but is not directly optimized by gradient methods.

\section{Empirical Recalibration to the Divergence Budget}
\label{sec:empirical-calib}

After training $\hat\theta$, we must choose a temperature $\beta$ that enforces the divergence budget.
So far we wrapped this into the population-calibrated policy $\pi_{\kappa,\theta}$, but this cannot be done in practice. We now discuss how to calibrate the divergence using a validation sample of contexts $x$, and we show how this leads to an additive policy error that we can simply add on top of our previous bounds on the error in $\pi_{\kappa,\hat\theta}$ for a learner $\hat\theta$.

The procedure solves a one-dimensional root-finding problem on held-out contexts. We show a uniform law of large numbers in the temperature, a root-$m$ calibration rate, and the ensuing policy error guarantee. While the results require a separate validation set, so that we can treat $\hat\theta$ as a fixed input, in practice we may calibrate on the same data (using only the contexts).

Fix a compact interval $\mathcal B=[\underline\beta,\overline\beta]\subset(0,\infty)$.
Let $x_1,\dots,x_m\stackrel{\mathrm{iid}}{\sim}P_x$ and define
\[
\hat\Phi_m(\beta,\theta)=\frac1m\sum_{i=1}^m D_f\!\bigl(\pi_{\beta,\theta}(\cdot\mid x_i)\,\big\|\,\pi_{\mathrm{ref}}(\cdot\mid x_i)\bigr),
\qquad
\hat\beta_\theta\in\argmin_{\beta\in\mathcal B}\bigl|\hat\Phi_m(\beta,\theta)-\kappa\bigr|.
\]

\begin{theorem}[Empirical $\beta$-calibration]\label{thm:empirical-beta-primitive}
Suppose the conditions of \cref{thm:policy-stability-metric} hold.
Then for some constant $c>0$, we have that for any $\theta\in\Theta$ and $\delta\in(0,0.5)$, with probability at least $1-\delta$,
\[
\|\pi_{\hat\beta_\theta,\theta}-\pi_{\kappa,\theta}\|_1\leq \frac{c}{v}\sqrt{\frac{\log(m/\delta)}{m}},
\]
provided $v=\mathbb E_x\operatorname{Var}_{y\sim\pi_{\mathrm{ref}}(\cdot\mid x)}(h_\theta(x,y))>0$ (which implies $\beta_{\kappa,\theta}$ is positive and unique) and $\beta_{\kappa,\theta}\in\Bcal$.
\end{theorem}

Therefore, we can simply add this additional $\tilde O(1/\sqrt{m})$ term to whatever policy error we derived on the population-calibrated policy by bounding $\rho(\theta)$ via \cref{thm:policy-stability-metric}. In theory, if we split the $n$ data points and use a constant fraction for calibration, we will never deteriorate the rate. In practice, we can use the same data since calibration should be rather stable and only ever uses the $x$-data.

\subsection{Realizing the Calibrated Policy in the Original Policy Class}
\label{sec:realizing-calibrated-policy}

The divergence-calibrated policy returned by the procedure in \cref{sec:empirical-calib} may or may not remain in the original policy $\{\pi_\theta:\theta\in\Theta\}$.
In some applications we may want it to, for example when the class is a family of transformer LLMs and the final output should be another transformer LLM that can be sampled autoregressively. The combination of the results in the paper, namely \cref{thm:policy-stability-metric}, \cref{thm:pspo,thm:ospo-rate,thm:rspo-rate} (depending on the learner chosen), and \cref{thm:empirical-beta-primitive}, together establish that empirical calibration on top of an SPO-learned $\hat\theta$ has small policy error to $\pis$. However, the calibrated policy returned, $\pi_{\hat\beta_{\hat\theta},\hat\theta}$, need not automatically be an element of $\{\pi_\theta:\theta\in\Theta\}$. Nonetheless, since the target $\pis$ is assumed to be in this class under \cref{asmp:realizability} (or approximately so; see \cref{sec:approx-realizability}), it is reasonable to desire that the returned policy, which is approximating $\pis$, to be in the class as well.

\paragraph*{Realizability under index scaling closure.}
One clean sufficient condition is that the policy class can realize the bounded index rescalings used by calibration. Let $\mathcal C=\{\beta^{-1}:\beta\in\mathcal B\}$ and assume $\mathcal C\mathcal T=\{c\,t:c\in\mathcal C,t\in\mathcal T\}\subseteq\mathcal T$. Then $\pi_{\beta,\theta}$ is already realizable in the policy class: its induced potential is $\beta^{-1}(h_\theta(x,y)-\lambda_{\beta,\theta}(x))$, hence its index is $\beta^{-1}t_\theta$, and the closure assumption gives a $\theta'\in\Theta$ inducing this same scaled index, and hence the same policy. 

\paragraph*{Approximate sampling for autoregressive policies.}
Even when exact scaling closure is not assumed, calibration in the KL case has a simple autoregressive approximation. For $f(u)=u\log u$,
\[
\pi_{\beta,\theta}(y\mid x)\propto \piref(y\mid x)^{1-1/\beta}\pi_\theta(y\mid x)^{1/\beta}.
\]
If $\piref$ and $\pi_\theta$ are autoregressive models with next-token logits $\ell_{\rm ref}(\cdot\mid x,y_{<t})$ and $\ell_\theta(\cdot\mid x,y_{<t})$, a practical sampler draws each token from the softmax of $(1-\beta^{-1})\ell_{\rm ref}+\beta^{-1}\ell_\theta$. This exactly samples the locally normalized geometric mixture at each prefix, but not generally the globally normalized sequence-level policy above. The exact conditional would also include a continuation-normalization factor measuring the total geometric-mixture mass of all future completions after each candidate next token, while this token-level mixing drops that factor. Thus the approximation is good exactly when these continuation factors vary little across plausible next tokens, which is the case when $\pi_\theta$ stays close to $\piref$ along the relevant prefixes, as in small alignment adjustments around a reference model.

\paragraph*{Calibrating a priori, instead of a posteriori.}
Another way to return an in-class policy is to build the divergence budget directly into the learner by limiting it only to indices corresponding to policies with calibrated divergences. Essentially, rather than calibrate the index after learning one, we restrict to only already-calibrated indices, moving the calibration step to be before (and uniform over indices) rather than after (and for the single given index). 

Using a validation sample $x_1,\ldots,x_m$ independent of the preference sample, define
\[
\overline\Theta_m
=
\left\{
\theta\in\Theta:
\left|
\frac1m\sum_{i=1}^m D_f\!\bigl(\pi_{\theta}(\cdot\mid x_i)\,\big\|\,\pi_{\mathrm{ref}}(\cdot\mid x_i)\bigr)
-\kappa
\right|
\le \varpi
\right\}.
\]
We can then run any learner in the paper over $\overline\Theta_m$ instead of $\Theta$. Since the validation sample is separate, $\overline\Theta_m$ is fixed relative to the preference sample; conditional on $\theta^\star\in\overline\Theta_m$, the existing learner theory applies unchanged to this smaller class, with no larger complexity and with any uniform-in-$\theta$ assumptions inherited from $\Theta$. 

In particular, we can choose $\varpi$ so that, with high probability,
\begin{equation}\label{eq:div_ulln}
\sup_{\theta\in\Theta}
\left|
\frac1m\sum_{i=1}^m D_f\!\bigl(\pi_{\theta}(\cdot\mid x_i)\,\big\|\,\pi_{\mathrm{ref}}(\cdot\mid x_i)\bigr)
-
\mathbb E_xD_f\!\bigl(\pi_{\theta}(\cdot\mid x)\,\big\|\,\pi_{\mathrm{ref}}(\cdot\mid x)\bigr)
\right|
\le \varpi .
\end{equation}
Then \cref{asmp:realizability} implies $\theta^\star\in\overline\Theta_m$, and every $\theta\in\overline\Theta_m$ has population divergence within $2\varpi$ of $\kappa$. For such a learned $\theta$, small $\rho(h_\theta)$ and small divergence error together imply small policy error. Indeed, \cref{thm:policy-stability-metric} gives $\|\pi_{\kappa,\theta}-\pi^\star\|_1\lesssim \rho(h_\theta)$. Also, $\pi_\theta=\pi_{1,\theta}$, and if its population divergence is within $2\varpi$ of $\kappa$, the calibration-stability argument underlying \cref{thm:empirical-beta-primitive} gives $\|\pi_\theta-\pi_{\kappa,\theta}\|_1\lesssim\varpi$ under the same nondegeneracy conditions. Thus, for $\theta\in\overline\Theta_m$, $\|\pi_\theta-\pi^\star\|_1\lesssim \rho(h_\theta)+\varpi$.

Compared to the argument in \cref{thm:empirical-beta-primitive}, however, this approach suffers the complexity of the policy class in $\varpi$ in order for \cref{eq:div_ulln} to hold with high probability. For example, let
\[
\mathcal D_\Pi
=
\left\{
x\mapsto D_f\!\bigl(\pi_\theta(\cdot\mid x)\,\big\|\,\pi_{\mathrm{ref}}(\cdot\mid x)\bigr)
:\theta\in\Theta
\right\},
\]
If $\mathfrak R_m(\mathcal D_\Pi)$ denotes its expected Rademacher complexity on $m$ contexts and $\mathcal D_\Pi$ is uniformly bounded $B_D$, then a standard uniform deviation bound gives \cref{eq:div_ulln} with probability at least $1-1/m$ for
\[
\varpi
=
2\mathfrak R_m(\mathcal D_\Pi)
+B_D\sqrt{{\log(2m)}/{2m}}.
\]
Thus the price of this a-priori calibration approach is suffering the complexity of the divergence functions induced by the full policy class. In contrast, in \cref{thm:empirical-beta-primitive}, $\hat\theta$ is already fixed and calibration is only a univariate search over the compact interval $\mathcal B$, so this policy-class complexity does not enter the calibration term.

\section{Extensions}\label{sec:extensions}

\subsection{Localizing $\hat g_\theta$ in OSPO}

OSPO as posited in \cref{eq:ospo-obj} requires we estimate $\hat g_\theta$ for each $\theta\in\Theta$, and moreover the bound in \cref{thm:ospo-rate} depends on its uniform convergence over $\theta\in\Theta$. One way to avoid this would be to localize the link estimate around an initial $\theta$ estimator in the style of \citet{kallus2019localized}. That is, first estimate an initial $\hat\theta$ (perhaps using PSPO), then estimate $\hat\Psi$ as the univariate regression of $z$ on $t_{\hat\theta}(w)$, then maximize the plug-in quasi-likelihood $\frac{1}{n}
\sum_{i=1}^n \ell(t_\theta,\hat\Psi;w_i,z_i)$.

This has the benefit that in the end what we optimize looks like a standard classification loss with a fixed link function, as in DPO. While $\hat\Psi$ may not be monotonic, we only get a better estimate via rearrangement \citep{chernozhukov2009improving}, ensuring monotonicity. It may, however, not be so smooth. Additionally, it may be challenging to ensure fast enough convergence of the initial $\hat\theta$ so that we get a fast enough rate on $\hat\Psi$, since our current analysis for PSPO gives no rate. Finally, the analysis will require adapting the approach of \citet{kallus2019localized}, which is developed for point estimation and inference, to orthogonal statistical learning in the vein of \citet{foster2023orthogonal}. This therefore remains a promising direction of future work.

\subsection{Max-score estimators}\label{sec:max_score_estimators}

In the single-index model, the distribution of $z$ only depends on $w$ via the index $\tst(w)$. This means that in the random utility interpretation, $z=\indic{\tst(w)\geq\epsilon}$, the idiosyncrasy $\epsilon$ is homoskedastic, $\epsilon\indep w$. Namely, $\Psi^\star$ is the $w$-independent CDF of $\epsilon$. A strict generalization is to require only that $\mathrm{Median}(\epsilon\mid w)=0$ is a constant function of $w$. This means that $z\mid w\sim\mathrm{Bernoulli}(\Psi^\star_w(\tst(w)))$ for an unknown $w$-dependent CDF $\Psi^\star_w$ under the only constraint $\Psi^\star_w(0)=1/2$. Note that this is not a single-index model since the distribution of $z$ would depend on $w$ not only via a single index $\tst(w)$.

This is the model studied by \citet{Manski1975,Manski1985,KimPollard1990,Horowitz1992} for linear indices $\Tcal=\{w\mapsto\beta\tr w\}$. Indeed, in this model, any $t$ with $\indic{t(w)\geq0}=\indic{\tst(w)\geq0}$ is observationally equivalent to $\tst$. Therefore, it is generally impossible to identify up to scale anything much more complex than a linear index. That is, requiring that sign agreement of $t,t'$ implies $t\propto t'$ is highly restrictive. It, for example, prevents any class that admits ``bump" functions, such as H\"older or Sobolev classes, since they would maintain sign while changing shape and scale. In contrast, the identification \cref{asm:aip-id} is much more lax and can admit such nonparametric classes. For example, for any atomless CDF $F_0$, any subset of $\{af:a>0,f(w)\sim F_0\}$ (such as further restricting $f$ in some nonparametric function class) satisfies \cref{asm:aip-id}. Fixing the marginal distribution thus is exactly what is done with normalizing flows / invertible networks.

For learning \emph{linear} indices up to scaling, \citet{KimPollard1990} establish $n^{-1/3}$ rates for the normalized coefficients that define the half-space that minimizes the binary 0-1 classification loss (the so-called max-score estimator). \citet{Horowitz1992} replaces the 0-1 step function with a smooth CDF with a vanishing bandwidth (the so-called smoothed max-score estimator) and establishes a $n^{-s/(2s+1)}$ rate when $\Psi_w^\star$ is $s$-times H\"older-continuously differentiable. 

These results can be directly applied to learning indices implied by realizability via \cref{prop:sim} up to scale (and chained with \cref{thm:policy-stability-metric} for policy error), provided they are linear. Although identification up to scale (and therefore of divergence-constrained policies) for more complex classes is doubtful, these methods may still be practical ways to align LLMs, that is, by aligning the \emph{sign} of $t_\theta(w)$ with that of $z-0.5$. In particular, the smoothed objective of \citet{Horowitz1992}, while non-convex and potentially very rough, is differentiable so that first-order updates can be applied. This essentially ensures the LLM aligns (in the sense of divergence-constrained max-reward) with some reward whose differences align with the signs of true reward differences, even if the rewards are in fact different and lead to different policies -- this is of course still valuable. Notice that if instead of smoothing, we used the logistic loss as a surrogate for the 0-1 loss we would have the DPO objective.

Notice that, while max score is equivalent to maximizing the frequency that $h_\theta(x,y_z)$ is ranked higher than $h_\theta(x,y_{1-z})$, this is \emph{distinct} from the ranking approach in RSPO. In RSPO we ranked \emph{pairs of pairs} of demonstrations. Crucially, by doing so we probed $\tst(w)$ at a \emph{variety of thresholds} $\tst(w')$, rather than only probing it at $0$. This ensures we were sensitive to information about all level sets, rather than just at $0$ (\ie, the sign). This is also the source of the better asymptotics seen in the economics literature on structural estimation of linear indices (\eg, $n^{-1/2}$ in \citep{Sherman1993} vs $n^{-1/3}$ in \citep{KimPollard1990} for \emph{linear} indices): averaging over the many thresholds smoothes out the step function, while always thresholding at $0$ gives no opportunity for such smoothing. On the other hand, for this to work requires a  single-index model (\ie, a common link function, equivalently $\epsilon\indep w$), while if all we know is $\mathrm{Median}(\epsilon\mid w)=0$ we cannot do much more than threshold at $0$ and learn only the sign. In this paper, of course, we focused on single-index models, in which case the \emph{pairs of pairs} ranking appears preferable, giving both a better chance for identifying the reward shape and thus the optimal policy as well as better asymptotic behavior.

\subsection{Approximate Realizability and Misspecified Index Classes}
\label{sec:approx-realizability}

\Cref{asmp:realizability} says that the divergence-constrained optimal policy $\pi^\star$ is contained in the policy class. This assumption is not needed for the preference distribution itself to have a single-index form. By \cref{thm:closed-form}, the target policy induces the canonical potential
\[
h^\star(x,y)=f'\!\left(\frac{\pi^\star(y\mid x)}{\pi_{\rm ref}(y\mid x)}\right),
\]
and therefore the canonical index $t^\star(w)=h^\star(x,y_1)-h^\star(x,y_0)$. The preference model remains
\[
P(z=1\mid w)=\Psi^\star(t^\star(w))
\]
for an unknown monotone link. What can fail is simply that this canonical index is not contained in the class $\Tcal$ induced by the policies $\{\pi_\theta:\theta\in\Theta\}$ that we optimize over.

Thus, when \cref{asmp:realizability} fails, the natural interpretation is not that the single-index reduction disappears. Rather, the induced index class is misspecified. Each learner should then be expected to target its own population best approximation within $\Tcal$: PSPO targets the best fit under the population profiled likelihood, OSPO under the population ideal quasi-likelihood, and RSPO under population AUC. These pseudo-targets need not coincide with one another, and if a population objective has multiple maximizers the estimator can only be expected to approach that set unless additional tie-breaking or regularization is used.

The theorems in the paper do not formally analyze this misspecified case. They prove that, under exact realizability and the corresponding identification or margin assumptions, optimizing the empirical objectives learns an index close enough to $t^\star$ to control policy error. Without realizability, the same empirical process arguments may still control convergence to a method-specific pseudo-target (but this may still require some small extensions of the proofs), but this would no longer by itself imply small error to $\pi^\star$.

Under \emph{approximate} realizability, the expected form of a result would be an oracle-type decomposition. The estimation term would measure how close the learned index is to the relevant pseudo-target. The approximation term would measure how far that pseudo-target is from the canonical potential $h^\star$, up to the same scale and $x$-dependent location transformations used in \cref{def:rho}. This approximation term is the right irreducible error for the method being used. It is not automatically equal to the best possible distance from $h^\star$ to the policy class, because the pseudo-target is chosen by preference fit, not directly by distance to $h^\star$.

Bounding the method-specific approximation term by the best policy-class approximation error requires an additional population calibration condition: roughly, the population objective must lose value in proportion to the scale- and location-invariant distance from $h^\star$. With such a condition, any policy class element close to $h^\star$ would have nearly optimal population objective value, forcing the pseudo-target also to be close to $h^\star$. Without such curvature or separation, a misspecified objective may be flat, have multiple far-apart maximizers, or prefer an index that fits preferences well but induces a poorer calibrated policy.

A full misspecification theory would therefore require new assumptions and proofs: uniqueness or set-valued control of pseudo-targets, population calibration of each objective to policy error, and oracle inequalities that add this approximation term to the statistical rates. Developing these results is outside the scope of the present paper, which focuses on establishing that link-robust policy learning is possible under exact policy realizability.

\section{Related Literature}
\label{sec:related}

\paragraph*{RLHF/RLAIF.}
From the first demonstrations of reward learning from preferences \citep{Christiano2017,Ziegler2019,Stiennon2020} to instruction-following models \citep{Ouyang2022} and safety alignments \citep{Bai2022}, RLHF/RLAIF is a primary ingredient in modern LLM practice. \citet{zhan2023provable} prove regret bounds for offline RL from preference data. While RLHF often employs KL constraints or penalties, \citet{wang2023beyond,huang2024correcting} study the use of other $f$-divergences.
Our analysis is complementary, allowing general $f$ while addressing misspecification of the preference link.

DPO \citep{Rafailov2023} avoids explicit reward modeling by recognizing the optimal policy encodes the reward and fitting the implied reward function to the data, assuming a known logistic link function. We follow the idea of implied reward function but drop the known link function (and generalize KL to any $f$-divergence). Many direct-alignment variants change the surrogate, regularization, or reference-policy treatment rather than leaving the link unspecified, including SLiC-HF and RRHF ranking/calibration losses, $\Psi$PO/IPO, and reference-free objectives such as ORPO and SimPO \citep{zhao2023slic,yuan2023rrhf,gheshlaghiAzar2024general,hong2024orpo,meng2024simpo}. Another line changes the observation model itself from binary pairwise preferences to rankwise, listwise, or multiple-positives feedback \citep{liu2025lipo,chen2024softmaxdpo,zhu2026mult}.
Recent work has examined limitations of this fixed-link or representative-preference setup from several angles: robustness of reward models and alignment under preference-model misspecification \citep{hong2025robustness,xu2024strong}, alternatives to Bradley-Terry modeling \citep{sun2024rethinking}, distortion between optimized preference objectives and underlying preferences \citep{golz2025distortion}, and preference heterogeneity \citep{chidambaram2025direct}.
Two closely related robustness papers take different routes: \citet{zhang2025provableunknownlink} handle unknown links in general MDPs through zeroth-order sign policy optimization from batched trajectory comparisons, and \citet{xu2025doublyrobust} use doubly robust preference optimization that is consistent if either a full pairwise preference model or the reference policy is correct. Our route is complementary: in contextual, divergence-constrained alignment, we retain a reward-difference random-utility structure but make the link nonparametric and prove direct policy-error guarantees.

A variety of alignment methods depart from the structural/generative framework, where preferences are assumed to optimize a random utility or equivalently that preferences are generated by some conditional probability distribution given demonstrations and reward. \citet{ethayarajh2024kto} consider loss-aversion in the spirit of the prospect theory of human decision making (as compared to utility-maximizing). \citet{kong2024perplexity,liang2024ropo} consider non-stochastic corruptions to labels.
\citet{munos2024nash} instead optimize a game-theoretic equilibrium induced by pairwise preferences, avoiding reduction to a scalar reward model.
These directions address important failures of standard DPO-style modeling; our complementary goal is to retain a random-utility interpretation while making the link nonparametric.

\paragraph*{Econometric choice and semiparametrics.}
Discrete choice modeling in econometrics was pioneered by \citet{McFadden1974}, who formalized random utility models and introduced the multinomial logit as a tractable framework for estimating preferences and welfare from choice data. Subsequent work generalized the logit to allow richer substitution patterns and unobserved heterogeneity, including nested logit \citep{McFadden1978}, mixed logit and random coefficients models \citep{Berry1995,Train2009}, and more general GEV formulations \citep{McFadden1981}. Addressing price endogeneity and unobserved product characteristics led to instrumental-variables and inversion-based approaches such as the BLP framework \citep{BerryLevinsohnPakes1995,BerryLevinsohnPakes2004}, along with related control-function methods \citep{Newey1999,BlundellPowell2004}. Discrete choice models were later extended to dynamic settings to study intertemporal decision-making, beginning with \citet{Rust1987} and followed by extensive work on identification, estimation, and computation in dynamic discrete choice models \citep{HotzMiller1993,AguirregabiriaMira2002,ArcidiaconoMiller2011}.

Semiparametric binary choice with an unknown link function originates with \citet{Cosslett1983}, with efficient estimation established by \citet{KleinSpady1993}. Maximum score estimators \citep{Manski1975,Manski1985} and smoothed variants \citep{Horowitz1992} provide robustness to link misspecification, with cube-root asymptotics characterized by \citet{KimPollard1990}. (See \cref{sec:max_score_estimators} for and expanded discussion on implications for alignment.) Related order-restricted and isotonic likelihood methods appear in \citet{Barlow1972,Robertson1988}. These models belong to the broader class of single-index models, which extend beyond binary choice. General estimators include semiparametric least squares \citep{Ichimura1993}, average derivative \citep{PowellStockStoker1989}, and rank correlation estimator \citep{Han1987,Sherman1993}. 
Unlike this literature, our statistical target is not a finite-dimensional structural parameter but the optimal policy itself. This changes both the statistical object and the admissible complexity: the index may be induced by a flexible policy class, need not be finite-dimensional, and need only be learned up to transformations that preserve the calibrated policy.

\section{Conclusion}
\label{sec:conclusion}

Divergence-constrained alignment of LLMs induces a semiparametric single-index model with the policy class parametrizing the index. This relaxes fixed-link DPO while retaining a structural random-utility interpretation. Overall, the contribution is primarily theoretical: showing link-robust preference alignment is possible with nonparametric policy learning, even when structural indices are unidentified. We instantiate this idea via profiled likelihood, orthogonalized quasi-likelihood, and bipartite ranking learners, with finite-sample bounds targeting policy error directly. Rather than a comprehensive alignment benchmark, proof-of-concept experiments illustrated the link-robustness predicted by theory and demonstrate the possible tradeoffs between DPO, PSPO, OSPO, and RSPO.

\bibliographystyle{plainnat}
\bibliography{sip}

\appendix

\setcounter{secnumdepth}{4}
\let\subsubsubsection\paragraph

\section{Proofs}\label{sec:proofs}

\subsection{Proof of \Cref{thm:policy-stability-metric}}

For each $\theta\in\Theta$, write $\beta_\theta=\beta_{\kappa,h_\theta}$ and $\pi_\theta=\pi_{\kappa,h_\theta}$.
Define the canonical scores
\[
u_\theta(x,y)=f'\bigl(r_\theta(x,y)\bigr),
\qquad
u_{\theta^\star}(x,y)=f'\bigl(r_{\theta^\star}(x,y)\bigr),
\]
where $r_\theta=\pi_\theta/\pi_{\mathrm{ref}}$. And for brevity, write $D(\theta)=\|\pi_{\kappa,\theta}-\pis\|_1$.

\subsection*{Step 1: Existence and positivity of $\beta_\theta$}

For a fixed $h$ and $x$, the set of policies maximizing $y\mapsto h(x,y)$ consists of all distributions supported on $S_h(x)=\arg\max_y h(x,y)$.
Among these, the unique minimizer of $D_f(\,\cdot\,\|\,\pi_{\mathrm{ref}}(\cdot\mid x))$ is given by
\[
\pi_{\mathrm{greedy},h}(y\mid x)
=
\begin{cases}
\dfrac{\pi_{\mathrm{ref}}(y\mid x)}{\omega_h(x)}, & y\in S_h(x),\\[0.5em]
0,&\text{otherwise},
\end{cases}
\]
where $\omega_h(x)=\sum_{y\in S_h(x)}\pi_{\mathrm{ref}}(y\mid x)$. We have
$
D_f\bigl(\pi_{\mathrm{greedy},h}(\cdot\mid x)\,\big\|\,\pi_{\mathrm{ref}}(\cdot\mid x)\bigr)
=\omega_h(x)\,f\!\bigl(1/\omega_h(x)\bigr).
$
Thus the minimal divergence among $h$-greedy policies is
$
\Psi_h
=\mathbb E_x\Bigl[\omega_h(x)\,f\!\bigl(1/\omega_h(x)\bigr)\Bigr].
$

For each $\theta$, denote $\Psi_\theta=\Psi_{h_\theta}$.
Under the assumptions on $f$, $
\Phi_\theta(\beta)
=\mathbb E_x D_f\bigl(\pi_{\beta,h_\theta}(\cdot\mid x)\,\big\|\,\pi_{\mathrm{ref}}(\cdot\mid x)\bigr)
$
is continuous and strictly decreasing on $(0,\infty)$, with
$
\lim_{\beta\downarrow 0}\Phi_\theta(\beta)=\Psi_\theta>\kappa,
\qquad
\lim_{\beta\to\infty}\Phi_\theta(\beta)=0.
$
So, there exists a unique $\beta_\theta\in(0,\infty)$ such that $\Phi_\theta(\beta_\theta)=\kappa$. In particular, $\beta_\star=\beta_{\theta^\star}>0$.

\subsection*{Step 2: Relating policy discrepancy to canonical scores}

\begin{lemma}\label{lem:D-vs-u-brief}
Under the curvature lower bound $f''\ge m_f>0$ along the ratios $r_\theta$ and $r_{\theta^\star}$,
$
D(\theta)\ \le\ \frac{1}{m_f}\,\|u_\theta-u_{\theta^\star}\|_2$
for all $\theta\in\Theta$.
\end{lemma}

\begin{proof}
Fix $(x,y)$.
Let $r_1=r_\theta(x,y)$, $r_2=r_{\theta^\star}(x,y)$ and $u_j=f'(r_j)$, $j=1,2$.
By the mean value theorem there exists $\xi$ between $r_1$ and $r_2$ such that
$
u_1-u_2=f''(\xi)(r_1-r_2).
$
Since $f''(\xi)\ge m_f$, we get $|r_1-r_2|\le m_f^{-1}|u_1-u_2|$.

Define
$
\widetilde D(\theta)
=\mathbb E_x\sum_{y}\pi_{\mathrm{ref}}(y\mid x)\,
   \bigl|r_\theta(x,y)-r_{\theta^\star}(x,y)\bigr|.
$
Then
$
\widetilde D(\theta)
\le \frac{1}{m_f}
    \mathbb E_x\sum_{y}\pi_{\mathrm{ref}}(y\mid x)\,
         \bigl|u_\theta(x,y)-u_{\theta^\star}(x,y)\bigr|
\le \frac{1}{m_f}\,\|u_\theta-u_{\theta^\star}\|_2,
$
using Cauchy-Schwarz in the last step.

Since $\pi_\theta(y\mid x)=\pi_{\mathrm{ref}}(y\mid x)r_\theta(x,y)$ and $\pi_{\mathrm{ref}}\le 1$,
$
D(\theta)
=\mathbb E_x\sum_{y}\pi_{\mathrm{ref}}(y\mid x)^2
  \bigl|r_\theta(x,y)-r_{\theta^\star}(x,y)\bigr|
\le \widetilde D(\theta),
$
where the extra factor of $\pi_{\mathrm{ref}}(y\mid x)$ appears because our $L_1$ norm for $(x,y)$-functions is taken with respect to $P_x\times\pi_{\mathrm{ref}}$ (see \cref{def:norms}). This yields the claim.
\end{proof}

\subsection*{Step 3: Lipschitz dependence of $u_{\beta,h}$ on $h$ for fixed $\beta$}

\begin{lemma}\label{lem:u-Lip-fixed-beta-brief}
Fix $\beta>0$. Under the bounds $m_f\le f''\le M_f$ along the relevant ratios, for any two bounded $h_1,h_2$,
\[
\|u_{\beta,h_1}-u_{\beta,h_2}\|_2
\ \le\ \frac{1}{\beta}\sqrt{\frac{M_f}{m_f}}\,
      \|h_1-h_2\|_2,
\]
where $u_{\beta,h}(x,y)=f'(r_{\beta,h}(x,y))$ and $r_{\beta,h}=\pi_{\beta,h}/\pi_{\mathrm{ref}}$.
\end{lemma}

\begin{proof}
Fix $x$ and for brevity omit it from arguments.
For a given $h$, write $u(y)=u_{\beta,h}(y)=\beta^{-1}(h(y)-\lambda)$, where $\lambda=\lambda_{\beta,h}(x)$.
Let $g=(f')^{-1}$, so $r(y)=g(u(y))$ and the normalization constraint is
$
\sum_{y\in\mathcal Y}\pi_{\mathrm{ref}}(y\mid x)\,g(u(y))=1.
$

Consider a perturbation $\Delta h$, with corresponding $\Delta u$ and $\Delta\lambda$ at fixed $\beta$.
From $u(y)=\beta^{-1}(h(y)-\lambda)$ we get
$
\Delta u(y)=\beta^{-1}\bigl(\Delta h(y)-\Delta\lambda\bigr).
$
Differentiating the normalization constraint yields
$
0=\sum_{y}\pi_{\mathrm{ref}}(y\mid x)\,g'(u(y))\,\Delta u(y)
=\sum_{y}w(y)\Delta u(y),
$
where $w(y)=\pi_{\mathrm{ref}}(y\mid x)g'(u(y))$.
Using the bounds $m_f\le f''\le M_f$ and $\pi_{\mathrm{ref}}\in[\varepsilon,1]$, we have
$
\frac{\varepsilon}{M_f}\ \le\ w(y)\ \le\ \frac{1}{m_f}.
$

Substituting $\Delta u(y)=eta^{-1}(\Delta h(y)-\Delta\lambda)$ gives
$
0=\frac{1}{\beta}\Bigl(\sum_y w(y)\Delta h(y)-\Delta\lambda\sum_y w(y)\Bigr),
$
hence
$
\Delta\lambda
=\frac{\sum_y w(y)\Delta h(y)}{\sum_y w(y)}.
$
Thus
$
\Delta u(y)
=\frac{1}{\beta}\Bigl(\Delta h(y)-\overline{\Delta h}\Bigr),
\qquad
\overline{\Delta h}=\frac{\sum_y w(y)\Delta h(y)}{\sum_y w(y)}.
$

Define weighted norms
$
\|v\|_w^2=\sum_y w(y)v(y)^2,
~
\|v\|_\pi^2=\sum_y\pi_{\mathrm{ref}}(y\mid x)v(y)^2.
$
The map $P_w:\Delta h\mapsto\overline{\Delta h}$ is the orthogonal projection in $\|\cdot\|_w$ onto the constants, so
$
\|\Delta u\|_w
=\frac{1}{\beta}\|(I-P_w)\Delta h\|_w
\le \frac{1}{\beta}\|\Delta h\|_w.
$
Since $w$ and $\pi_{\mathrm{ref}}$ are comparable, for any $v$,
$
\|v\|_\pi^2\le M_f\|v\|_w^2,
~
\|v\|_w^2\le \frac{1}{m_f}\|v\|_\pi^2.
$
Therefore
$
\|\Delta u\|_\pi
\le \sqrt{M_f}\|\Delta u\|_w
\le \frac{\sqrt{M_f}}{\beta}\|\Delta h\|_w
\le \frac{1}{\beta}\sqrt{\frac{M_f}{m_f}}\|\Delta h\|_\pi.
$
This bound holds pointwise in $x$. Squaring and integrating over $x$ yields the result.
\end{proof}

\subsection*{Step 4: Affine invariance and rescaling}

We next use the affine invariance of the construction.

\begin{lemma}\label{lem:affine-rescale-brief}
Let $h$, $a>0$, and $b:\mathcal X\to\mathbb R$ be given.
Define $\tilde h(x,y)=a\,h(x,y)+b(x)$.
Then
$\pi_{a\beta,\tilde h}=\pi_{\beta,h}$ for all $\beta>0$.
Consequently, $
\beta_{\kappa,\tilde h}=a\beta_{\kappa,h}$ and $\pi_{\kappa,\tilde h}=\pi_{\kappa,h}.$
\end{lemma}

\begin{proof}
Set $\tilde\beta=a\beta$ and
$
\lambda_{\tilde\beta,\tilde h}(x)
=a\lambda_{\beta,h}(x)+b(x).
$
Then
$
\tilde\beta^{-1}\bigl(\tilde h(x,y)-\lambda_{\tilde\beta,\tilde h}(x)\bigr)
=\beta^{-1}\bigl(h(x,y)-\lambda_{\beta,h}(x)\bigr),
$
so
$
\pi_{\tilde\beta,\tilde h}(y\mid x)
=\pi_{\mathrm{ref}}(y\mid x)\,
 (f')^{-1}\!\Bigl(\beta^{-1}\bigl(h(x,y)-\lambda_{\beta,h}(x)\bigr)\Bigr)
=\pi_{\beta,h}(y\mid x).
$

For the divergence constraint,
$
\mathbb E_x D_f\bigl(\pi_{a\beta,\tilde h}(\cdot\mid x)\,\big\|\,\pi_{\mathrm{ref}}(\cdot\mid x)\bigr)
=\mathbb E_x D_f\bigl(\pi_{\beta,h}(\cdot\mid x)\,\big\|\,\pi_{\mathrm{ref}}(\cdot\mid x)\bigr).
$
If $\beta_{\kappa,h}$ is the unique $\beta$ such that the right-hand side equals $\kappa$, then the unique $\tilde\beta$ such that the left-hand side equals $\kappa$ is $\beta_{\kappa,\tilde h}=a\beta_{\kappa,h}$, and the induced policies coincide at the constraint:
$
\pi_{\kappa,\tilde h}
=\pi_{\beta_{\kappa,\tilde h},\tilde h}
=\pi_{\beta_{\kappa,h},h}
=\pi_{\kappa,h}.
$
\end{proof}

\subsection*{Step 5: Stability of the calibration equation}

It remains to justify that solving the calibration equation is locally stable in the potential. For any potential $h$, write
\[
\Phi_h(\beta)
=\mathbb E_xD_f\!\bigl(\pi_{\beta,h}(\cdot\mid x)\,\big\|\,\pi_{\mathrm{ref}}(\cdot\mid x)\bigr),
\qquad
u_{\beta,h}=f'\!\left(\frac{\pi_{\beta,h}}{\pi_{\mathrm{ref}}}\right).
\]
Let $h^\star=h_{\theta^\star}$ and $\beta^\star=\beta_{\kappa,h^\star}$.

\begin{lemma}\label{lem:calibration-stability-local}
There are constants $r_0>0$ and $C_0<\infty$ such that, for every affine representative $\bar h=a h_\theta+b$ with $a>0$ and measurable $b:\mathcal X\to\mathbb R$,
\[
\|\bar h-h^\star\|_2\le r_0
\quad\Longrightarrow\quad
\|u_{\kappa,\bar h}-u_{\kappa,h^\star}\|_2
\le C_0\|\bar h-h^\star\|_2 .
\]
\end{lemma}

\begin{proof}
Fix $x$ and suppress it from notation. Let $g=(f')^{-1}$, so
$r_{\beta,h}(y)=g(u_{\beta,h}(y))$ and
\[
\sum_y\pi_{\mathrm{ref}}(y\mid x)g(u_{\beta,h}(y))=1.
\]
As in Step~3, differentiating the normalization equation with respect to $\beta$ gives
\begin{equation}\label{eq:local-calib-derivative}
\frac{d}{d\beta}
D_f\!\bigl(\pi_{\beta,h}(\cdot\mid x)\,\big\|\,\pi_{\mathrm{ref}}(\cdot\mid x)\bigr)
=-\frac{A_{\beta,h}(x)}{\beta^3}
\Var_{p_{\beta,h}(\cdot\mid x)}\!\bigl(h(x,\cdot)\bigr),
\end{equation}
where
\[
A_{\beta,h}(x)=\sum_y\pi_{\mathrm{ref}}(y\mid x)g'(u_{\beta,h}(x,y)),
\qquad
p_{\beta,h}(y\mid x)
=\frac{\pi_{\mathrm{ref}}(y\mid x)g'(u_{\beta,h}(x,y))}{A_{\beta,h}(x)} .
\]
Indeed, if $w(y)=\pi_{\mathrm{ref}}(y\mid x)g'(u_{\beta,h}(y))$, then
\[
\frac{d}{d\beta}
D_f\!\bigl(\pi_{\beta,h}(\cdot\mid x)\,\big\|\,\pi_{\mathrm{ref}}(\cdot\mid x)\bigr)
=\sum_y w(y)u_{\beta,h}(y)\partial_\beta u_{\beta,h}(y),
\]
and substituting the derivative of the normalization constraint gives \cref{eq:local-calib-derivative}.

Set
\[
v^\star=\mathbb E_x\Var_{\pi_{\mathrm{ref}}(\cdot\mid x)}\!\bigl(h^\star(x,\cdot)\bigr).
\]
We have $v^\star>0$: otherwise $h^\star(x,y)$ is an $x$-dependent constant on the support of $\pi_{\mathrm{ref}}(\cdot\mid x)$, so $\pi_{\beta,h^\star}=\pi_{\mathrm{ref}}$ for every $\beta$, contradicting the active constraint $\Phi_{h^\star}(\beta^\star)=\kappa>0$.

Let $I=[\beta^\star/2,2\beta^\star]$. For $r>0$, define the local ratio set
\[
\mathcal R(r)=
\left\{
r_{\beta,\bar h}(x,y):
\beta\in I,\ \bar h=a h_\theta+b,\ a>0,\ \|\bar h-h^\star\|_2\le r
\right\}.
\]
Choose $r_0$ small enough that the local curvature constants
\[
m_0=\inf_{r\in\mathcal R(r_0)}f''(r),
\qquad
M_0=\sup_{r\in\mathcal R(r_0)}f''(r)
\]
satisfy $0<m_0\le M_0<\infty$; this is the local bounded-smoothness condition from \cref{asmp:basic}, and $x$-dependent shifts do not affect the ratios. Thus, for every $\beta\in I$ and every $\bar h$ with $\|\bar h-h^\star\|_2\le r_0$,
\[
m_0\le f''(r_{\beta,\bar h}(x,y))\le M_0
\qquad\text{for all }(x,y).
\]
Hence $g'(u_{\beta,\bar h})\in[1/M_0,1/m_0]$, so
\[
p_{\beta,\bar h}(y\mid x)\ge\frac{m_0}{M_0}\pi_{\mathrm{ref}}(y\mid x),
\qquad
A_{\beta,\bar h}(x)\ge\frac{1}{M_0}.
\]
Also, the map $h\mapsto\{\mathbb E_x\Var_{\pi_{\mathrm{ref}}(\cdot\mid x)}(h(x,\cdot))\}^{1/2}$ is $1$-Lipschitz in $\|\cdot\|_2$, since conditional centering is an orthogonal projection. Thus, if $\|\bar h-h^\star\|_2\le\sqrt{v^\star}/2$, then
\[
\mathbb E_x\Var_{\pi_{\mathrm{ref}}(\cdot\mid x)}\!\bigl(\bar h(x,\cdot)\bigr)
\ge \frac{v^\star}{4}.
\]
Combining this with \cref{eq:local-calib-derivative}, for all $\beta\in I$,
\begin{equation}\label{eq:local-calib-derivative-lower}
-\Phi'_{\bar h}(\beta)
\ge
\frac{m_0v^\star}{4M_0^2(2\beta^\star)^3}
=:c_0>0 .
\end{equation}

Next, Step~3 implies the fixed-$\beta$ bound
\[
\|u_{\beta,\bar h}-u_{\beta,h^\star}\|_2
\le C_1\|\bar h-h^\star\|_2
\qquad(\beta\in I),
\]
with a constant depending only on the same local curvature bounds and $I$. Since $u\mapsto f(g(u))$ has bounded derivative on the corresponding compact range,
\begin{equation}\label{eq:local-calib-phi-lip}
\sup_{\beta\in I}\abs{\Phi_{\bar h}(\beta)-\Phi_{h^\star}(\beta)}
\le C_2\|\bar h-h^\star\|_2 .
\end{equation}
Because $\Phi_{h^\star}$ is strictly decreasing, $\Phi_{h^\star}(\beta^\star/2)>\kappa>\Phi_{h^\star}(2\beta^\star)$. Shrinking $r_0$ if necessary, \cref{eq:local-calib-phi-lip} ensures that the root $\beta_{\kappa,\bar h}$ lies in $I$. Since $\Phi_{\bar h}(\beta_{\kappa,\bar h})=\kappa=\Phi_{h^\star}(\beta^\star)$, \cref{eq:local-calib-derivative-lower,eq:local-calib-phi-lip} give
\[
\abs{\beta_{\kappa,\bar h}-\beta^\star}
\le c_0^{-1}\abs{\Phi_{\bar h}(\beta^\star)-\Phi_{h^\star}(\beta^\star)}
\le \frac{C_2}{c_0}\|\bar h-h^\star\|_2 .
\]
Finally, $u_{\beta,\bar h}$ is Lipschitz in $\beta$ on $I$ under the same local bounds, and Step~3 gives
\[
\begin{aligned}
\|u_{\kappa,\bar h}-u_{\kappa,h^\star}\|_2
&\le
\|u_{\beta_{\kappa,\bar h},\bar h}-u_{\beta^\star,\bar h}\|_2
+\|u_{\beta^\star,\bar h}-u_{\beta^\star,h^\star}\|_2\\
&\le C_3\abs{\beta_{\kappa,\bar h}-\beta^\star}
+C_1\|\bar h-h^\star\|_2
\le C_0\|\bar h-h^\star\|_2 ,
\end{aligned}
\]
which proves the lemma.
\end{proof}

\subsection*{Step 6: Completing the proof of the theorem}

\begin{proof}[Proof of \Cref{thm:policy-stability-metric}]
Fix $\theta\in\Theta$ and let $h=h_\theta$. If $\rho(h_\theta)<r_0/2$, choose $a>0$ and measurable $b:\mathcal X\to\mathbb R$ such that, with $\bar h=a h+b$,
$\|\bar h-h^\star\|_2\le 2\rho(h_\theta)$. By Lemma~\ref{lem:affine-rescale-brief}, $\pi_{\kappa,\bar h}=\pi_{\kappa,h}$, and by Lemma~\ref{lem:calibration-stability-local},
\[
\|u_\theta-u_{\theta^\star}\|_2
=\|u_{\kappa,\bar h}-u_{\kappa,h^\star}\|_2
\le 2C_0\rho(h_\theta).
\]
Combining this with Lemma~\ref{lem:D-vs-u-brief} gives
$
D(\theta)
\ \le\ \frac{1}{m_f}\|u_\theta-u_{\theta^\star}\|_2
\ \le\ \frac{2C_0}{m_f}\rho(h_\theta)
$.
If instead $\rho(h_\theta)\ge r_0/2$, then the total-variation bound gives
$D(\theta)\le2\le (4/r_0)\rho(h_\theta)$.
Combining the two cases proves the result after increasing the constant.
\qedhere
\end{proof}

\subsection{Proof of \Cref{thm:aip-id-sufficient}}
\label{app:aip-proof}

The key observation is that whenever we have $x,y_1,y_2,y_3$ with $t_{\thetas}(x;y_i,y_j)=m(t_{\theta}(x;y_i,y_j))$ for $(i,j)\in\{(1,2),(2,3),(1,3)\}$ then
\begin{align*}
m(t_{\theta}(x;y_1,y_2))+m(t_{\theta}(x;y_2,y_3))&=
t_{\thetas}(x;y_1,y_2)+t_{\thetas}(x;y_2,y_3)
\\&=
t_{\thetas}(x;y_1,y_3)
\\&=
m(t_{\theta}(x;y_1,y_3))
\\&=
m(t_{\theta}(x;y_1,y_2)+t_{\theta}(x;y_2,y_3)).
\end{align*}
This imposes Cauchy's functional equation (additivity) on a monotonic function. When the domain spanned by $t_{\theta}(x;y_i,y_j)$ is sufficiently rich, the only feasible solution for $m$ is a linear function \citep{Aczel1966}. Note that we need $\abs{\Ycal}\geq3$ to be able to make this rich enough, else all we can possibly get is the oddness of $m$.

We now proceed with the proof.

\begin{proof}
Let the premise of \cref{asm:aip-id} hold: for some $\theta\in\Theta$ and monotone nondecreasing $m$, we have
$t_{\thetas}(x;y_1,y_0)=m(t_{\theta}(x;y_1,y_0))$ $P$-almost surely. Let $M_\theta=\|t_\theta\|_\infty$.
If $M_\theta=0$, then $t_\theta=0$ $P$-a.s.; the premise and symmetry imply $t_{\thetas}=0$ $P$-a.s., so the conclusion below holds with $a=0$. We therefore assume $M_\theta>0$.

The support-richness condition and the algebra in the display above imply that, for every $u\in(-M_\theta,M_\theta)$ and every $v\in(-\delta,\delta)$ with $u+v\in(-M_\theta,M_\theta)$,
\[
m(u+v)=m(u)+m(v).
\]
In particular, $m(0)=0$. Now take any $u,v\in(-M_\theta,M_\theta)$ with $u+v\in(-M_\theta,M_\theta)$, and choose $N$ large enough that $|v|/N<\delta$. Since $(-M_\theta,M_\theta)$ is convex, $u+k v/N\in(-M_\theta,M_\theta)$ for $k=0,\ldots,N$. Iterating the local additivity gives
\[
m(u+v)=m(u)+N\,m(v/N),
\qquad
m(v)=N\,m(v/N),
\]
and hence $m(u+v)=m(u)+m(v)$.

Thus $m$ is monotone and additive on an interval containing zero, so by the standard monotone solution to Cauchy's equation it is linear on this interval: $m(r)=a r$ for some $a\ge0$. Since $t_\theta(w)\in[-M_\theta,M_\theta]$ almost surely, we have $t_{\thetas}=a t_\theta$ $P$-almost surely.

Finally, for $P_x$-almost every $x$, fix $y_x\in\mathrm{Support}(\piref(\cdot\mid x))$ and set $b(x)=h_{\theta^\star}(x,y_x)-a h_\theta(x,y_x)$. For any $y\in\mathrm{Support}(\piref(\cdot\mid x))$, the coverage condition in \cref{asmp:basic} gives the pair $(y_x,y)$ positive $P(\cdot\mid x)$ support, so $t_{\thetas}(x;y,y_x)=a t_\theta(x;y,y_x)$. Therefore $h_{\theta^\star}(x,y)=a h_\theta(x,y)+b(x)$ for $P_x\times\piref$-almost every $(x,y)$.
\end{proof}

\subsection{Theoretical Properties of PSPO}
\label{app:pspo}

We first prove a uniform law of large numbers. Then, we prove upper semi-continuity. Then, we prove all maximizers are scalings of $\tst$. Finally, we leverage compactness of the rescaled indices to establish the convergence.

In the following, define
\begin{align*}
L_n^{\rm PSPO}(t)&=\sup_{\Psi\uparrow}\;
\frac{1}{n}\sum_{i=1}^n\ell(t,\Psi;w_i,z_i),& L^{\rm PSPO}(t)&=\sup_{\Psi\uparrow}\E\big[\ell(t,\Psi;w,z)\big],\\
L_{n,\alpha}^{\rm PSPO}(t)&=\sup_{\substack{\Psi\uparrow\\ \Psi:\R\to[\alpha,1-\alpha]}}\;
\frac{1}{n}\sum_{i=1}^n\ell(t,\Psi;w_i,z_i),& L_{\alpha}^{\rm PSPO}(t)&=\sup_{\substack{\Psi\uparrow\\ \Psi:\R\to[\alpha,1-\alpha]}}\E\big[\ell(t,\Psi;w,z)\big],\\
\bar{\Tcal}&=\{t/\|t\|:t\in\Tcal\backslash\{0\}\},&\bar{\Tcal}_0&=\bar{\Tcal}\cup(\Tcal\cap\{0\}),\\
\omega(p,q)&=p\log q+(1-p)\log(1-q),&\varpi(p)&=\omega(p,p).\end{align*}

\subsubsection{Uniform Law of Large Numbers}

\begin{lemma}[Isotonic profiling ULLN for $L_n^{\rm PSPO}$]
\label{lem:isotonic-cont}
Suppose \cref{asmp:basic} holds so that $M=\sup_{t\in\Tcal}\|t\|_\infty<\infty$. Moreover, suppose that the class $\{\mathbb I[t(w)\leq\tau]:t\in\mathcal T,~\tau\in[-M,M]\}$ is 
$P$-Glivenko-Cantelli.  
Then,
\[
\sup_{t\in\Tcal}\Big|L_n^{\rm PSPO}(t)-L^{\rm PSPO}(t)\Big|\ \xrightarrow{p}\ 0.
\]
\end{lemma}
\begin{proof}
Let any monotone $\Psi:[-M,M]\to[0,1]$ be given. It has variation bounded by $1$ so it has a Lebesgue-Stieltjes representation as $\Psi(u)=\Psi(-M)+\int_{(-M,M]}\indic{u>v}\mu_\Psi(dv)$ for a positive measure $\mu_\Psi$ on $[-M,M]$ with mass at most 1. For any $t\in\Tcal$, we can therefore write $\Psi(t(w))=\Psi(-M)+\int_{(-M,M]}\indic{t(w)>v}\mu_\Psi(dv)$. Therefore,
\begin{align*}
\sup_{\Psi\uparrow,t\in\Tcal}\abs{(P_n-P)\Psi(t(w))}&\leq
\sup_{\Psi\uparrow,t\in\Tcal}\abs{\int_{(-M,M]}(P_n-P)\indic{t(w)>v}\mu_\Psi(dv)}\\
&\leq
\sup_{t\in\Tcal,v\in[-M,M]}\abs{(P_n-P)\indic{t(w)>v}}.
\end{align*}
Since the class of thresholded indices is $P$-Glivenko-Cantelli, we conclude that $\{\Psi(t(w)):\Psi\uparrow,\,t\in\Tcal\}$ is $P$-Glivenko-Cantelli.

Fix $\alpha\in(0,1/2)$. Since $q\mapsto\log q$ and $q\mapsto\log(1-q)$ are Lipschitz on $[\alpha,1-\alpha]$, the class
\[
\{(w,z)\mapsto \ell(t,\Psi;w,z):t\in\Tcal,\ \Psi\uparrow,\ \Psi:\R\to[\alpha,1-\alpha]\}
\]
is $P$-Glivenko-Cantelli: an $L_1(P_w)$ bracket for $\Psi(t(w))$ gives $L_1(P)$ brackets for $z\log\Psi(t(w))$ and $(1-z)\log(1-\Psi(t(w)))$ with widths enlarged by at most $\alpha^{-1}$. Hence
\[
e_{n,\alpha}:=\sup_{t\in\Tcal}\abs{L_{n,\alpha}^{\rm PSPO}(t)-L_{\alpha}^{\rm PSPO}(t)}\xrightarrow{p}0.
\]

Let $\eta(w)=\E[z\mid w]$. Since $\Psi^\star$ is strictly increasing and $\|t^\star\|_\infty\le M$, $\eta(w)$ takes values in $[\varepsilon,1-\varepsilon]$ for $\varepsilon=\min\{\Psi^\star(-M),1-\Psi^\star(M),0.25\}>0$. If $\alpha\le\varepsilon$, then $L_{\alpha}^{\rm PSPO}(t)=L^{\rm PSPO}(t)$ for every $t\in\Tcal$: for any $q\in[0,1]$, clipping $q$ to $[\alpha,1-\alpha]$ moves it weakly closer to $\eta(w)$, and $q\mapsto\omega(\eta(w),q)$ is increasing on $[0,\eta(w)]$ and decreasing on $[\eta(w),1]$.

It remains to compare the unconstrained and clipped empirical profiles. Let $q_\alpha=\min\{\max\{q,\alpha\},1-\alpha\}$ and $\delta_\alpha=-\log(1-\alpha)$. For every $q\in[0,1]$ and $z\in\{0,1\}$,
\[
z\log q+(1-z)\log(1-q)
\le
z\log q_\alpha+(1-z)\log(1-q_\alpha)+\delta_\alpha.
\]
Clipping a nondecreasing $\Psi$ preserves monotonicity, so for every $t$,
\[
L_{n,\alpha}^{\rm PSPO}(t)\le L_n^{\rm PSPO}(t)\le L_{n,\alpha}^{\rm PSPO}(t)+\delta_\alpha.
\]
Therefore, for every fixed $\alpha\le\varepsilon$,
\[
\sup_{t\in\Tcal}\abs{L_n^{\rm PSPO}(t)-L^{\rm PSPO}(t)}
\le e_{n,\alpha}+\delta_\alpha.
\]
Given any $\xi>0$, choose $\alpha\le\varepsilon$ with $\delta_\alpha<\xi/2$ and then take $n$ large enough that $e_{n,\alpha}<\xi/2$ with probability tending to one. This proves the claim.
\end{proof}

Since both $L^{\rm PSPO}$ and $L_n^{\rm PSPO}$ are scale invariant, an immediate corollary of \cref{lem:isotonic-cont} is that
\[
\sup_{t\in\bar{\Tcal}}\Big|L_n^{\rm PSPO}(t)-L^{\rm PSPO}(t)\Big|\ \xrightarrow{p}\ 0.
\]

\subsubsection{Upper Semi-Continuity}
\begin{lemma}
\label{lem:pspo-usc}
Suppose \cref{asmp:basic} holds and that $t(w)$ has no atoms for each $t\in\Tcal\backslash\{0\}$. 
Then, $L^{\rm PSPO}(t)$ is upper semi-continuous on $\bar\Tcal$.
\end{lemma}

\begin{proof}
Let $t_m$ be a sequence with $\|t_m-t\|\to0$.
Let $\varepsilon=\min\{\Psi^\star(-M),1-\Psi^\star(M),0.25\}>0$. As in the proof of \cref{lem:isotonic-cont}, the population profile is unchanged by restricting $\Psi$ to $[\varepsilon,1-\varepsilon]$. For each $m$, pick monotonic $\Psi_m:[-M,M]\to[\varepsilon,1-\varepsilon]$ such that $\E[\ell(t_m,\Psi_m;w,z)]\ge L^{\rm PSPO}(t_m)-1/m$.
Since $t_m\to t$ in $L_2$, along some subsequence we have $t_{m_i}\to t$ a.s.
Since each $\Psi_{m_i}$ is nondecreasing and takes values in $[\varepsilon,1-\varepsilon]$, Helly's selection theorem yields a further subsequence such that $\Psi_{m_{i_j}}(u)\to \Psi(u)$ for all continuity points $u$ of $\Psi$ for some nondecreasing $\Psi:[-M,M]\to[\varepsilon,1-\varepsilon]$.
Fix $w_0$ such that $t_{m_{i_j}}(w_0)\to t(w_0)=u$ and $u$ is a continuity point of $\Psi$.
Since the set of discontinuities of a monotone function is countable and hence has dense complement, choose continuity points $u_k\uparrow u$ and $u'_k\downarrow u$ of $\Psi$.
For each $k$, for all sufficiently large $j$, $u_k<t_{m_{i_j}}(w_0)<u'_k$, so by monotonicity $\Psi_{m_{i_j}}(u_k)\le \Psi_{m_{i_j}}(t_{m_{i_j}}(w_0))\le \Psi_{m_{i_j}}(u'_k)$.
Letting $j\to\infty$ gives $\Psi(u_k)\le \liminf_j \Psi_{m_{i_j}}(t_{m_{i_j}}(w_0))\le \limsup_j \Psi_{m_{i_j}}(t_{m_{i_j}}(w_0))\le \Psi(u'_k)$, and then taking $k\to\infty$ yields $\Psi_{m_{i_j}}(t_{m_{i_j}}(w_0))\to\Psi(u)$ by continuity of $\Psi$ at $u$.
Because $\Psi$ has only countably many discontinuities and $t(w)$ has no atoms, $\PP(t(w)\ \text{is a continuity point of }\Psi)=1$, hence $\Psi_{m_{i_j}}(t_{m_{i_j}}(w))\to \Psi(t(w))$ a.s.
Since the log loss is bounded on $[\varepsilon,1-\varepsilon]$, dominated convergence gives
$
\lim_{j\to\infty}\E[\ell(t_{m_{i_j}},\Psi_{m_{i_j}};w,z)] \;=\; \E[\ell(t,\Psi;w,z)] \;\le\; L^{\rm PSPO}(t),
$
which proves that $\limsup_m L^{\rm PSPO}(t_m)\le L^{\rm PSPO}(t)$.
\end{proof}

\subsubsection{Maximizers are scalings of $\tst$}

\begin{lemma}\label{lem:pspo-scaling} Under the conditions of \cref{thm:pspo}, $t$ is a maximizer of $L^{\rm PSPO}(t)$ if and only if $t\neq0$ has $\tst(w)=a t(w)$ almost surely for some $a\geq0$.\end{lemma}
\begin{proof}
Let $\eta(w)=\E[z\mid w]=\Psis(\tst(w))$ and $g^\star_t(u)=\E[z\mid t(w)=u]=\E[\eta(w)\mid t(w)=u]$. 
Then, $\sup_{t\in\Tcal}L^{\rm PSPO}(t)=L^{\rm PSPO}(\tst)=\E\omega(z,\eta(w))=\E\varpi(\eta(w))$. 
Let $J_t=\sup_{q:\mathbb R\to(0,1)}\E\omega(z,q(t(w)))$ where $q$ is unrestricted. 
Note that $\mathbb E[\omega(z,q)\mid t(w)]=\omega(g^\star_t(t(w)),q)$, which is strictly concave in $q$ with unique maximizer $q=g^\star_t(t(w))$. 
Therefore, $J_t=\E\omega(z,g^\star_t(t(w)))=\E\varpi(g^\star_t(t(w)))$, where the latter equality is by iterating expectations on $t(w)$. 
Since $L^{\rm PSPO}(t)$ is defined by restricting to monotone $\Psi$ while $J_t$ is defined with $q$ unrestricted, we have $L^{\rm PSPO}(t)\leq J_t$. 
By Jensen's, $\E[\varpi(\eta(w))\mid t(w)]\geq\varpi(\E[\eta(w)\mid t(w)])=\varpi(g^\star_t(t(w)))$. 
Taking expectations, $L^{\rm PSPO}(\tst)\geq J_t$. 
In all, $L^{\rm PSPO}(t)\leq J_t\leq L^{\rm PSPO}(\tst)$.

Suppose $L^{\rm PSPO}(t)=L^{\rm PSPO}(\tst)$. 
Then $L^{\rm PSPO}(t)= J_t= L^{\rm PSPO}(\tst)$. 
We have $L^{\rm PSPO}(t)=\sup_{q\uparrow:\mathbb R\to(0,1)}\E\omega(z,q(t(w)))=\sup_{q:\mathbb R\to(0,1)}\E\omega(z,q(t(w)))=J_t$, while $q\mapsto\E[\omega(z,q)\mid t(w)=u]$ is strictly concave with unique maximizer $q=g^\star_t(u)$. 
Hence, any maximizer in the definition of $J_t$ must be almost everywhere equal to $g^\star_t$. 
If $g^\star_t$ were not equal to a monotonic function almost everywhere, we would have $L^{\rm PSPO}(t)<J_t$ because $L^{\rm PSPO}(t)$ is restricted to the closed convex subset of monotonic functions. 
We conclude that $\tst(w)=m(t(w))$ almost surely for a monotonic function $m$ such that $m(t(w))={\Psis}^{-1}(g^\star_t(t(w)))$ almost surely.
Moreover, $\Var(\eta(w))\neq0$, so $\tst$ is not constant, and therefore $t(w)$ cannot be constant. Since $\E t(w)=0$, this means $t\neq0$. \Cref{asm:aip-id} completes the proof.
\end{proof}

\subsubsection{Proof of \cref{thm:pspo}}
\begin{proof}[Proof of Theorem~\ref{thm:pspo}]
Let $\bar t^\star=t^\star/\|t^\star\|_2\in\bar \Tcal$ and note that $L^{\rm PSPO}(\bar t^\star)=L^{\rm PSPO}(t^\star)$ by scale invariance.
By \cref{lem:pspo-usc}, $L^{\rm PSPO}$ is upper semicontinuous on $\bar \Tcal$.
By \cref{lem:pspo-scaling}, $\bar t^\star$ is the unique maximizer of $L^{\rm PSPO}$ over $\bar \Tcal$: if $\bar t\in\bar \Tcal$ and $L^{\rm PSPO}(\bar t)=L^{\rm PSPO}(\bar t^\star)$, write $\bar t=u/\|u\|_2$ with $u\in \Tcal\setminus\{0\}$; then $L^{\rm PSPO}(u)=L^{\rm PSPO}(\bar t)=L^{\rm PSPO}(\bar t^\star)=L^{\rm PSPO}(t^\star)$, hence $u=a t^\star$ for some $a>0$, so $\bar t=\bar t^\star$.

Fix $\zeta>0$ and set $A_\zeta=\{\bar t\in\bar \Tcal:\ \|\bar t-\bar t^\star\|_2\ge \zeta\}$, which is compact.
Upper semicontinuity implies $m_\zeta=\max_{\bar t\in A_\zeta}L^{\rm PSPO}(\bar t)$ is attained, and uniqueness gives $\Delta_\zeta=L^{\rm PSPO}(\bar t^\star)-m_\zeta>0$.
On the event $E_n=\{\sup_{\bar t\in\bar \Tcal}|L^{\rm PSPO}_n(\bar t)-L^{\rm PSPO}(\bar t)|\le \Delta_\zeta/3\}$ we have
$L^{\rm PSPO}_n(\bar t)\le L^{\rm PSPO}(\bar t)+\Delta_\zeta/3\le m_\zeta+\Delta_\zeta/3=L^{\rm PSPO}(\bar t^\star)-2\Delta_\zeta/3\le L^{\rm PSPO}_n(\bar t^\star)-\Delta_\zeta/3$
for all $\bar t\in A_\zeta$, hence every maximizer of $L^{\rm PSPO}_n$ over $\bar \Tcal$ lies in $\bar \Tcal\setminus A_\zeta$, i.e.\ within $\zeta$ of $\bar t^\star$.

If $0\notin \Tcal$, then $t_n\neq 0$ a.s. and $\bar t_n=t_n/\|t_n\|_2\in\bar \Tcal$ is well-defined; scale invariance yields $L^{\rm PSPO}_n(\bar t_n)=L^{\rm PSPO}_n(t_n)=\max_{t\in \Tcal}L^{\rm PSPO}_n(t)=\max_{\bar t\in\bar \Tcal}L^{\rm PSPO}_n(\bar t)$, so $\bar t_n$ is a maximizer of $L^{\rm PSPO}_n$ over $\bar \Tcal$.
If $0\in \Tcal$, let $\gamma=\tfrac12(L^{\rm PSPO}(t^\star)-L^{\rm PSPO}(0))>0$; on the event $\{|L^{\rm PSPO}_n(\bar t^\star)-L^{\rm PSPO}(\bar t^\star)|\le \gamma\}\cap\{|L^{\rm PSPO}_n(0)-L^{\rm PSPO}(0)|\le \gamma\}$ we have $L^{\rm PSPO}_n(\bar t^\star)>L^{\rm PSPO}_n(0)$ and thus $t_n\neq 0$, so the same conclusion holds with probability tending to $1$.
Therefore, for all large $n$ on an event of probability tending to $1$, $\bar t_n$ is a maximizer of $L^{\rm PSPO}_n$ over $\bar \Tcal$ and hence $\|\bar t_n-\bar t^\star\|_2<\zeta$ on $E_n$.
Thus $\PP(\|\bar t_n-\bar t^\star\|_2\ge \zeta)\le \PP(E_n^c)+o(1)\to 0$.

Finally, on $\{t_n\neq 0\}$,
\[
\inf_{a\ge 0}\|a t_n-t^\star\|_2\ \le\ \Big\|\frac{\|t^\star\|_2}{\|t_n\|_2}t_n-t^\star\Big\|_2
=\|t^\star\|_2\,\|\bar t_n-\bar t^\star\|_2.
\]
Since $\zeta>0$ was arbitrary, $\|\bar t_n-\bar t^\star\|_2\to 0$ in probability implies $\inf_{a\ge 0}\|a t_n-t^\star\|_2\to 0$ in probability, which dominates $\rho(h_\theta)$ up to the coverage constant in \cref{asmp:basic}.
\end{proof}

\subsection{Theoretical Properties of OSPO}
\label{sec:ospo-theory}

In this section, we define
\begin{align*}
g^\star_t(u)&=P(z=1\mid t(w)=u),&
\eta(w)&=\E[z\mid w]=\Psis(\tst(w))=g^\star_{\tst}(\tst(w)),
\\
L_n(t,g)
&= \frac{1}{n}\sum_{i=1}^n \ell(t,g_t;w_i,z_i),&
L(t,g)
&= \E\ell(t,g_t;w,z)
\\
L_n^{\rm ideal}(t)&=L_n(t,\gs),~
L_n^{\rm OSPO}(t)=L_n(t,\hat g),&
L^{\rm ideal}(t)&=L(t,\gs),~
L^{\rm OSPO}(t)=L(t,\hat g),\\
\omega(p,q)&=p\log q+(1-p)\log(1-q),&\varpi(p)&=\omega(p,p),\\
\varepsilon&=\min\{\Psis(-M),1-\Psis(M),0.25\}>0,&\tilde\varepsilon&=(1\wedge c_{\rm range})\varepsilon.
\end{align*}

\subsubsection{Proof of \cref{lem:margin}}

\begin{proof}
Since $\Psis$ is strictly increasing, $\E[z\mid w]$ only takes values in $[\varepsilon,1-\varepsilon]$. 
Note that we also have $\varpi''(u)\geq \varpi''_{\min}=4$.

We have that $L^{\rm ideal}(\tst)-L^{\rm ideal}(t)=\E[\varpi(\E[z\mid \tst])-\varpi(\E[\E[z\mid \tst]\mid t])]=\E[\E[\varpi(\E[z\mid \tst])-\varpi(\E[\E[z\mid \tst]\mid t])\mid t]]$. By Taylor's theorem $\E[\varpi(\E[z\mid \tst])-\varpi(\E[\E[z\mid \tst]\mid t])\mid t]\geq\frac{\varpi''_{\min}}{2}\mathrm{Var}(\E[z\mid \tst]\mid t)$.

Since $\Psis$ has derivative bounded below by some $c>0$ on $[-M,M]$, mean value theorem implies that $\abs{\Psis(u)-\Psis(u')}\geq c\abs{u-u'}$ for $u,u'\in[-M,M]$. Therefore, $\mathrm{Var}(\E[z\mid \tst]\mid t)=\mathrm{Var}(\Psis(\tst)\mid t)\geq c^2\mathrm{Var}(\tst\mid t)$.

The proof is therefore completed by noting that $\E\mathrm{Var}(\tst\mid t)=\inf_f\|\tst-f(t_\theta)\|^2$ and setting $\lambda_0=c^2_{\rm MSE}c^2\varpi''_{\min}/2$.
\end{proof}

\subsubsection{Proof of \cref{lem:fix-sign}}

\begin{proof}
Define $\mu=\E[t^\star(w) z]$.
Write $t=t_\theta$. Recall $\E[t]=\E[\tst]=0$ since $P$ is symmetric in actions.
Since $z=\Psi^\star(t^\star)$ with $\Psi^\star$ strictly increasing and $\E[t^\star]=0$, we have
\[
\mu=\E[t^\star z]=\E[t^\star \Psi^\star(t^\star)]>0 .
\]

Let $a_0=s/\E[t^2]$ be the unique minimizer of $\E[(t^\star-a t)^2]$, and set $R=t^\star-a_0 t$ so that $r^2=\E[R^2]$, $\E[tR]=0$, and $\E[(t^\star)^2] = a_0^2\E[t^2] + r^2$.
The condition $r\leq\gamma\mu$ with $\gamma<1$ implies $a_0\neq0$: if $a_0=0$, then $r=\|t^\star\|_2\ge \mu$, since $\|z\|_2\le1$.
Without loss of generality, $a_0>0$: if $a_0<0$, apply the argument below to $\tilde t=-t$, for which $\tilde a_0=-a_0>0$, $\tilde s=-s$, and $\tilde{\hat s}=-\hat s$, so $\{\hat s\,s\le0\}=\{\tilde{\hat s}\,\tilde s\le0\}$.

Using $t = (t^\star-R)/a_0$ and $z=\Psi^\star(t^\star)$,
we have $
\E[t z]
= \frac{1}{a_0}\E[(t^\star-R)z]
= \frac{1}{a_0}\bigl(\E[t^\star z] - \E[R z]\bigr).
$
Since $0\le z\le1$, Cauchy-Schwarz gives
$
|\E[R z]|\le \|R\|_2\|z\|_2\le r.
$
We obtain
$
\E[t^\star z] - \E[R z] \;\ge\; \mu - |\E[R z]|
\;\ge\; \mu-r \;\ge\; (1-\gamma)\mu.
$
Hence
$\E[t z]^2 = \frac{(\E[t^\star z]-\E[R z])^2}{a_0^2}
\;\ge\; \frac{(1-\gamma)^2\mu^2}{a_0^2}.$
From $\E[(t^\star)^2] = a_0^2\E[t^2] + r^2$ we also have
$
\E[t^2] = \frac{\E[(t^\star)^2]-r^2}{a_0^2}
\;\le\; \frac{\E[(t^\star)^2]}{a_0^2}.
$

By Chebyshev's inequality,
$
\mathbb P(\hat s\,s\le 0)
=\mathbb P(\hat s\le 0)
= \mathbb P\bigl(\hat s-\E[t z] \le -\E[t z]\bigr)
\le \mathbb P\bigl(|\hat s-\E[t z]|\ge|\E[t z]|\bigr)
\le \frac{\Var(tz)}{\abs{\Dcal}\E[t z]^2}
\le \frac{\E[(tz)^2]}{\abs{\Dcal}\E[t z]^2}.
$
Since $t^2 z^2 \le t^2$, we get
$
\mathbb P(\hat s\,s\le 0)
\le \frac{\E[t^2]}{\abs{\Dcal}\E[t z]^2}
\;\le\; \frac{\E[(t^\star)^2]/a_0^2}{\abs{\Dcal}\, ((1-\gamma)^2\mu^2/a_0^2)}
= \frac{1}{(1-\gamma)^2\abs{\Dcal}}\,
\frac{\E[(t^\star)^2]}{\mu^2}.
$
\end{proof}

\subsubsection{Proof of \cref{thm:ospo-rate}}

Define the $\hat g$-quasi-likelihood difference, scale-equivalence error, and logistic log-loss, respectively:
\[
\Errhatg
=L(t^\star,\hat g)-L(\hat t,\hat g)
\]
\[
\mathcal E(t)=\inf_{a\in\mathbb R}\norm{t^\star-a t}_2,
\]
\[
\phi_z(p)= - z\log p - (1-z)\log(1-p)=-\omega(z,p),
\]
so that $\ell(t,g;w,z)=-\phi_z\big(g_t(t(w))\big)$.

Further, we let $\hat g_t(u)=\hat g_\theta(u)$ when $t=t_\theta$. We also use the shorthand $q_g(t)(w)=g_t(t(w))$, $\hat q(t)(w)=\hat g_t(t(w))$, $q^\star(t)(w)=g^\star_t(t(w))$. We then define the errors
\[
\Delta q(t)=\hat q(t)-q^\star(t),
\qquad
\zeta=\sup_{t\in\Tcal}\norm{\Delta q(t)}_4.
\]
Lastly, let $\hat t=t_{\hat\theta^{\rm OSPO}}$.

\subsubsubsection{Supporting lemmas}

\begin{lemma}[Derivative bounds for the log-loss]\label{lem:phi-bounds}
For $p\in[{\tilde\varepsilon},1-{\tilde\varepsilon}]$ and $z\in\{0,1\}$,
\[
|\phi_z'(p)|\le \frac{1}{{\tilde\varepsilon}},\qquad
|\phi_z''(p)|\le \frac{1}{{\tilde\varepsilon}^2},\qquad
|\phi_z'''(p)|\le \frac{2}{{\tilde\varepsilon}^3}.
\]
In particular, $\phi_z''$ is $(2/{\tilde\varepsilon}^3)$-Lipschitz on $[{\tilde\varepsilon},1-{\tilde\varepsilon}]$.
\end{lemma}
\begin{proof}
Compute $\phi_z'(p)=-(z/p)+(1-z)/(1-p)$, $\phi_z''(p)=z/p^2+(1-z)/(1-p)^2$, and
$\phi_z'''(p)=-2z/p^3+2(1-z)/(1-p)^3$. The bounds follow since $p,(1-p)\ge {\tilde\varepsilon}$.
\end{proof}

\begin{lemma}[Score-zero identity for each fixed $t$]\label{lem:score0}
For any index $t$ and any measurable $a:\mathbb R\to\mathbb R$ with $\E[a(t(w))^2]<\infty$,
\[
\E\Big[\phi_z'\big(q^\star(t)(w)\big)\,a\big(t(w)\big)\Big]=0.
\]
\end{lemma}
\begin{proof}
Write $p^\star(w)=q^\star(t)(w)=\E[z\mid t(w)]$.
Condition on $t(w)$:
\[
\E\big[\phi_z'(p^\star(w))\,a(t(w))\big]
=\E\Big[ a(t(w))\cdot \E\big[\phi_z'(p^\star(w))\mid t(w)\big]\Big].
\]
Since $z\mid t(w)$ is Bernoulli with mean $p^\star(w)$, direct substitution gives
$\E[\phi_z'(p^\star(w))\mid t(w)]=0$.
\end{proof}

\begin{lemma}[Oracle conditional mean is Lipschitz in index error]\label{lem:true-q-lip}
Under \cref{asmp:div,asmp:realizability,asmp:basic,asm:ospo-regularity}, there exists $L_\Psi<\infty$ such that, for every $t\in\Tcal$,
\[
\|q^\star(t)-q^\star(\tst)\|_2
\leq
L_\Psi\,\Ecal(t).
\]
\end{lemma}
\begin{proof}
Let $\eta(w)=\E[z\mid w]=\Psis(\tst(w))$. Then $q^\star(t)(w)=\E[\eta(w)\mid t(w)]$ and $q^\star(\tst)(w)=\eta(w)$.
By \cref{asm:ospo-regularity}, $\Psis$ is Lipschitz on $[-M,M]$; let $L_\Psi$ be a Lipschitz constant.
Let $r_t(w)=\E[\tst(w)\mid t(w)]$. Since $\tst\in[-M,M]$, also $r_t(w)\in[-M,M]$ almost surely. The function $\Psis(r_t(w))$ is measurable with respect to $t(w)$, so the $L_2$ projection property of conditional expectation gives
\[
\|q^\star(t)-q^\star(\tst)\|_2
=\|\E[\eta\mid t]-\eta\|_2
\leq \|\Psis(r_t)-\eta\|_2
\leq L_\Psi\|\tst-r_t\|_2.
\]
Since conditional expectation is also the $L_2$ projection of $\tst$ onto functions of $t(w)$, $\|\tst-r_t\|_2\leq \|\tst-a t\|_2$ for every $a\in\R$. Taking the infimum over $a$ proves the claim.
\end{proof}

\begin{lemma}[Oracle link minimizes log-loss conditional on $t(w)$]\label{lem:oracle-min}
For any index $t$ and any measurable $p(t(w))\in[{\tilde\varepsilon},1-{\tilde\varepsilon}]$,
\[
\E\big[\phi_z(p(t(w)))\big]\ \ge\ \E\big[\phi_z(q^\star(t)(w))\big].
\]
In particular, $L(t,\hat g)\le L(t,g^\star)$ for every $t\in\Tcal$.
\end{lemma}
\begin{proof}
Fix $t$.
For any measurable $p(t(w))\in[{\tilde\varepsilon},1-{\tilde\varepsilon}]$,
$\E\big[\phi_z(p(t(w)))\mid t(w)\big]$
is the Bernoulli cross-entropy between the true conditional mean $\E[z\mid t(w)]=g_t^\star(t(w))$ and prediction $p(t(w))$.
For each realized $t(w)$, this is minimized at $p(t(w))=g_t^\star(t(w))$, i.e.\ at $p(t(w))=q^\star(t)(w)$.
Integrate over $t(w)$ to conclude.
\end{proof}

\begin{lemma}[Quadratic nuisance effect at a fixed $t$]\label{lem:quad-nuis}
For any $t$ such that $q^\star(t)(w),\hat q(t)(w)\in[{\tilde\varepsilon},1-{\tilde\varepsilon}]$ almost surely,
\[
0\ \le\ L(t,g^\star)-L(t,\hat g)\ \le\ \frac{1}{2{\tilde\varepsilon}^2}\,\E\big[(\Delta q(t)(w))^2\big]
\ \le\ \frac{1}{2{\tilde\varepsilon}^2}\,\norm{\Delta q(t)}_4^2.
\]
\end{lemma}
\begin{proof}
Fix $t$ and write $p^\star(w)=q^\star(t)(w)$ and $\hat p(w)=\hat q(t)(w)$, so $\Delta q(t)=\hat p-p^\star$.
By Lemma~\ref{lem:phi-bounds}, for all $p,p'\in[{\tilde\varepsilon},1-{\tilde\varepsilon}]$,
\[
\phi_z(p)\le \phi_z(p')+\phi_z'(p')(p-p')+\frac{1}{2{\tilde\varepsilon}^2}(p-p')^2.
\]
Apply with $p=\hat p$ and $p'=p^\star$ and take conditional expectation given $t(w)$.
The linear term vanishes since $\E[\phi_z'(p^\star)\mid t(w)]=0$.
Taking expectations yields the upper bound.
The lower bound is Lemma~\ref{lem:oracle-min}.
Finally, $\E[(\Delta q)^2]\le \|\Delta q\|_4^2$.
\end{proof}

\begin{lemma}[Entry into the local region in $\mathcal E$]\label{lem:enter}
Under \cref{asm:ospo-plugin-range} and the conditions of \cref{lem:margin},
if
\[
\Errhatg+\frac{1}{2{\tilde\varepsilon}^2}\zeta^2 \ <\ \lambda_0\,\iota^2,
\]
then $\mathcal E(\hat t)<\iota$.
\end{lemma}
\begin{proof}
$L(\hat t,g^\star)\ge L(\hat t,\hat g)$ by Lemma~\ref{lem:oracle-min}, so
\[
L(t^\star,g^\star)-L(\hat t,g^\star)
\le \Errhatg + \big(L(t^\star,g^\star)-L(t^\star,\hat g)\big).
\]
By Lemma~\ref{lem:quad-nuis} at $t=t^\star$ and $\|\Delta q(t^\star)\|_4\le \zeta$, the last term is at most $(2{\tilde\varepsilon}^2)^{-1}\zeta^2$.
Therefore $L(t^\star,g^\star)-L(\hat t,g^\star)<\lambda_0\iota^2$. \Cref{lem:margin} completes the argument.
\end{proof}

\begin{lemma}[Lipschitzness of likelihood to nuisance difference]\label{lem:F-lip}
Assume \cref{asm:ospo-plugin-range}. Define $F(t)=L(t,\hat g)-L(t,g^\star)$.
Then for all $t$ with $\Ecal(t)\leq\iota$,
\[
|F(\tst)-F(t)|
\ \le\
C_F\,(\Gamma\zeta+\zeta^2)\,\Ecal(t),
\qquad
C_F=\frac{2}{{\tilde\varepsilon}^2}+\frac{2(L_\Psi+1)}{{\tilde\varepsilon}^3},
\]
where $L_\Psi$ is the constant from \cref{lem:true-q-lip}.
\end{lemma}
\begin{proof}
Write $\Delta=\Delta q(t)$ and $p=q^\star(t)$.
By the integral remainder form of Taylor's theorem,
\[
\phi_z(p+\Delta)-\phi_z(p)=\phi_z'(p)\Delta+\int_0^1(1-s)\phi_z''(p+s\Delta)\Delta^2\,ds.
\]
Taking expectation and using Lemma~\ref{lem:score0} with $a(\cdot)=\Delta(\cdot)$ gives
$\E[\phi_z'(p)\Delta]=0$, hence
\[
F(t)=L(t,\hat g)-L(t,g^\star)=-\int_0^1(1-s)\,\E\big[\phi_z''(p+s\Delta)\Delta^2\big]\,ds.
\]
Similarly for $\tst$ with $p'=q^\star(\tst)$ and $\Delta'=\Delta q(\tst)$.
Thus, by triangle inequality,
\begin{align*}
|F(t)-F(\tst)|
&\le \int_0^1(1-s)\,\E\Big|\phi_z''(p+s\Delta)\Delta^2-\phi_z''(p'+s\Delta')(\Delta')^2\Big|\,ds\\
&\le \int_0^1(1-s)\Big[
\E\big(|\phi_z''(p+s\Delta)|\,|\Delta^2-(\Delta')^2|\big)
+\E\big(|\phi_z''(p+s\Delta)-\phi_z''(p'+s\Delta')|\,(\Delta')^2\big)
\Big]ds.
\end{align*}
For the first term, Lemma~\ref{lem:phi-bounds} gives $|\phi_z''|\le {\tilde\varepsilon}^{-2}$, hence
\[
\E\big(|\phi_z''(p+s\Delta)|\,|\Delta^2-(\Delta')^2|\big)
\le \frac{1}{{\tilde\varepsilon}^2}\E\big(|\Delta-\Delta'|\cdot|\Delta+\Delta'|\big)
\le \frac{1}{{\tilde\varepsilon}^2}\|\Delta-\Delta'\|_2\,\|\Delta+\Delta'\|_2.
\]
Since $\|\Delta\|_4,\|\Delta'\|_4\le \zeta$, we have $\|\Delta+\Delta'\|_2\le 2\zeta$.
Also \cref{def:ospo-plugin-rate} implies $\|\Delta-\Delta'\|_2\le \Gamma\,\Ecal(t)$.
Therefore the first term is bounded by $(2/{\tilde\varepsilon}^2)\Gamma\zeta\Ecal(t)$.

For the second term, Lemma~\ref{lem:phi-bounds} implies $\phi_z''$ is $(2/{\tilde\varepsilon}^3)$-Lipschitz, so
\[
|\phi_z''(p+s\Delta)-\phi_z''(p'+s\Delta')|
\le \frac{2}{{\tilde\varepsilon}^3}\,| (p-p')+s(\Delta-\Delta')|.
\]
Hence, by Cauchy-Schwarz and $\|(\Delta')^2\|_2=\|\Delta'\|_4^2\le \zeta^2$,
\[
\E\big(|\phi_z''(p+s\Delta)-\phi_z''(p'+s\Delta')|\,(\Delta')^2\big)
\le \frac{2}{{\tilde\varepsilon}^3}\,\|(p-p')+s(\Delta-\Delta')\|_2\,\|(\Delta')^2\|_2
\le \frac{2\zeta^2}{{\tilde\varepsilon}^3}\,\|(p-p')+s(\Delta-\Delta')\|_2.
\]
Now
\[
\|p-p'\|_2
\le L_\Psi\Ecal(t),
\]
by \cref{lem:true-q-lip}. Thus
\[
\|(p-p')+s(\Delta-\Delta')\|_2\le (L_\Psi+\Gamma)\,\Ecal(t).
\]
Since $|\Delta q(t)(w)|\le 1$ a.s.\ by the range assumption (difference of two $[{\tilde\varepsilon},1-{\tilde\varepsilon}]$ quantities),
we have $\zeta\le 1$, and hence $\Gamma\zeta^2\le\Gamma\zeta$.
Therefore the second term is bounded by $(2/{\tilde\varepsilon}^3)(L_\Psi\zeta^2+\Gamma\zeta)\Ecal(t)$. Finally, integrate over $s\in[0,1]$ with weight $(1-s)$, noting $\int_0^1(1-s)\,ds=1/2$, and obtain the desired bound.
\end{proof}

\subsubsubsection{Proof of the theorem}

\begin{proposition}[Localized convergence of projected norm]\label{thm:ospo-localized}
Suppose \cref{asmp:div,asmp:realizability,asmp:basic,asmp:ospo-margin,asm:ospo-regularity,asm:ospo-plugin-range} hold.
Then, if
\[
\Errhatg+\frac{1}{2{\tilde\varepsilon}^2}\zeta^2 \ <\ \lambda_0\,\iota^2,
\]
then
\[
\mathcal E(\hat t) \le \iota,
\qquad\text{and}\qquad
\mathcal E(\hat t)^2
\ \le\ \frac{2}{\lambda_0}\,\Errhatg\ +\ \frac{C_F^2}{\lambda_0^2}\,(\Gamma\zeta+\zeta^2)^2.
\]
\end{proposition}
\begin{proof}
Under the premise, $\mathcal E(\hat t)<\iota$ by Lemma~\ref{lem:enter}.

By optimality and the definitions of $\Errhatg$ and of $F(t)$, we have
\begin{align*}
L(t^\star,g^\star)-L(\hat t,g^\star)
&=\big(L(t^\star,\hat g)-L(\hat t,\hat g)\big)\\
&\quad+\Big(\big[L(\hat t,\hat g)-L(\hat t,g^\star)\big]-\big[L(t^\star,\hat g)-L(t^\star,g^\star)\big]\Big)\\
&\le \Errhatg + \big|F(\hat t)-F(t^\star)\big|.
\end{align*}
By Lemma~\ref{lem:F-lip},
\[
|F(\hat t)-F(t^\star)|
\le C_F\,(\Gamma\zeta+\zeta^2)\,\mathcal E(\hat t).
\]
Hence
\[
L(t^\star,g^\star)-L(\hat t,g^\star)\le \Errhatg+C_F\,(\Gamma\zeta+\zeta^2)\,\mathcal E(\hat t).
\]

\Cref{lem:margin} gives
\[
\lambda_0\,\mathcal E(\hat t)^2\le L(t^\star,g^\star)-L(\hat t,g^\star).
\]
Combine and apply $ab\le a^2/2+b^2/2$ with
$a=\mathcal E(\hat t)\sqrt{\lambda_0}$ and $b=C_F(\Gamma\zeta+\zeta^2)/\sqrt{\lambda_0}$:
\[
\frac{\lambda_0}{2}\,\mathcal E(\hat t)^2\le \Errhatg+\frac{C_F^2}{2\lambda_0}\,(\Gamma\zeta+\zeta^2)^2.
\]
Rearranging yields the stated bound.
\end{proof}

\begin{proof}[Proof of \cref{thm:ospo-rate}]
Let
\[
R_\delta=(1+\Gamma)\prns{\varrho_n+\sqrt{\frac{\log(1/\delta)}{n}}}+\Gamma\zeta+\zeta^2.
\]
By taking the scaling in the definition of $\rho$ to zero and using the boundedness and coverage in \cref{asmp:basic}, there is a constant $B_\rho<\infty$ such that $\rho(s h_\theta)\le B_\rho$ for all $\theta\in\Theta$ and all scalars $s\in\RR$. We prove the result when $R_\delta<c'$ for a constant $c'>0$ chosen below. If instead $R_\delta\ge c'$, then $\rho(\hat s h_{\hat\theta_{\rm OSPO}})\le B_\rho\le (B_\rho/c')R_\delta$, so the theorem follows after enlarging the final constant.

Define the local Rademacher complexity
$$
\Rcal_n(\Fcal,\varrho)=\frac1{2^n}\sum_{\epsilon\in\{-1,1\}^n}\E_{w_1,\dots,w_n\sim P}\sup_{f\in\Fcal:\|f\|\leq \varrho}\abs{\frac1n\sum_{i=1}^n\epsilon_if(w_i)}.$$
Consider the star hull $\Tcal_\star=\{\beta(t_\theta-t^\star):\theta\in\Theta,\beta\in[0,1]\}$. 
We will first seek to bound the critical radius $\bar\varrho_n$, being the smallest solution to inequality
$\Rcal_n(\Tcal_\star,\varrho)\leq \varrho^2/M$. Toward that end, invoke the refined Dudley integral \citep{SrebroSridharanRefinedDudley,bartlett2005local}:
\[
\Rcal_n(\Tcal_\star,\varrho)
\lesssim
\inf_{\epsilon\in(0,\varrho]}
\left\{\epsilon+\frac1{\sqrt n}\int_\epsilon^\varrho
\sqrt{\log N_2(u,\Tcal_\star)}\,du\right\}.
\]
By lemma 4.5 of \citet{mendelson2002improving}, for each $u\in(0,\varrho]$,
\[
\log N_2(u,\Tcal_\star)
\lesssim
\log\frac2u+\log N_2(u/2,\Tcal-t^\star)
\lesssim
\log\frac2u+\log N_2(u/2,\Tcal).
\]
Altogether, we find that $\bar\varrho_n\lesssim\varrho_n$ for $\varrho_n$ as given in the statement of the theorem.

Next, given this critical radius bound, we apply lemma 12 of \citet{foster2023orthogonal} to $\Errhatg$, redoing the peeling with $\Ecal(\hat t)$ instead of $\|\hat t-\tst\|$. The required local Lipschitz modulus is controlled because \cref{lem:true-q-lip} and \cref{def:ospo-plugin-rate} give
\[
\|\hat q(t)-\hat q(\tst)\|_2
\le
\|q^\star(t)-q^\star(\tst)\|_2+\|\Delta q(t)-\Delta q(\tst)\|_2
\le
(L_\Psi+\Gamma)\,\Ecal(t)
\]
whenever $\Ecal(t)\le\iota$. Because $\Tcal_\star$ is star-shaped, $\varrho\mapsto\Rcal_n(\Tcal_\star,\varrho)/\varrho$ is nonincreasing, so multiplying the local Lipschitz modulus by $L_\Psi+\Gamma$ multiplies the corresponding critical radius by at most a constant times $1+\Gamma$. Combining this empirical-process bound with \cref{thm:ospo-localized} gives, for some $c_0>0$ independent of $\Gamma$ and $\zeta$,
$$
\textstyle \mathcal E(\hat t)\;\leq\; c_0\prns{
(1+\Gamma)\prns{\varrho_n+\sqrt{\frac{\log\log n}{n}}+\sqrt{\frac{\log(1/\delta)}{n}}}+\Gamma\zeta+\zeta^2
}.
$$
Since $\sqrt{\log\log n/n}$ is bounded by a constant multiple of $\varrho_n$ in all cases listed in the theorem, this yields
\[
\mathcal E(\hat t)
\le c\prns{
(1+\Gamma)\prns{\varrho_n+\sqrt{\frac{\log(1/\delta)}{n}}
}
+\Gamma\zeta+\zeta^2
}.
\]

It remains to fix the sign. Let $\mu=\E[\tst(w)z]$. Choose $c'>0$ small enough that $R_\delta<c'$ makes the right-hand side in the preceding display at most $\mu/2$. Thus we may intersect the preceding event with $\mathcal E(\hat t)\le \mu/2$.

If $\hat t=0$, then $\inf_{a\geq0}\|\tst-a\hat s \hat t\|=\mathcal E(\hat t)$ already. Otherwise, condition on the training data. The sign sample is independent, and \cref{lem:fix-sign} with $\gamma=1/2$ gives
\[
\mathbb P\left(\hat s\,\E[\hat t(w)\tst(w)]\le0\mid \hat t\right)
\le
\frac{4}{\abs{\Dcal}}\frac{\E[\tst(w)^2]}{\mu^2}.
\]
For sufficiently large $\abs{\Dcal}$, this probability is at most another constant fraction of $\delta$. On the complement of this sign-error event, $\hat s$ has the same sign as the best affine coefficient relating $\hat t$ to $\tst$, so
\[
\inf_{a\geq0}\|\tst-a\hat s \hat t\|=\mathcal E(\hat t).
\]
The result follows by a union bound, absorbing constant factors in $\delta$ into the universal constant, and by noting that $\inf_{a\geq0}\|\tst-a\hat s \hat t\|$ dominates $\rho(\hat s\,h_{\hat\theta_{\rm OSPO}})$ under the coverage in \cref{asmp:basic}.
\end{proof}

\subsubsection{Analysis of local-polynomial regression as plug-in for OSPO}\label{sec:analysis_of_local_polynomial_regression_as_plug_in_for_ospo}

The deferred assumptions from \cref{sub:local_polynomial_regression_plug_in_error} are below. The first condition below is the standard local-polynomial smoothness and design requirement \citep[Ch.~3]{fan1996local}; the second adds the uniformity over indices needed because OSPO optimizes over $\theta$ and we have one local polynomial regression per $\theta$.

Fix $\iota>0$ and let
\[
\Tcal_+=\Tcal\cup
\left\{\tst+\alpha(\bar t-\tst):
\alpha\in[0,1],\ \bar t=a t_\theta,\ \theta\in\Theta,\ \|\bar t-\tst\|_2\le\iota
\right\}.
\]

\begin{assumption}[Local-polynomial regression conditions]\label{asm:local-poly-regression}
The kernel $K$ is supported on $[-1,1]$, bounded, symmetric, and continuously differentiable. For $t\in\Tcal_+$, let $f_t$ be the density of $t(w)$ and $\mu_t(u)=g_t^\star(u)f_t(u)$. Uniformly over $t\in\Tcal_+$, $f_t$ is bounded above and below on $[-M,M]$, $\sigma(t)$ is bounded above and below away from zero, and $f_t,\mu_t$ are $s$-H\"older with a common constant. Writing $K_b(v)=b^{-1}K(v/b)$ and $b_t=h\sigma(t)$, the local moment matrices
\[
\E\left[
K_{b_t}(t(w)-u)
R_r\!\left(\frac{t(w)-u}{b_t}\right)
R_r\!\left(\frac{t(w)-u}{b_t}\right)^\top
\right]
\]
have smallest eigenvalues bounded below by a positive constant uniformly over $t\in\Tcal_+$, $u\in[-M,M]$, and all sufficiently small $h$.
\end{assumption}

\begin{assumption}[Uniform local-polynomial complexity]\label{asm:local-poly-uniform}
Whenever $t_\alpha=t+\alpha d$ stays in $\Tcal_+$ for $\alpha$ in a neighborhood of $0$, the pathwise derivatives $\dot f_{t,d}(u)=\partial_\alpha f_{t_\alpha}(u)|_{\alpha=0}$ and $\dot\mu_{t,d}(u)=\partial_\alpha\mu_{t_\alpha}(u)|_{\alpha=0}$ exist and are $(s-1)$-H\"older with a common constant for all unit directions $d$ in
\[
\Dcal_{\Tcal}=
\left\{w\mapsto
\frac{t(w)-t'(w)}{\|t-t'\|_2}:t,t'\in\Tcal_+,\ t\ne t'
\right\}.
\]
The class $\Dcal_{\Tcal}$ is uniformly bounded and has polynomial $\|\cdot\|_\infty$ entropy. Also, $\Tcal_+$ has $N_\infty(\epsilon,\Tcal_+)\le (A/\epsilon)^v$ for all $\epsilon\in(0,1]$, and, for every $j,\ell\in\{0,1\}$ and polynomial $P$ of degree at most $2r+1$, the class
\[
\begin{aligned}
\Big\{(w,z)\mapsto\;&
e(w)b_t^{-1-j}K^{(j)}\!\left(\frac{t(w)-u}{b_t}\right)
P\!\left(\frac{t(w)-u}{b_t}\right)z^\ell:
t\in\Tcal_+,\ e\in\Dcal_{\Tcal}\cup\{1\},\ u\in[-M,M]\Big\},
\end{aligned}
\]
where $K^{(0)}=K$ and $K^{(1)}=K'$, has logarithmic covering number at most $v\log(c/(\epsilon h))$ in sup norm, for constants $v,c$ independent of $h$.
\end{assumption}

\subsubsubsection{Proof of \cref{thm:local-poly-plugin}}

\begin{proof}
Write the local-polynomial estimator using normalized weights $K_b(v)=b^{-1}K(v/b)$; this does not change the minimizer. Let
\[
\hat S_t(u)=\frac1m\sum_i K_{b_t}(t(w_i)-u)R_i(u)R_i(u)^\top,\qquad
\hat T_t(u)=\frac1m\sum_i K_{b_t}(t(w_i)-u)R_i(u)z_i,
\]
where $R_i(u)=R_r((t(w_i)-u)/b_t)$, and define $S_t(u)=\E\hat S_t(u)$ and $T_t(u)=\E\hat T_t(u)$. The population local-polynomial coefficient is $a_t^\circ(u)=S_t(u)^{-1}T_t(u)$.

The uniform nonsingularity condition in \cref{asm:local-poly-regression} gives $\lambda_{\min}\{S_t(u)\}\ge c>0$ for all $t,u$ and all sufficiently small $h$. Since local polynomials reproduce polynomials of degree $r$, Taylor expansion of the $s$-H\"older functions $f_t$ and $\mu_t=g_t^\star f_t$ around $u$ gives
\[
\sup_{t,u}|e_1^\top a_t^\circ(u)-g_t^\star(u)|\le Ch^s,
\qquad
\sup_{t,u}\left|\partial_u\{e_1^\top a_t^\circ(u)-g_t^\star(u)\}\right|\le Ch^{s-1}.
\]
Applying the same polynomial-reproduction argument to the pathwise derivatives $\dot f_{t,d}$ and $\dot\mu_{t,d}$ in \cref{asm:local-poly-uniform} gives
\[
\sup_{t,u,d\in\Dcal_{\Tcal}}\left|
D_d\{e_1^\top a_t^\circ(u)-g_t^\star(u)\}
\right|\le Ch^{s-1},
\]
where $D_d$ denotes the derivative along $t_\alpha=t+\alpha d$ at $\alpha=0$. The derivative of $b_t=h\sigma(t)$ contributes only $O(h^s)$ to this deterministic bias derivative because $D_d\sigma(t)$ is uniformly bounded, so it is absorbed by the $h^{s-1}$ bound.

For the stochastic terms, apply Bernstein's inequality and a union bound over the entropy nets in \cref{asm:local-poly-uniform} to each entry of $\hat S_t-S_t$ and $\hat T_t-T_t$. The envelopes are $O(h^{-1})$ and the variances are $O(h^{-1})$, so with probability at least $1-\delta/2$,
\[
\sup_{t,u}\{\|\hat S_t(u)-S_t(u)\|_{\rm op}+\|\hat T_t(u)-T_t(u)\|\}
\le
C\left\{\sqrt{\frac{A_m(\delta)}{mh}}+\frac{A_m(\delta)}{mh}\right\}.
\]
For $u$-derivatives and pathwise derivatives indexed by $\Dcal_{\Tcal}$, each differentiated summand is a finite linear combination of the function classes in \cref{asm:local-poly-uniform}. The derivative of the scale factor is harmless because
\[
|D_d\sigma(t)|=\left|\frac{\operatorname{Cov}(t(w),d(w))}{\sigma(t)}\right|
\le C
\]
uniformly over $t\in\Tcal_+$ and $d\in\Dcal_{\Tcal}$. Thus these derivative classes have envelopes $O(h^{-2})$ and variances $O(h^{-3})$. Bernstein's inequality and the same net argument give, on the same event after enlarging $C$,
\[
\sup_{t,u}\{\|\partial_u(\hat S_t-S_t)(u)\|_{\rm op}+\|\partial_u(\hat T_t-T_t)(u)\|\}
\le
C\left\{\sqrt{\frac{A_m(\delta)}{mh^3}}+\frac{A_m(\delta)}{mh^2}\right\}.
\]
The identical bound holds for the directional derivatives $D_d(\hat S_t-S_t)$ and $D_d(\hat T_t-T_t)$, uniformly over $d\in\Dcal_{\Tcal}$. Since $mh^3\ge C_0A_m(\delta)$, the perturbation of $\hat S_t(u)$ is small enough that $\hat S_t(u)^{-1}$ exists and is uniformly bounded. Differentiating $e_1^\top\hat S_t(u)^{-1}\hat T_t(u)$, and using the identity $D(A^{-1})=-A^{-1}(DA)A^{-1}$, gives
\[
\sup_{t,u}|\tilde g_t(u)-g_t^\star(u)|\le C\zeta_m,
\qquad
\sup_{t,u,d\in\Dcal_{\Tcal}}\left(|\partial_u e_t(u)|+|D_d e_t(u)|\right)\le C\gamma_m,
\]
where $e_t(u)=\tilde g_t(u)-g_t^\star(u)$.

It remains to translate the derivative bound into stability in the index. Write $\Delta q_t(w)=e_t(t(w))$. Reduce $\iota$, if needed, so that $\iota<\|\tst\|_2/2$. For any local $\theta$, choose $a_\theta$ with $\|\tst-a_\theta t_\theta\|_2$ arbitrarily close to $\inf_a\|\tst-at_\theta\|_2$ and set $\bar t=a_\theta t_\theta$; then $a_\theta\ne0$. Because the bandwidth scales by $\sigma(t)$, the kernel is symmetric, and the polynomial basis only changes by an invertible diagonal sign matrix under negative rescaling, the untruncated fitted values $\tilde g_t(t(w))$ and errors $\Delta q_t(w)$ are invariant to nonzero scalar rescaling of $t$. Thus it is enough to compare $\bar t$ with $\tst$.

Let $t_\alpha=\tst+\alpha(\bar t-\tst)$, $0\le \alpha\le1$. Differentiating along this path and setting $d=(\bar t-\tst)/\|\bar t-\tst\|_2$ gives
\[
\left\|\frac{d}{d\alpha}e_{t_\alpha}(t_\alpha(w))\right\|_2
\le C\gamma_m\|\bar t-\tst\|_2
\]
uniformly in $\alpha$: the derivative is the sum of the pathwise derivative $D_d e_{t_\alpha}(t_\alpha(w))$ and the evaluation derivative $\partial_u e_{t_\alpha}(t_\alpha(w))d(w)$, multiplied by $\|\bar t-\tst\|_2$, and $\|d\|_2=1$. Integrating from $0$ to $1$ yields
\[
\|\Delta q_\theta-\Delta q_{\thetas}\|_2
=\|\Delta q_{\bar t}-\Delta q_{\tst}\|_2
\le C\gamma_m\|\bar t-\tst\|_2.
\]
Taking the infimum over admissible $a_\theta$ proves the stability display.

Finally, $\eta(w)=\Psis(\tst(w))\in[\varepsilon_g,1-\varepsilon_g]$, so every $g_t^\star$ also lies in this interval. If $C\zeta_m\le \varepsilon_g-\varepsilon_0$, the uniform sup-norm bound shows $\tilde g_t(u)\in[\varepsilon_0,1-\varepsilon_0]$ for all $t,u$, so clipping is inactive and the displayed bounds hold for the clipped estimator $\hat g_t$ as well. The clipping definition also gives \cref{asm:ospo-plugin-range} with $c_{\rm range}=\varepsilon_0$, since $\min\{\hat g_t(u),1-\hat g_t(u)\}\ge\varepsilon_0$. Thus the error terms in \cref{def:ospo-plugin-rate} obey $\zeta\le C\zeta_m$ and $\Gamma\le C\gamma_m$.
\end{proof}

\subsection{Theoretical Properties of RSPO}
\label{sec:rspo-theory}

\subsubsection{Proof of \cref{prop:auc-max}}

\subsubsubsection{Supporting lemmas}

\begin{lemma}[AUC as a pairwise identification success probability]\label{lem:pairwise}
Let $(w_0,z_0),(w_1,z_1)$ be i.i.d.\ and set $D=\{z_0\neq z_1\}$.
Define the success indicator
\[
S_t
=\indic{t(w_1)>t(w_0),\,z_1=1,z_0=0}
+\indic{t(w_0)>t(w_1),\,z_0=1,z_1=0}
+\frac12\,\indic{t(w_1)=t(w_0),\,D}.
\]
Then $\AUC(t)=\E[S_t\mid D]$.
\end{lemma}

\begin{proof}
By exchangeability, $\PP(z_0=0,z_1=1\mid D)=\PP(z_0=1,z_1=0\mid D)=1/2$.
On $\{z_0=0,z_1=1\}$ we have $S_t=\phi(t(w_1)-t(w_0))$, and on $\{z_0=1,z_1=0\}$ we have
$S_t=\phi(t(w_0)-t(w_1))$.
Moreover, the conditional law of $(w_0,w_1)$ given $\{z_0=1,z_1=0\}$ equals the law of $(w_1,w_0)$ given $\{z_0=0,z_1=1\}$, hence
\[
\E\big[\phi(t(w_0)-t(w_1))\mid z_0=1,z_1=0\big]
=\E\big[\phi(t(w_1)-t(w_0))\mid z_0=0,z_1=1\big]=\AUC(t).
\]
Therefore $\E[S_t\mid D]=\AUC(t)$.
\end{proof}

\begin{lemma}[Monotone version lemma]\label{lem:mono-version}
Let $(X,Y)$ and $(X',Y')$ be i.i.d.\ real-valued pairs, and suppose $Y\in[L,U]$ a.s.\ for finite $L\leq U$.
If
\[
\PP(Y>Y',\,X\le X')=0,
\]
then there exists a nondecreasing Borel function $g:\R\to[L,U]$ such that $Y=g(X)$ a.s.
\end{lemma}

\begin{proof}
Fix $q\in\mathbb Q\cap[L,U]$ and write $A_q=\{Y>q\}$ and $B_q=\{Y\le q\}$.
Since $Y>q\ge Y'$ on $A_q\times B_q$, the hypothesis gives
\[
\PP\big(A_q,\ B_q',\ X\le X'\big)=0,
\]
where $B_q'$ denotes the event $\{Y'\le q\}$ for the independent copy.
We show that, up to null sets, $A_q$ is an upper set in $X$.
If $\PP(A_q)=0$, set $E_q=\emptyset$.
If $\PP(B_q)=0$, set $E_q=\R$.
Otherwise, let
\[
c_q=\operatorname*{ess\,sup}(X\mid B_q).
\]
This value is finite: since $\PP(A_q)>0$, there exists $a$ with $\PP(A_q,\ X\le a)>0$, and the preceding display then forces $\PP(B_q,\ X>a)=0$.
Then $\PP(B_q,\ X>c_q)=0$ and, for every $a<c_q$, $\PP(B_q,\ X>a)>0$.
If $\PP(A_q,\ X<c_q)>0$, then for some $a<c_q$ we have $\PP(A_q,\ X\le a)>0$; independence then gives positive probability to $A_q$, $B_q'$, and $X\le a<X'$, a contradiction.
Thus $\PP(A_q,\ X<c_q)=0$.
Also $\PP(A_q,\ X=c_q)$ and $\PP(B_q,\ X=c_q)$ cannot both be positive, since otherwise the independent pair has positive probability of satisfying $A_q$, $B_q'$, and $X=X'=c_q$.
If $\PP(A_q,\ X=c_q)>0$, set $E_q=[c_q,\infty)$; otherwise set $E_q=(c_q,\infty)$.
In either case,
\[
\indic{A_q}=\indic{\{X\in E_q\}}\qquad\text{a.s.}
\]

Intersecting the above equality over the countable set $\mathbb Q\cap[L,U]$, we obtain an event of probability one on which it holds for every such $q$.
Define
\[
g(x)=\sup\bigl(\{q\in\mathbb Q\cap[L,U]:x\in E_q\}\cup\{L\}\bigr).
\]
Each $E_q$ is an upper interval, so $g$ is nondecreasing.
Moreover, $\{g>r\}=\R$ for $r<L$, $\{g>r\}=\emptyset$ for $r\ge U$, and
\[
\{g>r\}=\bigcup_{q\in\mathbb Q\cap[L,U],\,q>r}E_q
\qquad (L\le r<U),
\]
so $g$ is Borel.
On the probability-one event above,
\[
g(X)=\sup\bigl(\{q\in\mathbb Q\cap[L,U]:Y>q\}\cup\{L\}\bigr)=Y.\qedhere
\]
\end{proof}

\subsubsubsection{Proof of the proposition}

\begin{proof}
Let $(w_0,z_0),(w_1,z_1)$ be i.i.d.\ copies and $D=\{z_0\neq z_1\}$.
Fix a measurable $t$.

\medskip
\noindent\textbf{Pointwise Bayes optimal comparison.}
Condition on $(w_0,w_1)=(u,v)$.
Write $\eta_u=\eta(u)$ and $\eta_v=\eta(v)$.
Since $z_0\mid w_0=u\sim \mathrm{Bernoulli}(\eta_u)$ and $z_1\mid w_1=v\sim \mathrm{Bernoulli}(\eta_v)$ are conditionally independent,
\[
\PP(z_1=1,z_0=0\mid w_0=u,w_1=v)=\eta_v(1-\eta_u),\qquad
\PP(z_1=0,z_0=1\mid w_0=u,w_1=v)=(1-\eta_v)\eta_u.
\]
On $D$, define (when the denominator is $0$, set $p(u,v)=1/2$; this case has $\PP(D\mid u,v)=0$ and is irrelevant):
\[
p(u,v)=\PP(z_1=1,z_0=0\mid w_0=u,w_1=v,D)
=\frac{\eta_v(1-\eta_u)}{\eta_v(1-\eta_u)+(1-\eta_v)\eta_u}.
\]
Therefore,
\[
p(u,v)>\tfrac12 \iff \eta_v>\eta_u,\qquad p(u,v)=\tfrac12 \iff \eta_v=\eta_u,\qquad p(u,v)<\tfrac12 \iff \eta_v<\eta_u.
\]
Given $(u,v,D)$, the success probability of the rule induced by $t$ (pick the larger score, break ties uniformly) equals
\[
\E[S_t\mid w_0=u,w_1=v,D]
=
\begin{cases}
p(u,v), & t(v)>t(u),\\
1-p(u,v), & t(v)<t(u),\\
\tfrac12, & t(v)=t(u).
\end{cases}
\]
Hence, pointwise,
\[
\E[S_t\mid w_0=u,w_1=v,D]\ \le\ \max\{p(u,v),1-p(u,v)\},
\]
with equality if and only if:
\begin{itemize}
\item if $\eta_v>\eta_u$, then $t(v)>t(u)$;
\item if $\eta_v<\eta_u$, then $t(v)<t(u)$;
\item if $\eta_v\ne \eta_u$, then $t(v)\ne t(u)$ (no tie across unequal $\eta$).
\end{itemize}

\medskip
\noindent\textbf{Global optimality and equality conditions.}
Combining with Lemma~\ref{lem:pairwise} gives
\[
\AUC(t)=\E[S_t\mid D]\ \le\ \E\!\left[\max\{p(w_0,w_1),1-p(w_0,w_1)\}\,\big|\,D\right].
\]
The right-hand side is achieved by any score that ranks $w$ by $\eta(w)$ (ties allowed only when $\eta$ ties), e.g.\ by $\eta$ itself.
Therefore,
\[
\sup_{t}\AUC(t)=\AUC(\eta).
\]
Since $\eta=\Psi^\star(t^\star)$ and $\Psi^\star$ is strictly increasing, $\AUC(\eta)=\AUC(t^\star)$, so $t^\star$ is a maximizer.
Moreover, a score $t$ is a maximizer if and only if it attains equality in the pointwise bound above for $(w_0,w_1)$ distributed as $(w_0,w_1)\mid D$.
Because $\eta(w)\in(0,1)$ a.s.\ (since $\Psi^\star:\R\to(0,1)$), $\PP(D\mid w_0,w_1)>0$ a.s., so equality holds for $(w_0,w_1)$ under $(w_0,w_1)\mid D$ if and only if it holds for i.i.d.\ $(w_0,w_1)$ under $P\otimes P$.
Thus, for a maximizer $t$, with i.i.d.\ $w,w'$,
\begin{equation}\label{eq:no-inversion}
\PP\big(\eta(w)>\eta(w'),\ t(w)\le t(w')\big)=0.
\end{equation}

\medskip
\noindent\textbf{Deduce $\eta(w)=g(t(w))$ with $g$ nondecreasing, hence $t^\star(w)=m(t(w))$.}
Since $M$ bounds the index class and $\tst\in\Tcal$, $\tst\in[-M,M]$ a.s., so $\eta(w)\in J=[\Psi^\star(-M),\Psi^\star(M)]$ a.s.
Applying Lemma~\ref{lem:mono-version} with $X=t(w)$, $Y=\eta(w)$, and $[L,U]=J$ yields a nondecreasing Borel $g:\R\to J$ with
\[
\eta(w)=g(t(w))\qquad\text{a.s.}
\]
Finally, since $\Psi^\star$ is strictly increasing, define the nondecreasing Borel inverse on $J$ by
\[
\Gamma(u)=\inf\{s\in[-M,M]:\Psi^\star(s)\ge u\},\qquad u\in J.
\]
Set $m=\Gamma\circ g$. Then $m$ is nondecreasing and, by strict monotonicity of $\Psi^\star$,
\[
t^\star(w)=\Gamma(\eta(w))=\Gamma(g(t(w)))=m(t(w))\qquad\text{a.s.}
\]
This shows that every AUC maximizer lies in the stated class.

\medskip
\noindent\textbf{Converse.}
Conversely, suppose $t^\star(w)=m(t(w))$ a.s.\ for some nondecreasing $m$.
Then $\eta(w)=\Psi^\star(t^\star(w))=(\Psi^\star\circ m)(t(w))$, where $\Psi^\star\circ m$ is nondecreasing.
On a $P\otimes P$-full set of pairs $(u,v)$,
\[
\eta(u)>\eta(v)\ \Rightarrow\ t^\star(u)>t^\star(v)\ \Rightarrow\ m(t(u))>m(t(v))\ \Rightarrow\ t(u)>t(v),
\]
and $t(u)=t(v)\Rightarrow t^\star(u)=t^\star(v)\Rightarrow \eta(u)=\eta(v)$.
Thus $t$ satisfies the pointwise equality conditions of Step~1, so $\AUC(t)=\AUC(\eta)=\AUC(t^\star)$, i.e.\ $t$ is an AUC maximizer.

Combining both directions proves the proposition.
\end{proof}

\subsubsection{Proof of \cref{thm:auc-margin}}

\begin{proof}
First we make some definitions based on the assumptions. Let $I=[\underline{t^\star},\overline{t^\star}]\subseteq[-M,M]$ be the support of $\tst$. Let $p$ its density satisfying $0<\underline p \le p(u)\le \overline p<\infty$ for all $u\in I$. Let $\eta(w)=\PP(z=1\mid w)=\PP(z=1\mid \tst(w))=\Psis(\tst(w))\in[\varepsilon,1-\varepsilon]$ for $\varepsilon=\min\{\Psi^\star(-M),1-\Psi^\star(M),0.25\}>0$. The derivative $(\Psi^\star)'$ satisfies $0<\underline\psi \le (\Psi^\star)'(u)\le \overline\psi<\infty$ for all $u\in[-M,M]$. Finally, define $\underline q = \underline p/\overline\psi$, $\lambda_0 = c_{\rm MSE}\,\underline q\,\underline\psi^{2}
= c_{\rm MSE}\,\frac{\underline p}{\overline\psi}\,\underline\psi^{2}$, $B_0=\frac{1-\varepsilon}{\varepsilon}\,\overline p$, $C_H = 6\,\varepsilon^{-1/3}\,B_0^{2/3}
=6\,\frac{(1-\varepsilon)^{2/3}}{\varepsilon}\,\overline p^{2/3}$, $\gamma_0 = \frac{1}{16}\,\underline q^{2}\left(\frac{\underline\psi}{\overline p}\right)^{3}
=\frac{\underline p^{2}\,\underline\psi^{3}}{16\,\overline\psi^{2}\,\overline p^{3}}$, and $\epsilon_0 = \min\Big\{\frac{\gamma_0}{2},\ \lambda_0\Big(\frac{\gamma_0}{2C_H}\Big)^3\Big\}$. Let $\mathrm{sgn}_0(u)=2\phi(u)-1\in\{-1,0,1\}$. And let $\pi_1=\PP(z=1)$ and $\pi_0=1-\pi_1$ (in our setting $\pi_0=\pi_1=1/2$ by the symmetry of $P$ but the proof actually does not use this and works more generally so we keep these as symbols).

\medskip
\noindent\textbf{A symmetrized representation of AUC and the Bayes score.}
Let $(w_0,z_0)$ and $(w_1,z_1)$ be i.i.d.\ copies of $(w,z)$ and set $\eta_i=\eta(w_i)$ and $t_i=t(w_i)$.
A direct conditioning on $(w_0,w_1)$ gives
\begin{equation}\label{eq:auc-cond}
\pi_0\pi_1\,\AUC(t)=\E\!\left[(1-\eta_0)\eta_1\,\phi(t_1-t_0)\right].
\end{equation}
Swapping indices $0\leftrightarrow 1$ and using $\phi(u)+\phi(-u)=1$ yields
\[
\pi_0\pi_1\,\AUC(t)
=\E\!\left[(1-\eta_1)\eta_0\,\phi(t_0-t_1)\right]
=\pi_0\pi_1-\E\!\left[(1-\eta_1)\eta_0\,\phi(t_1-t_0)\right].
\]
Adding this identity to \cref{eq:auc-cond} gives
\[
2\pi_0\pi_1\,\AUC(t)
=\pi_0\pi_1 + \E\!\left[(\eta_1-\eta_0)\,\phi(t_1-t_0)\right].
\]
Since $\E[\eta_1-\eta_0]=0$ and $\phi(u)=\tfrac12(1+\mathrm{sgn}_0(u))$, we obtain
\begin{equation}\label{eq:auc-sym}
\AUC(t)=\frac12+\frac{1}{4\pi_0\pi_1}\,\E\!\left[(\eta_1-\eta_0)\,\mathrm{sgn}_0(t_1-t_0)\right].
\end{equation}
Because $\Psi^\star$ is strictly increasing, $\mathrm{sgn}_0(t^\star_1-t^\star_0)=\mathrm{sgn}_0(\eta_1-\eta_0)$, hence
\begin{equation}\label{eq:auc-tstar}
\AUC(t^\star)=\frac12+\frac{1}{4\pi_0\pi_1}\,\E\!\left[|\eta_1-\eta_0|\right].
\end{equation}

\medskip
\noindent\textbf{AUC-regret controls the best nonlinear $L^2$ approximation error.}
Fix $t\in\mathcal T$ and define the posterior given the score
\[
\bar\eta(w)=\E[\eta(w)\mid t(w)].
\]
Since $(w_0,w_1)$ are independent and $\bar\eta_i$ is measurable w.r.t.\ $t_i$, we have
$\E[\eta_1-\eta_0\mid t_1,t_0]=\bar\eta_1-\bar\eta_0$.
Applying \cref{eq:auc-sym} to the score $\bar\eta$ and conditioning on $(t_1,t_0)$ gives
\begin{equation}\label{eq:auc-bareta}
\AUC(\bar\eta)
=\frac12+\frac{1}{4\pi_0\pi_1}\,\E\!\left[(\eta_1-\eta_0)\,\mathrm{sgn}_0(\bar\eta_1-\bar\eta_0)\right]
=\frac12+\frac{1}{4\pi_0\pi_1}\,\E\!\left[|\bar\eta_1-\bar\eta_0|\right].
\end{equation}
Furthermore, for any measurable $g:\R\to\R$, the score $g(t)$ is measurable w.r.t.\ $\sigma(t)$, and by \cref{eq:auc-sym},
the conditional expectation given $(t_1,t_0)$ is maximized by choosing the sign $\mathrm{sgn}_0(g(t_1)-g(t_0))$ equal to
$\mathrm{sgn}_0(\bar\eta_1-\bar\eta_0)$.
Hence $\AUC(g(t))\le \AUC(\bar\eta)$ for all $g$, in particular $\AUC(t)\le \AUC(\bar\eta)$.
Combining \cref{eq:auc-tstar,eq:auc-bareta} and $\AUC(t)\le\AUC(\bar\eta)$ yields
\begin{equation}\label{eq:regret-gini}
\AUC(t^\star)-\AUC(t)\ \ge\ \AUC(t^\star)-\AUC(\bar\eta)
=\frac{1}{4\pi_0\pi_1}\Big(\E|\eta_1-\eta_0|-\E|\bar\eta_1-\bar\eta_0|\Big).
\end{equation}
Since $\pi_0\pi_1\le 1/4$, we have $(4\pi_0\pi_1)^{-1}\ge 1$, so
\begin{equation}\label{eq:regret-gini-drop}
\AUC(t^\star)-\AUC(t)\ \ge\ \E|\eta_1-\eta_0|-\E|\bar\eta_1-\bar\eta_0|.
\end{equation}

We now lower bound the right-hand side by an $L^2$ projection error.

\begin{lemma}[Gini gap controls $L^2$ error]\label{lem:gini-gap}
Let $U$ be a real-valued random variable with support an interval $J=[a,b]$ and density $q$ satisfying
$q(u)\ge \underline q>0$ for all $u\in J$.
Let $\mathcal F$ be any sub-$\sigma$-field and set $V=\E[U\mid \mathcal F]$.
Let $(U,V)$ and $(U',V')$ be i.i.d.\ copies. Then
\[
\E|U-U'|-\E|V-V'|\ \ge\ \underline q\,\E[(U-V)^2].
\]
\end{lemma}

\begin{proof}
Define $\varphi(r)=\E|U-r|$. Since $U$ has density $q$ supported on $J=[a,b]$, for all $r\in(a,b)$,
\[
\varphi'(r)=\E[-\mathrm{sgn}(U-r)]=2F_U(r)-1,
\qquad
\varphi''(r)=2q(r)\ge 2\underline q,
\]
so $\varphi$ is $(2\underline q)$-strongly convex on $J$.
Equivalently, for all $u,v\in J$,
\begin{equation}\label{eq:strong-convex}
\varphi(u)\ \ge\ \varphi(v)+\varphi'(v)(u-v)+\underline q\,(u-v)^2.
\end{equation}

Next, since $V=\E[U\mid \mathcal F]$, conditional Jensen gives that for every convex $\psi$,
$\E[\psi(U)\mid \mathcal F]\ge \psi(\E[U\mid \mathcal F])=\psi(V)$, hence $\E\psi(U)\ge \E\psi(V)$.
Applying this with $\psi(\cdot)=|v-\cdot|$ (convex) shows that for every fixed $v\in J$,
\[
\E|v-U'|\ \ge\ \E|v-V'|.
\]
Taking $v=V$ (which is independent of $(U',V')$) and then expectation yields
\begin{equation}\label{eq:conv-order}
\E|V-U'|\ \ge\ \E|V-V'|.
\end{equation}
Finally, note that $\E|U-U'|=\E[\varphi(U)]$ and $\E|V-U'|=\E[\varphi(V)]$.
Therefore, by \cref{eq:conv-order},
\[
\E|U-U'|-\E|V-V'|\ \ge\ \E|U-U'|-\E|V-U'|
=\E[\varphi(U)-\varphi(V)].
\]
Apply \cref{eq:strong-convex} with $(u,v)=(U,V)$ and take conditional expectation given $V$.
Since $\E[U-V\mid V]=0$, the linear term vanishes and we get
$\E[\varphi(U)-\varphi(V)]\ge \underline q\,\E[(U-V)^2]$, which proves the claim.
\end{proof}

Under our assumptions, $\eta(w)=\Psi^\star(t^\star(w))$ has support an interval
$J=[\Psi^\star(\underline{t^\star}),\Psi^\star(\overline{t^\star})]$ and admits a density $q$ given by the change-of-variables formula
$q(s)=p(r)/(\Psi^\star)'(r)$ where $s=\Psi^\star(r)$.
Thus $q(s)\ge \underline p/\overline\psi=\underline q$ on $J$.
Applying Lemma~\ref{lem:gini-gap} with $U=\eta(w)$ and $\mathcal F=\sigma(t(w))$ gives
\[
\E|\eta_1-\eta_0|-\E|\bar\eta_1-\bar\eta_0|
\ \ge\ \underline q\,\E[(\eta-\bar\eta)^2]
= \underline q\,\inf_{h:\R\to\R}\ \|\eta-h(t)\|^2.
\]
Combining with \cref{eq:regret-gini-drop} yields
\begin{equation}\label{eq:regret-to-eta-mse}
\AUC(t^\star)-\AUC(t)\ \ge\ \underline q\,\inf_{h:\R\to\R}\ \|\eta-h(t)\|^2.
\end{equation}

We now pass from $\eta$ back to $t^\star$.
Since $\eta=\Psi^\star(t^\star)$ and $(\Psi^\star)'\ge \underline\psi$ on $[-M,M]$, the inverse $(\Psi^\star)^{-1}$ exists on $J$
and satisfies
\[
|(\Psi^\star)^{-1}(s)-(\Psi^\star)^{-1}(s')|\ \le\ \underline\psi^{-1}\,|s-s'|\quad (s,s'\in J).
\]
For any measurable $h:\R\to\R$, define the clipped map $\Pi(h)=\min\{\max\{h,\,\Psi^\star(\underline{t^\star})\},\,\Psi^\star(\overline{t^\star})\}$.
Then $\|\eta-\Pi(h)(t)\|\le \|\eta-h(t)\|$, so we may restrict to $h$ with range contained in $J$.
For such $h$, set $f=(\Psi^\star)^{-1}\circ h$. Then pointwise,
\[
|\eta-h(t)|=|\Psi^\star(t^\star)-\Psi^\star(f(t))|\ \ge\ \underline\psi\,|t^\star-f(t)|,
\]
hence $\|\eta-h(t)\|^2\ge \underline\psi^2\,\|t^\star-f(t)\|^2$.
Taking infima implies
\[
\inf_{h:\R\to\R}\ \|\eta-h(t)\|^2\ \ge\ \underline\psi^2\ \inf_{f:\R\to\R}\ \|t^\star-f(t)\|^2.
\]
Plugging into \cref{eq:regret-to-eta-mse} and using \cref{asmp:ospo-margin} yields the \emph{global} inequality
\begin{equation}\label{eq:global-unrestricted}
\AUC(t^\star)-\AUC(t)\ \ge\ \lambda_0\ \inf_{a\in\R}\ \|t^\star-a t\|^2,
\qquad \lambda_0=c_{\rm MSE}\,\underline q\,\underline\psi^2.
\end{equation}

\medskip
\noindent\textbf{Small AUC-regret forces the optimal scaling to be nonnegative.}
Write
\[
d(t)^2 = \inf_{a\in\R}\ \|t^\star-a t\|^2.
\]
Let $\hat a$ be a minimizer (it exists because $a\mapsto\|t^\star-a t\|^2$ is a quadratic polynomial).
Assume for contradiction that $\hat a<0$ and set $s=(-\hat a)\,t$.
Then $s$ is a positive multiple of $t$, so $\AUC(s)=\AUC(t)$, and $d(t)^2=\|t^\star+s\|^2$.

We need a continuity bound comparing $\AUC(s)$ and $\AUC(-t^\star)$.
First, the conditional density of $r=t^\star(w)$ given $z=1$ is
$p_1(r)=\eta(r)p(r)/\pi_1$, hence
\[
\|p_1\|_\infty \le \frac{1-\varepsilon}{\varepsilon}\,\overline p =B_0.
\]
Similarly, the conditional density $p_0$ of $t^\star(w)$ given $z=0$ satisfies $\|p_0\|_\infty\le B_0$.
Let $r^+\sim p_1$ and $r^-\sim p_0$ be independent; then $r^+-r^-$ has a density bounded by $B_0$, hence
\begin{equation}\label{eq:margin}
\PP\big(|r^+-r^-|\le u\big)\ \le\ 2B_0\,u,\qquad u\ge 0.
\end{equation}

\begin{lemma}[H\"older continuity of AUC in $L^2$]\label{lem:holder}
Assume \cref{eq:margin} holds for $r^+-r^-$ with constant $B_0$.
Then for any measurable scores $r,s:\mathcal W\to[-M,M]$,
\[
|\AUC(r)-\AUC(s)|\ \le\ C_H\ \|r-s\|^{2/3},
\qquad C_H = 6\,\varepsilon^{-1/3}\,B_0^{2/3}.
\]
\end{lemma}

\begin{proof}
Fix $u>0$ and let $w^+,w^-$ be independent draws from $w\mid z=1$ and $w\mid z=0$, respectively.
If $|r(w^+)-r(w^-)|>2u$ and $|s(w^\pm)-r(w^\pm)|\le u$ for both $\pm$, then
\[
|(s(w^+)-s(w^-))-(r(w^+)-r(w^-))|\le |s(w^+)-r(w^+)|+|s(w^-)-r(w^-)|\le 2u,
\]
so $\mathrm{sgn}_0(s(w^+)-s(w^-))=\mathrm{sgn}_0(r(w^+)-r(w^-))$ and hence $\phi(s(w^+)-s(w^-))=\phi(r(w^+)-r(w^-))$.
Therefore,
\[
|\AUC(s)-\AUC(r)|
\le \PP\big(|r(w^+)-r(w^-)|\le 2u\big)
+ \PP\big(|(s-r)(w^+)|>u\big)
+ \PP\big(|(s-r)(w^-)|>u\big).
\]
The first term is bounded by $4B_0u$ using \cref{eq:margin}.
For the second term, Markov's inequality and $\PP(z=1)=\pi_1\ge \varepsilon$ give
\[
\PP\big(|(s-r)(w^+)|>u\big)
\le \frac{\E[(s-r)(w^+)^2]}{u^2}
= \frac{\E[(s-r)(w)^2\,\indic{z=1}]}{\pi_1 u^2}
\le \frac{\|s-r\|^2}{\varepsilon u^2},
\]
and similarly $\PP(|(s-r)(w^-)|>u)\le \|s-r\|^2/(\varepsilon u^2)$.
Thus
\[
|\AUC(s)-\AUC(r)|\le 4B_0u + \frac{2}{\varepsilon}\frac{\|s-r\|^2}{u^2}.
\]
Optimizing over $u>0$ yields
$|\AUC(s)-\AUC(r)|\le 6\,\varepsilon^{-1/3}B_0^{2/3}\,\|s-r\|^{2/3}$.
\end{proof}

Apply Lemma~\ref{lem:holder} with $r=-t^\star$ and $s$ as above.
Since $-t^\star(w^+)-(-t^\star(w^-))=-(r^+-r^-)$, the margin bound \cref{eq:margin} applies, hence
\begin{equation}\label{eq:auc-close-minus}
\AUC(t)=\AUC(s)\ \le\ \AUC(-t^\star)+C_H\,\|s-(-t^\star)\|^{2/3}
=\AUC(-t^\star)+C_H\,d(t)^{2/3}.
\end{equation}
By \cref{eq:global-unrestricted}, $d(t)^2\le (\AUC(t^\star)-\AUC(t))/\lambda_0$.
Let $\Delta=\AUC(t^\star)-\AUC(t)$; then $d(t)^{2/3}\le (\Delta/\lambda_0)^{1/3}$.
Using $\AUC(-t^\star)=1-\AUC(t^\star)$ (order reversal) in \cref{eq:auc-close-minus} gives
\[
\AUC(t)\le 1-\AUC(t^\star) + C_H\Big(\frac{\Delta}{\lambda_0}\Big)^{1/3}.
\]
Equivalently,
\begin{equation}\label{eq:sign-ineq}
\AUC(t^\star)-\Delta \ \le\ 1-\AUC(t^\star) + C_H\Big(\frac{\Delta}{\lambda_0}\Big)^{1/3}.
\end{equation}

We now lower bound the separation $\AUC(t^\star)-\tfrac12$ using only the constants in the assumptions.
Let $s=\eta(w)=\Psi^\star(t^\star(w))$.
Since $p(r)\le \overline p$ on $I$ and integrates to $1$, we have $|I|=\overline{t^\star}-\underline{t^\star}\ge 1/\overline p$.
Since $(\Psi^\star)'\ge \underline\psi$, the support $J=[\Psi^\star(\underline{t^\star}),\Psi^\star(\overline{t^\star})]$ has length at least
\[
|J| \ge \underline\psi\,|I|\ \ge\ \frac{\underline\psi}{\overline p}.
\]
Also $s$ has density $q\ge \underline q$ on $J$.
Let $J_L$ and $J_R$ be the leftmost and rightmost subintervals of $J$ of length $|J|/4$.
Then $\PP(s\in J_L)\ge \underline q\,|J|/4$ and similarly for $J_R$.
Hence, with $s_0,s_1$ i.i.d.\ copies of $s$,
\[
\E|s_1-s_0|
\ \ge\ \frac{|J|}{2}\,\PP(s_0\in J_L,s_1\in J_R)+\frac{|J|}{2}\,\PP(s_0\in J_R,s_1\in J_L)
\ \ge\ \frac{|J|}{2}\cdot 2\Big(\frac{\underline q\,|J|}{4}\Big)^2
\ =\ \frac{1}{16}\,\underline q^2\,|J|^3
\ \ge\ \gamma_0.
\]
By \cref{eq:auc-tstar} and $(4\pi_0\pi_1)^{-1}\ge 1$, we get
\[
\AUC(t^\star)-\frac12=\frac{1}{4\pi_0\pi_1}\E|\eta_1-\eta_0|\ \ge\ \E|\eta_1-\eta_0|\ \ge\ \gamma_0.
\]
Therefore $\AUC(t^\star)\ge \tfrac12+\gamma_0$ and $\AUC(-t^\star)=1-\AUC(t^\star)\le \tfrac12-\gamma_0$.

Now assume $\Delta\le \epsilon_0$.
Then $\AUC(t)=\AUC(t^\star)-\Delta\ge \tfrac12+\gamma_0-\epsilon_0\ge \tfrac12+\gamma_0/2$.
On the other hand, \cref{eq:sign-ineq} and $\Delta\le\epsilon_0\le \lambda_0(\gamma_0/(2C_H))^3$ give
\[
\AUC(t)\le \AUC(-t^\star)+C_H\Big(\frac{\Delta}{\lambda_0}\Big)^{1/3}
\le \Big(\tfrac12-\gamma_0\Big) + \frac{\gamma_0}{2}
= \tfrac12-\frac{\gamma_0}{2},
\]
a contradiction. Hence $\hat a\ge 0$ whenever $\Delta\le\epsilon_0$, and thus
\[
\inf_{a\in\R}\ \|t^\star-a t\|^2=\inf_{a\ge 0}\ \|t^\star-a t\|^2
\qquad\text{for all }t\in\mathcal T\text{ with }\AUC(t^\star)-\AUC(t)\le\epsilon_0.
\]

To complete the theorem, combine the above with \cref{eq:global-unrestricted}.
\end{proof}

\subsubsection{Proof of \cref{lemma:entropy}}

\begin{proof}
Define $R_t = \{(w,w'):\ t(w)-t(w')\ge 0\}$, $S\triangle S'=(S\backslash S')\cup(S'\backslash S)$, and $d_\triangle(S,S')=P^2(S\triangle S')$ for $S,S'\subset(\mathcal W)^2$. We will show that for all $t,s\in\Tcal$, we have $d_\triangle(R_t,R_s)\le c_{q}\,\|t-s\|_q^{\beta}$ for $c_{q}=2^{\,1+\beta}\,c_{\rm margin}^{\,q/(\alpha+q)}$.

Fix $t,s\in\mathcal T$ and write $\delta=t-s$.
Let $D_t=t(w)-t(w')$ and $D_s=s(w)-s(w')$.
On the symmetric difference $R_t\triangle R_s$, the quantities $D_t$ and $D_s$ have opposite negativity/nonnegativity,
hence $|D_t|\le |D_t-D_s|$. Since
\[
D_t-D_s \;=\; (t-s)(w)-(t-s)(w')\;=\;\delta(w)-\delta(w'),
\]
we have the inclusion
\[
R_t\triangle R_s \ \subseteq\ \bigl\{|D_t|\le |\delta(w)-\delta(w')|\bigr\}.
\]
Therefore, for any $\lambda>0$,
\[
R_t\triangle R_s
\subseteq
\bigl\{|D_t|\le \lambda\bigr\}\ \cup\ \bigl\{|\delta(w)-\delta(w')|>\lambda\bigr\},
\]
and thus
\begin{equation}\label{eq:split}
d_\triangle(R_t,R_s)
\le
P^2(|D_t|\le \lambda)\;+\;P^2(|\delta(w)-\delta(w')|>\lambda).
\end{equation}
The assumed margin condition gives $P^2(|D_t|\le \lambda)\le c_{\rm margin}\lambda^\alpha$.
For the second term, Markov's inequality and $|a-b|^q\le 2^{q-1}(|a|^q+|b|^q)$ yield
\[
P^2(|\delta(w)-\delta(w')|>\lambda)
\le \frac{\mathbb E|\delta(w)-\delta(w')|^q}{\lambda^q}
\le \frac{2^{q}\,\mathbb E|\delta(w)|^q}{\lambda^q}
= \frac{2^{q}\,\|t-s\|_q^q}{\lambda^q}.
\]
Plugging into \cref{eq:split} gives
\[
d_\triangle(R_t,R_s)
\le c_{\rm margin}\lambda^\alpha + \frac{2^q\|t-s\|_q^q}{\lambda^q}.
\]
Choose $\lambda = \bigl(2^q\|t-s\|_q^q/c_{\rm margin}\bigr)^{1/(\alpha+q)}$, which equalizes the two terms, to obtain
$
d_\triangle(R_t,R_s)
\le
2\,c_{\rm margin}^{q/(\alpha+q)}\,(2^q\|t-s\|_q^q)^{\alpha/(\alpha+q)}.
$
\end{proof}

\subsubsection{Proof of \cref{thm:rspo-rate}}

To prove \cref{thm:rspo-rate}, we first consider the implication of theorem 5 of \citet{clemenccon2008ranking} under symmetric difference entropy conditions in order obtain rates beyond the case of VC classes. Then we show that a Tsybakov-style margin condition implies their generic noise condition. Then we use these building blocks to prove \cref{thm:rspo-rate}.

\subsubsection{Guarantees for ranking empirical risk optimization under symmetric-difference entropy conditions}

For this subsection, we abandon all previously defined notation and conventions and adopt the notation of \citet{clemenccon2008ranking}.

\subsubsubsection{Setup and notation (following \citep{clemenccon2008ranking})}

Let $(X,Y)$ be a random pair taking values in $\Xcal\times\R$, and let $(X',Y')$ be an independent copy.
Write $Z=(Y-Y')/2$.
A (pairwise) ranking rule is a measurable function $r:\Xcal\times\Xcal\to\{-1,1\}$.
 Throughout we consider \emph{symmetric} ranking rules (as in the paper): $r(x,x')=-r(x',x)$, and extend the co-domain just for identical inputs in order to set $r(x,x)=0$.

The ranking risk is
\[
L(r)=\Pbb\{Z\,r(X,X')<0\}.
\]
Let $r^\star$ denote a risk minimizer over all measurable ranking rules and $L^\star=L(r^\star)$.

Given i.i.d.\ training data $D_n=(X_1,Y_1),\dots,(X_n,Y_n)$, define the empirical ranking risk (a $U$-statistic)
\[
L_n(r)=\frac{1}{n(n-1)}\sum_{i\neq j}\indic{Z_{i,j}\,r(X_i,X_j)<0},\qquad Z_{i,j}=(Y_i-Y_j)/2,
\]
and an empirical risk minimizer over $\Rcal$ by
\[
r_n\in\arg\min_{r\in\Rcal}L_n(r).
\]

In Section~4 of the paper the excess risk $\Lambda(r)=L(r)-L^\star$ is written as a $U$-statistic with kernel
\[
q_r\big((x,y),(x',y')\big)=\indic{(y-y')\,r(x,x')<0}-\indic{(y-y')\,r^\star(x,x')<0},
\]
so that $\Lambda(r)=\E\,q_r\big((X,Y),(X',Y')\big)$ and its empirical estimate $\Lambda_n(r)$ admits the Hoeffding decomposition
\[
\Lambda_n(r)-\Lambda(r)=2T_n(r)+W_n(r),
\]
where $T_n(r)=(1/n)\sum_{i=1}^n h_r(X_i,Y_i)$ and $W_n(r)$ is a \emph{degenerate} $U$-statistic with kernel $\widetilde h_r$ (see the paper for the exact definitions).

\begin{assumption}[Assumption 4 of \citet{clemenccon2008ranking}]\label{asm:clemenccon}
There exist constants $c>0$ and $\alpha\in[0,1]$ such that for all $r\in\Rcal$,
\[
\Var\big(h_r(X,Y)\big)\le c\,\Lambda(r)^\alpha.
\]
\end{assumption}

We next restate theorem~5 of \citet{clemenccon2008ranking} for reference. Let $\varepsilon_1,\dots,\varepsilon_n$ be i.i.d.\ Rademacher signs independent of the data.
Define the three (random) complexity quantities
\[
Z_\varepsilon=\sup_{r\in\Rcal}\sum_{i,j}\varepsilon_i\varepsilon_j\,\widetilde h_r\big((X_i,Y_i),(X_j,Y_j)\big),
\]
\[
U_\varepsilon=\sup_{r\in\Rcal}\sup_{\|\beta\|_2\le 1}\sum_{i,j}\varepsilon_i\beta_j\,\widetilde h_r\big((X_i,Y_i),(X_j,Y_j)\big),
\]
\[
M=\sup_{r\in\Rcal}\sup_{1\le k\le n}\sum_{i=1}^n \varepsilon_i\,\widetilde h_r\big((X_i,Y_i),(X_k,Y_k)\big).
\]
Also define the centered empirical process $\nu_n$ based on the ``loss'' $\ell$ in the paper:
\[
\ell(r,(x,y))=2\,\E\big[\indic{(y-Y)\,r(x,X)<0}\big]-L(r),
\qquad
\nu_n(r)=\frac{1}{n}\sum_{i=1}^n \ell(r,(X_i,Y_i))-L(r).
\]
Finally define the pseudo-distance
\[
d(r,r')=\Big(\E\big[\E\big[\indic{r(X,X')\neq r'(X,X')}\mid X\big]^2\big]\Big)^{1/2}.
\]
Let $\phi$ be any nondecreasing function such that $\phi(x)/x$ is nonincreasing, $\phi(1)\ge 1$, and
\[
\sqrt n\,\E\sup_{r'\in\Rcal:\,d(r,r')\le \sigma}\big|\nu_n(r)-\nu_n(r')\big|\le \phi(\sigma)
\quad\text{for all }r\in\Rcal,\;\sigma>0.
\]
Let $\rho>0$ be the unique solution of $\sqrt n\,\rho^2=\phi(\rho^\alpha)$.

\begin{theorem}[Theorem 5 of \citet{clemenccon2008ranking}]\label{thm:CLV5}
Assume \cref{asm:clemenccon} holds. Then there exists a universal constant $C$ such that for all $\delta\in(0,1)$, with probability at least $1-\delta$,
\[
L(r_n)-L^\star
\le 2\inf_{r\in\Rcal}\big(L(r)-L^\star\big)
+ C\left(
\frac{\E Z_\varepsilon}{n^2}
+\frac{\E U_\varepsilon}{n^2}\sqrt{\log\frac1\delta}
+\frac{\E M}{n^2}\log\frac1\delta
+\frac{1}{n}\log\frac1\delta
+\rho^2\log\frac1\delta
\right).
\]
\end{theorem}

\subsubsubsection{Symmetric-difference covering numbers}

For a ranker $r\in\Rcal$, define the set $A_r=\{(x,x')\in\Xcal\times\Xcal:\ r(x,x')=1\}$.
Define the symmetric-difference pseudo-metric
\[
\Delta(r,r')=\Pbb\big(r(X,X')\neq r'(X,X')\big)=\Pbb\big((X,X')\in A_r\triangle A_{r'}\big).
\]
For $\varepsilon\in(0,1]$, define the symmetric-difference covering number $N_\triangle(\varepsilon,\Rcal)$
as the smallest $m$ such that there exist sets $S_1,\dots,S_m\subseteq \Xcal\times\Xcal$ with the property that
for each $r\in\Rcal$ there exists $j$ such that $\Pbb(A_r\triangle S_j)\le \varepsilon$. (Notice that the slightly different notation for $N_\triangle$ compared to the main text, since here $\Rcal$ is a class of indicators.)
Equivalently (up to the identification $r\leftrightarrow A_r$), $N_\triangle(\varepsilon,\Rcal)$ is the $\varepsilon$-covering number
of $\Rcal$ under the pseudo-metric $\Delta$.

We will consider either of the following entropy assumptions.

\subsubsubsection{From symmetric-difference entropy to the quantities in Theorem~\ref{thm:CLV5}}

We need to control the expectations $\E Z_\varepsilon$, $\E U_\varepsilon$, $\E M$ and the fixed point $\rho$.
The following reductions are elementary and isolate the role of the entropy.

\begin{lemma}[Comparing $\Delta$ and $d$]\label{lem:dDelta}
For all $r,r'\in\Rcal$ we have $d(r,r')^2\le \Delta(r,r')$.
Consequently, for all $\sigma\in(0,1]$,
\[
N\big(\sigma,\Rcal,d\big)\le N_\triangle(\sigma^2,\Rcal),
\qquad
\log N(\sigma,\Rcal,d)\le \log N_\triangle(\sigma^2,\Rcal),
\]
where $N(\sigma,\Rcal,d)$ is the $\sigma$-covering number with respect to the pseudo-metric $d$.
\end{lemma}

\begin{proof}
Let $\Delta_X(r,r')=\Pbb\big(r(X,X')\neq r'(X,X')\mid X\big)$.
Then $d(r,r')^2=\E[\Delta_X(r,r')^2]\le \E[\Delta_X(r,r')]=\Delta(r,r')$ by Jensen's inequality.
The covering-number inequality follows because every $\sigma^2$-cover under $\Delta$ is a $\sigma$-cover under $d$.
\end{proof}

\begin{lemma}[Lipschitz control of the degenerate kernels]\label{lem:kernelLip}
There is a universal constant $B_0<\infty$ such that for all $r,r'\in\Rcal$:
\begin{enumerate}
\item[(i)] $\|\widetilde h_r\|_\infty\le B_0$,
\item[(ii)] $\|\widetilde h_r-\widetilde h_{r'}\|_{L_2(\Pbb_{(X,Y)}^2)}\le B_0\,\Delta(r,r')^{1/2}$.
\end{enumerate}
Consequently, the $L_2(\Pbb_{(X,Y)}^2)$ covering numbers of the kernel class
$\widetilde{\mathcal H}=\{\widetilde h_r:\ r\in\Rcal\}$ satisfy
\[
\log N_2(u,\widetilde{\mathcal H})
\;\le\;
\log N_\triangle\!\left(\frac{u^2}{B_0^2},\Rcal\right),
\qquad u\in(0,B_0],
\]
where $N_2$ denotes covering numbers in $L_2(\Pbb_{(X,Y)}^2)$.
\end{lemma}

\begin{proof}
We use only the definitions in Section~4 of \citet{clemenccon2008ranking}.
Since $q_r$ is a difference of indicators, $\|q_r\|_\infty\le 1$.
By definition,
\[
\widetilde h_r=q_r-\Lambda(r)-h_r(\cdot)-h_r(\cdot),
\]
so $\|\widetilde h_r\|_\infty\le 1+|\Lambda(r)|+2\|h_r\|_\infty\le 4$ because $|\Lambda(r)|\le 1$ and $\|h_r\|_\infty\le 1$.
Thus (i) holds with (say) $B_0=4$.

For (ii), observe that for any two rankers $r,r'$,
\[
\big|q_r\big((x,y),(x',y')\big)-q_{r'}\big((x,y),(x',y')\big)\big|
\le \indic{r(x,x')\neq r'(x,x')}.
\]
Therefore,
\[
\|q_r-q_{r'}\|_{L_2(\Pbb_{(X,Y)}^2)}^2
\le \Pbb\big(r(X,X')\neq r'(X,X')\big)=\Delta(r,r').
\]
Next, $h_r(x,y)=\E[q_r((x,y),(X',Y'))]-\Lambda(r)$, hence
\[
|h_r(x,y)-h_{r'}(x,y)|
\le \E\big|q_r((x,y),(X',Y'))-q_{r'}((x,y),(X',Y'))\big|+|\Lambda(r)-\Lambda(r')|
\le \Delta_x(r,r')+\Delta(r,r'),
\]
where $\Delta_x(r,r')=\Pbb(r(x,X')\neq r'(x,X'))$.
Taking $L_2(\Pbb_{(X,Y)})$ norms and using $\Delta(r,r')\le \Delta(r,r')^{1/2}$ (since $\Delta\le 1$) gives
\[
\|h_r-h_{r'}\|_{L_2(\Pbb_{(X,Y)})}
\le \|\Delta_X(r,r')\|_{L_2(\Pbb_X)}+\Delta(r,r')^{1/2}
= d(r,r')+\Delta(r,r')^{1/2}
\le 2\Delta(r,r')^{1/2},
\]
where we used Lemma~\ref{lem:dDelta} in the last step.
Finally, by the triangle inequality,
\[
\|\widetilde h_r-\widetilde h_{r'}\|_2
\le \|q_r-q_{r'}\|_2 + |\Lambda(r)-\Lambda(r')| + 2\|h_r-h_{r'}\|_2
\le \Delta^{1/2}+\Delta^{1/2}+4\Delta^{1/2}\le 6\Delta^{1/2}.
\]
Thus (ii) holds for $B_0=6$, and we may keep $B_0=6$ for the covering-number statement.
\end{proof}

\subsubsubsection{Complexity proxies}
For $n\ge 2$, define the entropy integral proxy
\begin{equation}\label{eq:JnDef}
\mathcal J_n
=
\inf_{0<\eta\le 1}
\left\{
\eta + \frac{1}{\sqrt n}\int_{\eta}^{1}\sqrt{\log N_\triangle(u^2,\Rcal)}\,du
\right\}.
\end{equation}
For the second-order chaos term, also define
\begin{equation}\label{eq:KnDef}
\mathcal K_n
=
\inf_{0<\eta\le 1}
\left\{
\eta + \frac{1}{n}\int_{\eta}^{1}\log N_\triangle(u^2,\Rcal)\,du
\right\}.
\end{equation}
(Any other fixed upper limit in $(0,\infty)$ is equivalent up to constants; the choice $1$ is only for notational convenience.)

\begin{lemma}[Controlling $\E Z_\varepsilon,\E U_\varepsilon,\E M$ via entropy]\label{lem:ZUM}
There exists a universal constant $C_0<\infty$ such that
\[
\E Z_\varepsilon \le C_0\,\bigl(n^2\mathcal K_n+n\bigr),\qquad
\E U_\varepsilon \le C_0\,\bigl(n^{3/2}\mathcal J_n+n\bigr),\qquad
\E M \le C_0\,n.
\]
\end{lemma}

\begin{proof}
Write $Z_i=(X_i,Y_i)$ and $\widetilde{\mathcal H}=\{\widetilde h_r:\ r\in\Rcal\}$.
Let $B_0=\sup_{h\in\widetilde{\mathcal H}}\|h\|_\infty<\infty$ be as in Lemma~\ref{lem:kernelLip} and, for $u>0$,
let $N_2(u)=N_2(u,\widetilde{\mathcal H})$ be the $L_2(\Pbb_{(X,Y)}^2)$ covering number of $\widetilde{\mathcal H}$.
By Lemma~\ref{lem:kernelLip} and the change of variables $v=u/B_0$,
\[
\inf_{0<\eta\le B_0}\left\{\eta+\frac{1}{\sqrt n}\int_\eta^{B_0}\sqrt{\log N_2(u)}\,du\right\}
\lesssim \mathcal J_n
\]
and
\[
\inf_{0<\eta\le B_0}\left\{\eta+\frac{1}{n}\int_\eta^{B_0}\log N_2(u)\,du\right\}
\lesssim \mathcal K_n.
\]

\medskip
\noindent\textbf{Bounding $\E Z_\varepsilon$.}
The off-diagonal part of $Z_\varepsilon$ is an order-$2$ Rademacher chaos indexed by the bounded canonical kernels
$\widetilde{\mathcal H}$. The order-$2$ chaos moment bound of \citet[Prop.~2.2]{ArconesGine93}
(see also \citep[Cor.~3.2.7]{deLaPenaGine99}) gives $\psi_1$ increments with metric
$n\|\widetilde h-\widetilde h'\|_{L_2(\Pbb_{(X,Y)}^2)}$. Applying the stopped-chain bound of
\citet[Thm.~3.2]{Dirksen2015Tail} with $p=n$ and the entropy estimate for $\gamma_{1,n}$
\citep[Eq.~(4)]{Dirksen2015Tail} yields
\[
\begin{aligned}
\E Z_\varepsilon
\le C\Bigg[
&n^2\inf_{0<\eta\le B_0}\left\{\eta+\frac{1}{n}\int_\eta^{B_0}\log N_2(u)\,du\right\}\\
+n
\Bigg].
\end{aligned}
\]
The last $n$ term accounts for the diagonal and for the value of the chaos at a fixed kernel. The covering comparison
above gives $\E Z_\varepsilon\lesssim n^2\mathcal K_n+n$.

\medskip
\noindent\textbf{Bounding $\E U_\varepsilon$.}
Likewise, $U_\varepsilon$ is the Hilbert-valued Rademacher process associated to $\widetilde{\mathcal H}$.
Applying \citet[Thm.~3.2 and Eq.~(4)]{Dirksen2015Tail} in the subgaussian case, with metric
$n\|\widetilde h-\widetilde h'\|_{L_2(\Pbb_{(X,Y)}^2)}$ and $p=n$, gives
\[
\E U_\varepsilon
\le
C\left[
n^{3/2}\inf_{0<\eta\le B_0}\left\{\eta+\frac{1}{\sqrt n}\int_\eta^{B_0}\sqrt{\log N_2(u)}\,du\right\}
+n
\right]
\lesssim n^{3/2}\mathcal J_n+n.
\]
The additive $n$ term is the contribution of any fixed kernel before chaining over increments.

\medskip
\noindent\textbf{Bounding $\E M$.}
For every $r$ and $k$, $\big|\sum_i\varepsilon_i\widetilde h_r(Z_i,Z_k)\big|\le nB_0$, hence $M\le nB_0$ deterministically.
Collecting the three estimates proves the lemma.
\end{proof}

\subsubsubsection{A usable bound for the modulus function $\phi$ and the fixed point $\rho$}

We will use the following symmetrization + chaining estimate for the ordinary empirical process $\nu_n$.
Let $\mathcal L=\{\ell(r,\cdot):r\in\Rcal\}$ be the loss class in the notation of Theorem~\ref{thm:CLV5}.

\begin{lemma}[A choice of $\phi$ from metric entropy]\label{lem:phi}
There exists a universal constant $C_1<\infty$ such that for every $n\ge 2$, one may take
\[
\phi(\sigma)=
C_1\inf_{0<\eta\le \sigma}\left\{
\eta\sqrt{\log N(\eta,\Rcal,d)}+\int_{\eta}^{\sigma}\sqrt{\log N(u,\Rcal,d)}\,du
\right\},
\qquad \sigma\in(0,1],
\]
so that the condition in Theorem~\ref{thm:CLV5} holds.
In particular, using Lemma~\ref{lem:dDelta}, one has the crude upper bound
\[
\phi(\sigma)
\le
C_1\inf_{0<\eta\le \sigma}\left\{
\eta\sqrt{\log N_\triangle(\eta^2,\Rcal)}+\int_{\eta}^{\sigma}\sqrt{\log N_\triangle(u^2,\Rcal)}\,du
\right\}.
\]
\end{lemma}

\begin{proof}
Fix $r\in\Rcal$ and consider the increment class
$\mathcal F_{r,\sigma}=\{\ell(r,\cdot)-\ell(r',\cdot):\ r'\in\Rcal,\ d(r,r')\le \sigma\}$.
Then
\[
\nu_n(r)-\nu_n(r')=(\Pbb_n-\Pbb)\big(\ell(r,\cdot)-\ell(r',\cdot)\big),
\]
so by symmetrization,
\[
\E\sup_{d(r,r')\le \sigma}|\nu_n(r)-\nu_n(r')|
\le
2\E\sup_{f\in\mathcal F_{r,\sigma}}\left|\frac{1}{n}\sum_{i=1}^n \varepsilon_i f(X_i,Y_i)\right|.
\]
By Dudley's entropy integral bound for sub-Gaussian processes \citep[Thm.~5.22]{wainwright2019high},
for a universal constant $C_1$,
\[
\E\sup_{f\in\mathcal F_{r,\sigma}}\left|\frac{1}{n}\sum_{i=1}^n \varepsilon_i f(X_i,Y_i)\right|
\le
\frac{C_1}{\sqrt n}\inf_{0<\eta\le \sigma}\left\{
\eta\sqrt{\log N_2(\eta,\mathcal F_{r,\sigma})}+\int_{\eta}^{\sigma}\sqrt{\log N_2(u,\mathcal F_{r,\sigma})}\,du
\right\}.
\]
Finally, the map $r'\mapsto \ell(r,\cdot)-\ell(r',\cdot)$ is Lipschitz in $d$ (in $L_2(\Pbb)$) up to a universal constant,
so $N_2(u,\mathcal F_{r,\sigma})\le N(u,\Rcal,d)$ for all $u\le \sigma$, again up to constants.
Absorbing constants into $C_1$ yields the displayed formula for $\phi$.
The last displayed inequality follows from Lemma~\ref{lem:dDelta}.
\end{proof}

\subsubsubsection{Ranking suboptimality bound under symmetric-difference entropy}

We now state and prove the entropy-based instantiation of theorem~5 of \citet{clemenccon2008ranking}.

\begin{theorem}[Entropy-based instantiation of \cref{thm:CLV5}]\label{cor:main}
Assume \cref{asm:clemenccon} holds and let $r_n$ be an empirical minimizer of $L_n(r)$ over $\Rcal$.
There exists a constant $C<\infty$, depending only on the universal constants in Theorem~\ref{thm:CLV5}, on the constant $c$ from \cref{asm:clemenccon}, and on the fixed entropy exponents,
such that, for every $\delta\in(0,0.5)$ and all sufficiently large $n$, with probability at least $1-\delta$,
\[
L(r_n)-L^\star
\le 2\inf_{r\in\Rcal}\big(L(r)-L^\star\big)+ C\,\mathfrak R_{n,\delta},
\]
where $\mathfrak R_{n,\delta}$ can be taken as follows.

\medskip
\noindent\textbf{Case 1:} Assume $\log N_\triangle(\varepsilon,\Rcal)\le A\varepsilon^{-p}$ for some $A>0$ and $p>0$.
Then:
\begin{enumerate}
\item[$(1a)$] If $0<p<1$, then
\[
\mathfrak R_{n,\delta}
=
\left(\frac{A}{n}\right)^{\! \frac{1}{2-\alpha(1-p)}}\log\frac1\delta
+\frac{A+\log(1/\delta)}{n}.
\]
\item[$(1b)$] If $p=1$, then
\[
\mathfrak R_{n,\delta}
=
\frac{\sqrt A\,\log n}{\sqrt n}\,\log\frac1\delta
+\frac{A\log^2 n+\log(1/\delta)}{n}.
\]
\item[$(1c)$] If $p>1$, then
\[
\mathfrak R_{n,\delta}
=
A^{\frac{1}{2p}}\,n^{-\frac{1}{2p}}\log\frac1\delta
+A^{\frac{1}{p}}\,n^{-\frac{1}{p}}
+\frac{\log(1/\delta)}{n}.
\]
\end{enumerate}

\medskip
\noindent\textbf{Case 2:} Assume $N_\triangle(\varepsilon,\Rcal)\le (A/\varepsilon)^v$ for some $A\ge 1$ and $v\ge 1$.
Then
\[
\mathfrak R_{n,\delta}
=
\left(\frac{v\,\log(An)}{n}\right)^{\!\frac{1}{2-\alpha}}\,\log\frac1\delta
+\frac{v+\log(1/\delta)}{n}.
\]
\end{theorem}

\begin{proof}
We apply Theorem~\ref{thm:CLV5}. By Lemma~\ref{lem:ZUM},
\[
\begin{aligned}
\frac{\E Z_\varepsilon}{n^2}
&\le C_0\left(\mathcal K_n+\frac1n\right),\\
\frac{\E U_\varepsilon}{n^2}\sqrt{\log\frac1\delta}
&\le C_0\left(\frac{\mathcal J_n}{\sqrt n}+\frac1n\right)\sqrt{\log\frac1\delta},\\
\frac{\E M}{n^2}\log\frac1\delta
&\le C_0\frac{\log(1/\delta)}{n}.
\end{aligned}
\]
The terms $1/n$ and $n^{-1}\sqrt{\log(1/\delta)}$ are absorbed by $C\log(1/\delta)/n$ because $\delta<0.5$.
Likewise, by $ab\le (a^2+b^2)/2$,
\[
\frac{\mathcal J_n}{\sqrt n}\sqrt{\log\frac1\delta}
\le \frac12\mathcal J_n^2+\frac12\frac{\log(1/\delta)}{n}.
\]
Thus Theorem~\ref{thm:CLV5} implies, with probability at least $1-\delta$,
\begin{equation}\label{eq:mainbound1}
L(r_n)-L^\star
\le
2\inf_{r\in\Rcal}\big(L(r)-L^\star\big)
+ C\left(\mathcal K_n+\mathcal J_n^2+\frac{\log(1/\delta)}{n}+\rho^2\log\frac1\delta\right),
\end{equation}
for a possibly enlarged universal constant $C$.

It remains to control $\mathcal J_n$, $\mathcal K_n$, and $\rho$ in each entropy regime.

\subparagraph*{Bounding $\mathcal J_n$.}
Under Case 1, we have $\log N_\triangle(u^2,\Rcal)\le A u^{-2p}$ for $u\in(0,1]$, hence
\[
\mathcal J_n
\le
\inf_{0<\eta\le 1}\left\{\eta+\frac{\sqrt A}{\sqrt n}\int_\eta^1 u^{-p}\,du\right\}.
\]
If $0<p<1$, the integral is bounded by $(1-p)^{-1}$, so $\mathcal J_n^2\lesssim A/n$.
If $p=1$, the integral equals $\log(1/\eta)$ and optimizing at $\eta\asymp \sqrt A/\sqrt n$ yields $\mathcal J_n\lesssim (\sqrt A/\sqrt n)\log n$ and thus $\mathcal J_n^2\lesssim A\log^2 n/n$.
If $p>1$, the integral behaves like $\eta^{1-p}$ and optimizing at $\eta\asymp ( \sqrt A/\sqrt n)^{1/p}$ yields $\mathcal J_n\lesssim A^{1/(2p)}n^{-1/(2p)}$ and thus $\mathcal J_n^2\lesssim A^{1/p}n^{-1/p}$.

Under Case 2, we have $\log N_\triangle(u^2,\Rcal)\le v\log(A/u^2)$ and the integral
$\int_0^1\sqrt{\log(A/u^2)}\,du$ is finite (as $\int_0^1\sqrt{\log(1/u)}\,du<\infty$), so $\mathcal J_n\lesssim \sqrt{v/n}$ and hence $\mathcal J_n^2\lesssim v/n$.

\subparagraph*{Bounding $\mathcal K_n$.}
Under Case 1,
\[
\mathcal K_n
\le
\inf_{0<\eta\le 1}\left\{\eta+\frac{A}{n}\int_\eta^1 u^{-2p}\,du\right\}.
\]
Thus $\mathcal K_n\lesssim A/n$ if $0<p<1/2$, $\mathcal K_n\lesssim A\log n/n$ if $p=1/2$, and
$\mathcal K_n\lesssim A^{1/(2p)}n^{-1/(2p)}$ if $p>1/2$.
For $0<p<1$, these bounds are dominated by
$C\{(A/n)^{1/(2-\alpha(1-p))}+A/n\}$ for sufficiently large $n$.
For $p=1$, $\mathcal K_n\lesssim \sqrt A/\sqrt n$, and for $p>1$,
$\mathcal K_n\lesssim A^{1/(2p)}n^{-1/(2p)}$.
Under Case 2,
\[
\mathcal K_n
\le
\inf_{0<\eta\le 1}\left\{\eta+\frac{v}{n}\int_\eta^1\log(A/u^2)\,du\right\}
\lesssim \frac{v\log(An)}{n},
\]
which is dominated by $C(v\log(An)/n)^{1/(2-\alpha)}$ for sufficiently large $n$.

\subparagraph*{Bounding $\rho$.}
We use Lemma~\ref{lem:phi} to select an admissible modulus $\phi$ and then solve (or upper bound) the fixed point equation
$\sqrt n\,\rho^2=\phi(\rho^\alpha)$.

\smallskip
\noindent\emph{Case 1 with $0<p<1$.}
Lemma~\ref{lem:phi} and Lemma~\ref{lem:dDelta} imply that we may take $\phi(\sigma)\lesssim \int_0^\sigma \sqrt{\log N_\triangle(u^2,\Rcal)}\,du$.
Under Case 1, this gives
\[
\phi(\sigma)\lesssim \sqrt A\int_0^\sigma u^{-p}\,du \;\asymp\; \sqrt A\,\sigma^{1-p}.
\]
Thus the fixed point condition $\sqrt n\,\rho^2=\phi(\rho^\alpha)$ yields
\[
\sqrt n\,\rho^2 \;\lesssim\; \sqrt A\,(\rho^\alpha)^{1-p}=\sqrt A\,\rho^{\alpha(1-p)},
\]
equivalently $\rho^{2-\alpha(1-p)}\lesssim \sqrt A\,n^{-1/2}$, hence
\[
\rho^2\lesssim \left(\frac{A}{n}\right)^{\frac{1}{2-\alpha(1-p)}}.
\]

\smallskip
\noindent\emph{Case 1 with $p=1$.}
When $p=1$ the local entropy integral diverges at $0$; a convenient admissible choice is to upper bound local increments by global ones:
\[
\sup_{d(r,r')\le \sigma}|\nu_n(r)-\nu_n(r')|\le 2\sup_{r\in\Rcal}|\nu_n(r)|.
\]
By symmetrization and a truncated entropy-integral bound for Rademacher averages (as in the definition \cref{eq:JnDef}),
$\sqrt n\,\E\sup_{r}|\nu_n(r)|\lesssim \sqrt A \log n$ under $p=1$.
Thus we may take $\phi(\sigma)\equiv C\sqrt A\log n$ (constant in $\sigma$), which is admissible.
The fixed point equation then gives $\sqrt n\,\rho^2\lesssim \sqrt A\log n$, i.e.
\[
\rho^2\lesssim \frac{\sqrt A\,\log n}{\sqrt n}.
\]

\smallskip
\noindent\emph{Case 1 with $p>1$.}
Similarly, for $p>1$, the global (truncated) entropy integral implies
$\sqrt n\,\E\sup_{r}|\nu_n(r)|\lesssim A^{1/(2p)}\,n^{1/2-1/(2p)}$.
Thus we may take $\phi(\sigma)\equiv C A^{1/(2p)}\,n^{1/2-1/(2p)}$, yielding
\[
\rho^2\lesssim A^{1/(2p)}\,n^{-1/(2p)}.
\]

\smallskip
\noindent\emph{Case 2.}
Under Case 2 and Lemma~\ref{lem:phi}, we may bound
\[
\phi(\sigma)\lesssim \int_0^\sigma \sqrt{v\log(A/u^2)}\,du
\;\lesssim\;
\sigma\sqrt{v\log(A/\sigma^2)}\qquad (\sigma\in(0,1]),
\]
using the standard estimate $\int_0^\sigma \sqrt{\log(1/u)}\,du\lesssim \sigma\sqrt{\log(1/\sigma)}$.
Hence the fixed point condition implies
\[
\sqrt n\,\rho^2 \lesssim \rho^\alpha \sqrt{v\log\!\left(\frac{A}{\rho^{2\alpha}}\right)}.
\]
Set $\beta=1/(2-\alpha)$ and define $\bar\rho^2=D\,(v\log(An)/n)^\beta$ with $D$ large enough.
A direct substitution shows $\sqrt n\,\bar\rho^2 \gtrsim \bar\rho^\alpha \sqrt{v\log(A/\bar\rho^{2\alpha})}$, since
$\log(A/\bar\rho^{2\alpha})\lesssim \log(An)$ for $n\ge 2$ (and fixed $A,v,\alpha$).
By monotonicity of $\sigma\mapsto \sqrt n\,\sigma^2-\phi(\sigma^\alpha)$, the unique solution $\rho$ satisfies $\rho\le \bar\rho$, that is,
\[
\rho^2\lesssim \left(\frac{v\log(An)}{n}\right)^{\frac{1}{2-\alpha}}.
\]

\subparagraph*{Assemble the bounds.}
Plugging the bounds for $\mathcal K_n$, $\mathcal J_n^2$, and $\rho^2$ into \cref{eq:mainbound1} yields the four displayed forms of $\mathfrak R_{n,\delta}$.
\end{proof}

\subsubsubsection{A bipartite margin condition is sufficient for \cref{asm:clemenccon}}

We next show that the following margin condition is sufficient for \cref{asm:clemenccon} in the case of bipartite ranking. Recall that in the bipartite ranking setting, we define $\eta(X)=\PP(Y=1\mid X)$.
\begin{assumption}[Margin condition]\label{asm:tsybmargin}
There exist constants $c_{\rm margin}>0$ and $\alpha\in[0,1]$ such that for all $x\in\Xcal$ and $t>0$, we have
$\PP(\abs{\eta(X)-\eta(x)}\leq t)\leq ct^\alpha$.
\end{assumption}
This is similar to the standard classification margin condition \citep{Tsybakov2004,tsybakov2007fast}, except it holds around every feasible threshold $\eta(x)$ of the conditional probability $\eta(X)$ instead of just around $0$, which is the only threshold that matters for classification.

For a scoring function $s:\mathcal{X}\to\mathbb{R}$, denote by $L(s)$ the ranking
risk induced by $r_s(x,x')=2\indic{s(x)\ge s(x')}-1$ and by $L^*$ the Bayes ranking
risk. As in example~1 of \citet{clemenccon2008ranking}, the excess ranking risk is
\[
\Ecal(s)=L(s)-L^*=\E\Big[\,|\eta(X)-\eta(X')|\,\indic{(s(X)-s(X'))(\eta(X)-\eta(X'))<0}\Big].
\]
Let $h_s$ be the first-order term in the Hoeffding decomposition of the excess-risk
U-statistic $\Ecal_n(s)-\Ecal(s)$ as in section~4 of \citet{clemenccon2008ranking}.
\begin{lemma}\label{lem:noiseimplication}
Under \Cref{asm:tsybmargin}, for every scoring function $s$,
\[
\mathrm{Var}(h_s(X,Y)) \;\le\; (c_{\rm margin}+1)^2\,\Ecal(s)^\alpha.
\]
In particular, \cref{asm:clemenccon} holds with the same exponent $\alpha$.
\end{lemma}
\begin{proof}
Define the misranking event
\[
\mathcal{A}_s(x,x')=\{(s(x)-s(x'))(\eta(x)-\eta(x'))<0\}.
\]
The first line in the proof of proposition~7 in \citet{clemenccon2008ranking} yields the generic bound
\begin{equation}\label{eq:prop7start}
\mathrm{Var}(h_s(X,Y))
\;\le\;
\E\Big[\big(\E_{X'}\,\indic{\mathcal{A}_s(X,X')}\big)^2\Big].
\end{equation}
Fix $x\in\mathcal{X}$ and set
\[
g(x)=\E_{X'}\,\indic{\mathcal{A}_s(x,X')},\qquad
e(x)=\E_{X'}\Big[|\eta(x)-\eta(X')|\,\indic{\mathcal{A}_s(x,X')}\Big].
\]
Note that $0\le e(x)\le 1$ since $|\eta(\cdot)-\eta(\cdot)|\le 1$.
For any $t>0$, with $\Delta=|\eta(x)-\eta(X')|$,
\[
g(x)\le \PP(\Delta\le t)+\E_{X'}\big[\indic{\mathcal{A}_s(x,X')}\indic{\Delta>t}\big]
\le c_{\rm margin} t^\alpha + \frac{1}{t}\,e(x),
\]
because on $\{\Delta>t\}$ we have
$\indic{\mathcal{A}_s}\le \Delta\,\indic{\mathcal{A}_s}/t$.
Choose $t=e(x)^{1/(1+\alpha)}$ (if $e(x)=0$, the bound is trivial). Then
\[
g(x)\le (c_{\rm margin}+1)\,e(x)^{\alpha/(1+\alpha)} \quad\Longrightarrow\quad
g(x)^2 \le (c_{\rm margin}+1)^2\,e(x)^{2\alpha/(1+\alpha)}.
\]
Since $\alpha\in[0,1]$, one has $2\alpha/(1+\alpha)\ge \alpha$, and since $e(x)\in[0,1]$,
$e(x)^{2\alpha/(1+\alpha)}\le e(x)^\alpha$. Therefore
\[
g(x)^2 \le (c_{\rm margin}+1)^2\,e(x)^\alpha.
\]
Taking expectation over $X$ and using \cref{eq:prop7start} gives
\[
\mathrm{Var}(h_s(X,Y)) \le (c_{\rm margin}+1)^2\,\E\big[e(X)^\alpha\big].
\]
Because $u\mapsto u^\alpha$ is concave on $[0,\infty)$ for $\alpha\in[0,1]$, Jensen's
inequality yields $\E[e(X)^\alpha]\le (\E e(X))^\alpha$. Finally, by the definition
of $\Ecal(s)$, $\E e(X)=\Ecal(s)$, so
$\mathrm{Var}(h_s(X,Y)) \le (c_{\rm margin}+1)^2\,\Ecal(s)^\alpha$.
\end{proof}

\subsubsection{Proof of \cref{thm:rspo-rate}}

\begin{proof}
First, the density conditions in \cref{asm:rspo-regularity} imply \cref{asm:tsybmargin} with $\alpha=1$. Hence, by \cref{lem:noiseimplication}, \cref{asm:clemenccon} holds with $\alpha=1$. Apply \cref{cor:main} to bound $\mathrm{AUC}(\tst)-\mathrm{AUC}(t_{\hat\theta_{\rm RSPO}})$, which is the quantity denoted $L(r_n)-L^\star$ there. The approximation term $\inf_{r\in\Rcal}(L(r)-L^\star)$ is zero because $\tst\in\Tcal$ by \cref{asmp:div,asmp:realizability}. Enlarging the constant and using $\delta\le0.5$ leaves only the leading terms displayed in the theorem. For sufficiently large $n$, this AUC-suboptimality bound is smaller than the local radius $\iota$ in \cref{thm:auc-margin} with the stated probability, so \cref{thm:auc-margin} bounds $\inf_{a\geq0}\|\tst-a t_{\hat\theta_{\rm RSPO}}\|^2$. Taking square roots gives the displayed rate, since $\inf_{a\geq0}\|\tst-a t_{\hat\theta_{\rm RSPO}}\|$ dominates $\rho(h_{\hat\theta_{\rm RSPO}})$ under the coverage in \cref{asmp:basic}.
\end{proof}

\subsection{Proof of \cref{thm:empirical-beta-primitive}}\label{sec:proof:empirical-beta-primitive}
\begin{proof}
We proceed to prove this for each $\theta$ satisfying the conditions with $\theta$-independent constant. Fix $\theta$. Write $h=h_\theta$, $\beta_\star=\beta_{\kappa,\theta}$, and suppress the $\theta$ dependence of $\hat\Phi$ and $\hat\beta$.

Without loss of generality assume also $f'(1)=0$ (replace $f(t)$ by $f(t)-f'(1)(t-1)$, which leaves $D_f$ unchanged). 
Also, $\pi_{\beta,h}$ and $Z_\beta(x)$ are unchanged if we replace $h(x,y)$ by $h(x,y)-b(x)$ and $\lambda_\beta(x)$ by $\lambda_\beta(x)-b(x)$; the variance $v$ is unchanged as well, and values outside $\mathrm{supp}(\pi_{\rm ref}(\cdot\mid x))$ never enter the sums. We therefore use this invariance to center $h$ within each $x$. By the bounded-index and coverage conditions in \cref{asmp:basic}, for $P_x$-almost every $x$ and all $y,y'$ with $\pi_{\rm ref}(y\mid x)\pi_{\rm ref}(y'\mid x)>0$,
\[
\abs{h(x,y)-h(x,y')}=\abs{t_\theta(x;y,y')}\le M .
\]
Choosing $b(x)$ as the midpoint of the range of $h(x,\cdot)$ on $\mathrm{supp}(\pi_{\rm ref}(\cdot\mid x))$, we may assume throughout this proof that $\abs{h(x,y)}\le M/2\le M$ whenever $\pi_{\rm ref}(y\mid x)>0$.
Define
\[
Z_\beta(x)=D_f\!\bigl(\pi_{\beta,h}(\cdot\mid x)\,\big\|\,\pi_{\mathrm{ref}}(\cdot\mid x)\bigr),
\qquad
\Phi(\beta)=\mathbb E_x Z_\beta(x).
\]
Set $g=(f')^{-1}:\mathbb R\to(0,\infty)$, which under our assumptions is well-defined, $C^1$, and increasing.
Define
\[
c_-=g\!\left(-\frac{2M}{\underline\beta}\right),\qquad
c_+=g\!\left(\frac{2M}{\underline\beta}\right),
\]
and
\[
m_f=\inf_{t\in[c_-,c_+]} f''(t),\qquad
M_f=\sup_{t\in[c_-,c_+]} f''(t),\qquad
D_{\max}=\sup_{t\in[c_-,c_+]} f(t).
\]
Under our assumptions, these are finite and satisfy $0<m_f\le M_f<\infty$.

\textbf{Existence/uniqueness of $\lambda_\beta(x)$ and ratio bounds.}
Fix $x$ and $\beta$. Let $u_\beta(y)=\beta^{-1}(h(x,y)-\lambda)$ and define
\[
F(\lambda)=\sum_{y\in\mathcal Y}\pi_{\mathrm{ref}}(y\mid x)\,g(u_\beta(y)) - 1.
\]
Since $g$ is continuous and increasing, $F$ is continuous and strictly decreasing in $\lambda$.
Moreover, as $\lambda\to+\infty$, $u_\beta(y)\to-\infty$ so $g(u_\beta(y))\to 0$ and $F(\lambda)\to-1$; as $\lambda\to-\infty$, $u_\beta(y)\to+\infty$ so $g(u_\beta(y))\to+\infty$ and $F(\lambda)\to+\infty$.
Thus there exists a unique $\lambda_\beta(x)$ with $F(\lambda_\beta(x))=0$.

Let $h_{\max}(x)=\max_{y:\pi_{\rm ref}(y\mid x)>0} h(x,y)$ and $h_{\min}(x)=\min_{y:\pi_{\rm ref}(y\mid x)>0} h(x,y)$.
Because $g(0)=1$ (since $f'(1)=0$), if $\lambda=h_{\max}(x)$ then $u_\beta(y)\le 0$ for all $y$ with $\pi_{\rm ref}(y\mid x)>0$, hence $g(u_\beta(y))\le 1$ on the support and $\sum_y\pi_{\mathrm{ref}}(y\mid x) g(u_\beta(y))\le 1$; similarly if $\lambda=h_{\min}(x)$ then $\sum_y\pi_{\mathrm{ref}}(y\mid x) g(u_\beta(y))\ge 1$.
By monotonicity of $F$, we obtain
\[
h_{\min}(x)\ \le\ \lambda_\beta(x)\ \le\ h_{\max}(x).
\]
By the centered bound above, this implies $|\lambda_\beta(x)|\le M$ and therefore, on the support of $\pi_{\rm ref}(\cdot\mid x)$,
\[
|h(x,y)-\lambda_\beta(x)|\le 2M
\quad\Rightarrow\quad
u_\beta(x,y)\in\Bigl[-\frac{2M}{\beta},\frac{2M}{\beta}\Bigr]
\subseteq \Bigl[-\frac{2M}{\underline\beta},\frac{2M}{\underline\beta}\Bigr]
\quad\text{for all }\beta\in\mathcal B.
\]
Hence the ratio
\[
r_\beta(x,y)=\frac{\pi_{\beta,h}(y\mid x)}{\pi_{\mathrm{ref}}(y\mid x)}=g(u_\beta(x,y))
\]
lies in $[c_-,c_+]$ on the support of $\pi_{\rm ref}$ for all $\beta\in\mathcal B$.

\textbf{Boundedness of $Z_\beta(x)$.}
Since $r_\beta(x,y)\in[c_-,c_+]$ and $Z_\beta(x)=\sum_y\pi_{\mathrm{ref}}(y\mid x)f(r_\beta(x,y))$ with $\sum_y\pi_{\mathrm{ref}}(y\mid x)=1$,
\[
0\le Z_\beta(x)\le D_{\max}
\qquad\text{for all $x$ and $\beta\in\mathcal B$}.
\]

\textbf{Differentiability and bounds for $\Phi'(\beta)$ and Lipschitzness of $Z_\beta(\cdot)$.}
Fix $x$ and abbreviate $\lambda_\beta=\lambda_\beta(x)$, $u(y)=\beta^{-1}(h(x,y)-\lambda_\beta)$, $r(y)=g(u(y))$, and
\[
D_x(\beta)=Z_\beta(x)=\sum_y\pi_{\mathrm{ref}}(y\mid x)\,f(r(y)).
\]
The constraint $\sum_y\pi_{\mathrm{ref}}(y\mid x)\,r(y)=1$ can be written as
\[
G(\beta,\lambda)=\sum_y\pi_{\mathrm{ref}}(y\mid x)\,g\!\left(\beta^{-1}(h(x,y)-\lambda)\right)-1=0.
\]
Since $g$ is $C^1$ and $g'>0$, we have
\[
\partial_\lambda G(\beta,\lambda)
=-\frac{1}{\beta}\sum_y \pi_{\mathrm{ref}}(y\mid x)\,g'(u(y))<0,
\]
so by the implicit function theorem $\beta\mapsto \lambda_\beta$ is $C^1$ on $(0,\infty)$.

Differentiate $D_x(\beta)$ using $f'(r(y))=u(y)$ and $r'(y)=g'(u(y))u'(y)$:
\[
D_x'(\beta)
=\sum_y \pi_{\mathrm{ref}}(y\mid x)\,f'(r(y))\,r'(y)
=\sum_y w(y)\,u(y)\,u'(y),
\qquad
w(y)=\pi_{\mathrm{ref}}(y\mid x)\,g'(u(y)).
\]
Differentiating $u(y)=\beta^{-1}(h(x,y)-\lambda_\beta)$ gives
$u'(y)=-\beta^{-2}(h(x,y)-\lambda_\beta)-\beta^{-1}\lambda_\beta'$.
Differentiating the normalization $\sum_y\pi_{\mathrm{ref}}(y\mid x)r(y)=1$ yields $\sum_y w(y)u'(y)=0$, which implies
\[
\lambda_\beta'
=-\frac{1}{\beta}\frac{\sum_y w(y)\,(h(x,y)-\lambda_\beta)}{\sum_y w(y)}.
\]
Substituting this identity and simplifying (a standard variance algebra) gives
\begin{equation}\label{eq:Dxprime}
D_x'(\beta)
=-\frac{A(\beta)}{\beta^3}\,\mathrm{Var}_{p_\beta}\bigl(h(x,\cdot)\bigr),
\qquad
A(\beta)=\sum_y w(y),
\quad
p_\beta(y)=\frac{w(y)}{A(\beta)}.
\end{equation}

On $[c_-,c_+]$ we have $m_f\le f''\le M_f$, hence
\[
g'(u(y))=\frac{1}{f''(r(y))}\in\Bigl[\frac{1}{M_f},\frac{1}{m_f}\Bigr],
\]
so $A(\beta)\in[1/M_f,\,1/m_f]$ and thus from \cref{eq:Dxprime} and the centered bound on $h$,
\[
|D_x'(\beta)|
\le \frac{1}{m_f}\frac{1}{\underline\beta^{\,3}}\cdot \mathrm{Var}_{p_\beta}(h(x,\cdot))
\le \frac{1}{m_f}\frac{1}{\underline\beta^{\,3}}\cdot \mathbb E_{p_\beta}[h(x,\cdot)^2]
\le \frac{M^2}{m_f\,\underline\beta^{\,3}}.
\]
Therefore $\beta\mapsto Z_\beta(x)$ is Lipschitz on $\mathcal B$ with constant
\[
L_Z=\frac{M^2}{m_f\,\underline\beta^{\,3}}.
\]
Averaging \cref{eq:Dxprime} over $x$ yields $\Phi'(\beta)=\mathbb E_x D_x'(\beta)$.

To lower bound $-\Phi'$, note that $p_\beta(y)=\pi_{\mathrm{ref}}(y\mid x)\,g'(u(y))/A(\beta)$ and $g'(u(y))\ge 1/M_f$ while $A(\beta)\le 1/m_f$, hence
\[
p_\beta(y)\ \ge\ \frac{m_f}{M_f}\,\pi_{\mathrm{ref}}(y\mid x)
\qquad\text{for all $y$}.
\]
Thus for any $c\in\mathbb R$,
$\sum_y p_\beta(y)(h(x,y)-c)^2 \ge \frac{m_f}{M_f}\sum_y\pi_{\mathrm{ref}}(y\mid x)(h(x,y)-c)^2$,
and minimizing over $c$ gives
\[
\mathrm{Var}_{p_\beta}\bigl(h(x,\cdot)\bigr)
\ge \frac{m_f}{M_f}\,\mathrm{Var}_{\pi_{\mathrm{ref}}(\cdot\mid x)}\bigl(h(x,\cdot)\bigr).
\]
Using $A(\beta)\ge 1/M_f$ in \cref{eq:Dxprime}, for all $\beta\in\mathcal B$,
\[
-\Phi'(\beta)
=\mathbb E_x\frac{A(\beta)}{\beta^3}\mathrm{Var}_{p_\beta}(h(x,\cdot))
\ge \frac{1}{M_f}\frac{1}{\overline\beta^{\,3}}
     \cdot \frac{m_f}{M_f}\,
     \mathbb E_x\mathrm{Var}_{\pi_{\mathrm{ref}}(\cdot\mid x)}(h(x,\cdot))
=\ c_\Phi,
\]
where
\[
c_\Phi
=\frac{m_f}{M_f^2\,\overline\beta^{\,3}}\,v
\ >\ 0.
\]

\textbf{Existence and uniqueness of $\beta_\star$.}
By Step~3, $\Phi$ is $C^1$ on $(0,\infty)$ with $\Phi'(\beta)\le -c_\Phi<0$ on $\mathcal B$, hence strictly decreasing.
As $\beta\to\infty$, Step~1 implies $u_\beta(x,y)\to 0$ uniformly and thus $r_\beta(x,y)\to g(0)=1$, whence $Z_\beta(x)\to f(1)=0$ and $\Phi(\beta)\to 0$ by dominated convergence.
As $\beta\downarrow 0$, a standard vanishing-regularization argument shows $\Phi(\beta)\to\Psi$ (the unique minimum-divergence greedy limit), hence $\lim_{\beta\downarrow 0}\Phi(\beta)=\Psi>\kappa$.
Therefore, by continuity and strict monotonicity, there is a unique $\beta_\star\in(0,\infty)$ with $\Phi(\beta_\star)=\kappa$.
(And by construction $\beta_\star\in(\underline\beta,\overline\beta)$.)

\textbf{Lipschitzness of the policy map in $\beta$.}
Fix $x$ and $y$ and write $\pi_\beta=\pi_{\beta,h}$.
Since $\pi_\beta(y\mid x)=\pi_{\mathrm{ref}}(y\mid x)\,r_\beta(x,y)$ with $r_\beta=g(u_\beta)$, and $g'(u)\le 1/m_f$ on $\mathcal B$, it suffices to bound $\partial_\beta u_\beta$.
From $u_\beta(x,y)=\beta^{-1}(h(x,y)-\lambda_\beta(x))$ and $|h-\lambda_\beta|\le 2M$ (Step~1),
\[
\left|\partial_\beta u_\beta(x,y)\right|
\le \frac{|h(x,y)-\lambda_\beta(x)|}{\beta^2} + \frac{|\lambda_\beta'(x)|}{\beta}
\le \frac{2M}{\underline\beta^{\,2}}+\frac{|\lambda_\beta'(x)|}{\underline\beta}.
\]
Using the explicit formula for $\lambda_\beta'$ above and $|h-\lambda_\beta|\le 2M$,
\[
|\lambda_\beta'(x)|
=\frac{1}{\beta}\left|\frac{\sum_y w(y)(h(x,y)-\lambda_\beta(x))}{\sum_y w(y)}\right|
\le \frac{2M}{\underline\beta}.
\]
Thus $\sup_{\beta\in\mathcal B,x,y}|\partial_\beta u_\beta(x,y)|\le 4M/\underline\beta^{\,2}$ and hence
\[
\sup_{\beta\in\mathcal B,x,y}|\partial_\beta r_\beta(x,y)|
\le \frac{1}{m_f}\cdot \frac{4M}{\underline\beta^{\,2}}
=\frac{4M}{m_f\,\underline\beta^{\,2}}.
\]
Therefore for any $\beta,\beta'\in\mathcal B$ and any $x$,
\[
\sum_{y}\pi_{\mathrm{ref}}(y\mid x)\,
|\pi_\beta(y\mid x)-\pi_{\beta'}(y\mid x)|
=\sum_y \pi_{\mathrm{ref}}(y\mid x)^2\,|r_\beta-r_{\beta'}|
\le \sum_y \pi_{\mathrm{ref}}(y\mid x)\,|r_\beta-r_{\beta'}|
\le \frac{4M}{m_f\,\underline\beta^{\,2}}|\beta-\beta'|.
\]
Averaging over $x$ gives
\begin{equation}\label{eq:Lpi}
\mathbb E_x\sum_{y}\pi_{\mathrm{ref}}(y\mid x)\,
|\pi_{\beta,h}(y\mid x)-\pi_{\beta',h}(y\mid x)|
\le L_\pi|\beta-\beta'|,
\qquad
L_\pi=\frac{4M}{m_f\,\underline\beta^{\,2}}.
\end{equation}

\textbf{Uniform concentration of $\hat\Phi_m$ on $\mathcal B$.}
Let $N=m$ and $\Delta=(\overline\beta-\underline\beta)/N$, and consider the grid $\beta_k=\underline\beta+k\Delta$ for $k=0,\dots,N$.
Since $Z_\beta(x)$ is $L_Z$-Lipschitz in $\beta$ (Step~3), the standard discretization argument gives
\[
\sup_{\beta\in\mathcal B}|\hat\Phi_m(\beta)-\Phi(\beta)|
\le \max_{0\le k\le N}|\hat\Phi_m(\beta_k)-\Phi(\beta_k)| + 2L_Z\Delta.
\]
By Step~2, $0\le Z_{\beta_k}(x)\le D_{\max}$, so Hoeffding's inequality and a union bound yield that with probability at least $1-\delta$,
\[
\max_{0\le k\le N}|\hat\Phi_m(\beta_k)-\Phi(\beta_k)|
\le D_{\max}\sqrt{\frac{\log(2(N+1)/\delta)}{2m}}
= D_{\max}\sqrt{\frac{\log(2(m+1)/\delta)}{2m}}.
\]
Hence on this event,
\[
\varepsilon
=\sup_{\beta\in\mathcal B}|\hat\Phi_m(\beta)-\Phi(\beta)|
\le
D_{\max}\sqrt{\frac{\log(2(m+1)/\delta)}{2m}}
+2L_Z\frac{\overline\beta-\underline\beta}{m}.
\]

\textbf{Convert uniform calibration error to $\beta$-error to policy error.}
Because $\hat\beta$ minimizes $|\hat\Phi_m(\beta)-\kappa|$ over $\mathcal B$ and $\Phi(\beta_\star)=\kappa$,
\[
|\hat\Phi_m(\hat\beta)-\kappa|
\le |\hat\Phi_m(\beta_\star)-\kappa|
=|\hat\Phi_m(\beta_\star)-\Phi(\beta_\star)|
\le \varepsilon.
\]
Thus $|\Phi(\hat\beta)-\kappa|\le 2\varepsilon$, hence
$|\Phi(\hat\beta)-\Phi(\beta_\star)|\le 2\varepsilon$.
By the mean value theorem and $-\Phi'(\beta)\ge c_\Phi$ on $\mathcal B$ (Step~3),
\[
|\hat\beta-\beta_\star|\le \frac{2\varepsilon}{c_\Phi}.
\]
Finally, applying \cref{eq:Lpi} with $\beta=\hat\beta$ and $\beta'=\beta_\star$ yields
\[
\mathbb E_x\sum_y\pi_{\mathrm{ref}}(y\mid x)\,
|\pi_{\hat\beta,h}(y\mid x)-\pi_{\beta_\star,h}(y\mid x)|
\le L_\pi|\hat\beta-\beta_\star|
\le \frac{2L_\pi}{c_\Phi}\,\varepsilon.
\]
Substituting $L_\pi=\frac{4M}{m_f\underline\beta^{\,2}}$, $c_\Phi=\frac{m_f}{M_f^2\overline\beta^{\,3}}v$, and
$L_Z=\frac{M^2}{m_f\underline\beta^{\,3}}$ gives the bound
\[
\mathbb E_x\sum_{y\in\mathcal Y}\pi_{\mathrm{ref}}(y\mid x)\,
\bigl|\pi_{\kappa,h}(y\mid x)-\pi_{\hat\beta,h}(y\mid x)\bigr|
\ \le\
\frac{8M M_f^2\,\overline\beta^{\,3}}{m_f^{\,2}\,\underline\beta^{\,2}\,v}
\left(
D_{\max}\sqrt{\frac{\log(2(m+1)/\delta)}{2m}}
\;+\;
\frac{2M^2(\overline\beta-\underline\beta)}{m_f\,\underline\beta^{\,3}\,m}
\right).
\]
Noting that $\delta\leq0.5$ means the statement is true for some $c>0$.
\end{proof}

\section{Additional detail on experiments}\label{sec:additional_detail_on_experiments}

\subsection{Synthetic experiment}\label{apx:synth_exp}

We consider preference observations $(x,y_0,y_1,z)$ with $x\in\mathbb{R}^{20}$ drawn i.i.d.\ from a standard Gaussian and $y_0,y_1\in\{1,\dots,10\}$ sampled independently from the uniform reference policy $\pi_{\mathrm{ref}}(y\mid x)=0.1$. A teacher policy $\pi_{\theta^\star}$ is a randomly chosen feedforward network with two hidden layers (32 units each, ReLU activations) and a 10-way softmax output. The reward is defined as
$r^\star(x,y)=10 \log(\pi_{\theta^\star}(y\mid x)/\pi_{\mathrm{ref}}(y\mid x))$.
Given a pair $(y_0,y_1)$ we set $u=r^\star(x,y_1)-r^\star(x,y_0)$ and draw $z\sim\mathrm{Bern}(g(u))$ where
$g(u)=0.5\,\sigma(4(u-s))+0.5\,\sigma(4(u+s))$, $\sigma$ is the logistic link, and $s\in\{0,0.25,0.5,0.75,1.0,1.25,1.5\}$ controls misspecification.

Policies $\pi_\theta$ are parameterized by the same architecture as the teacher and trained by batched SGD (Adam, learning rate $2\times 10^{-3}$, 100 epochs, batch size 128). DPO minimizes the standard logistic loss on the implied reward index.
PSPO alternates between (i) isotonic regression for the nondecreasing link function via PAVA and (ii) SGD updates of $\theta$ using the fitted link; we use a short DPO warm-start and then 4 outer iterations with 2 inner epochs.
OSPO estimates the link by kernel regression of $z$ on the index $t_\theta(x,y_0,y_1)$ using a Gaussian kernel with bandwidth of $n^{-1/5}$ and an in-memory sample of all training points; we use a short DPO warm-start and align the sign of the implicit reward by the empirical covariance between $t_\theta$ and $z$ on the training data.
RSPO uses the PoP-DPO loss computed within batches (so a batch of $B$ preference observations involves $B^2$ loss terms).

For each learned policy $\pi_\theta$, we evaluate the family
$\pi_{\beta,\theta}(y\mid x)\propto \pi_{\mathrm{ref}}(y\mid x)\exp(\beta^{-1}h_\theta(x,y))$
with $h_\theta(x,y)=\log(\pi_\theta(y\mid x)/\pi_{\mathrm{ref}}(y\mid x))$ on an independent evaluation set of size $m=2000$. We scan a fixed grid of $\beta$ values and report the resulting average reward and KL divergence. For reward-vs-$s$ plots, we fix a KL budget $\kappa=0.2$ and, for each run, pick the two $\beta$ values that yield KL closest to $\kappa$ and interpolate linearly between them. We average the results over 1000 seeds.

\subparagraph*{Computer resources.}
The experiments from this subsection are run on a g5.48xlarge AWS instance. 

\subsection{Qwen3 alignment experiment}\label{apx:llm_exp}

\subparagraph*{Dataset and prompts.}
We use the UltraFeedback dataset (\texttt{openbmb/UltraFeedback}), split \texttt{train}. For each example we extract a prompt using the first available field among \texttt{prompt}, \texttt{instruction}, \texttt{query}, or \texttt{question}. We shuffle the dataset with a fixed seed and take a prefix of size $n=5000$ for each experiment.

\subparagraph*{Reference model and SFT.}
The base model is Qwen3-0.6B. The reference policy $\pi_{\mathrm{ref}}$ is a light SFT model trained as follows:
\begin{itemize}
  \item Data: 2,000 UltraFeedback examples, shuffled.
  \item Target response: the completion with the highest \texttt{overall\_score} (fallback to \texttt{fine-grained\_score}).
  \item Training: 1 epoch, learning rate $2\times 10^{-5}$, max length 1024, per-device batch size 1, BF16.
\end{itemize}

\subparagraph*{Reward model and preference generation.}
We use Skywork-Reward-V2-Qwen3-1.7B as the proxy reward model. For each prompt, we sample two responses from $\pi_{\mathrm{ref}}$ using nucleus sampling (temperature 0.7, top-$p$ 0.95) with a max generation length of 128 tokens. The reward model scores each prompt-response pair by concatenating prompt and response and taking the sequence classification logit.

We define the reward difference $\Delta r = r_1 - r_0$ and rescale it to have standard deviation 1 across the dataset. Preferences are sampled using the shifted-mixture logistic link function
\[
g(u) = \tfrac{1}{2}\sigma\left({u-s}\right) + \tfrac{1}{2}\sigma\left({u+s}\right),
\]
with shift $s$, where $\sigma$ is the logistic sigmoid. The absence of the factor $4$ used in the synthetic link reflects the unit-standard-deviation rescaling of $\Delta r$ here. We then draw $z \sim \mathrm{Bern}(g(\Delta r))$.

\subparagraph*{Policy parameterization and log-probabilities.}
All optimizers train a LoRA-adapted policy on top of the base model. We use LoRA with $r=16$, $\alpha=32$, and dropout 0.05. The reference model is kept frozen. Log-probabilities are computed only on response tokens by concatenating prompt and response, masking prompt tokens, and summing log-probs over response tokens.

\subparagraph*{Optimization methods.}
We train with AdamW (learning rate $5\times 10^{-6}$, weight decay 0), batch size 4, gradient accumulation 4, and a cosine schedule with 10 warmup steps. Each alignment run uses one GPU, so batch size 4 is the global minibatch and the effective optimizer batch size is 16. We clip gradient norm at 1.0 and skip updates on non-finite losses. Default training length is 100 minibatch steps (25 optimizer updates) in the main experiments.

\subparagraph*{DPO.}
We implement the standard DPO loss with temperature 1:
\[
\ell = \log\left(1 + \exp\left(\log\frac{\pi_\theta(y_{\mathrm{other}}\mid x)}{\pi_{\mathrm{ref}}(y_{\mathrm{other}}\mid x)} - \log\frac{\pi_\theta(y_{\mathrm{pref}}\mid x)}{\pi_{\mathrm{ref}}(y_{\mathrm{pref}}\mid x)}\right)\right).
\]

\subparagraph*{PSPO.}
We alternate between gradient updates on $\theta$ and isotonic regression to estimate $\Psi$.
At each update of $\Psi$, we compute $t_\theta(x,y_0,y_1)$ over the dataset and fit an isotonic regression with the Pool Adjacent Violators Algorithm (PAVA). The isotonic fit is updated every 50 gradient steps. To allow gradients through $\Psi$, we use a soft nearest-neighbor interpolation (softmax over distances with temperature 0.1) and mix in a small logistic term (5\%) to avoid flat regions. Outputs are clipped to $[10^{-6},1-10^{-6}]$.

\subparagraph*{OSPO.}
We estimate $g^\star_\theta(u)$ by kernel regression over a memory buffer of past $(t_\theta, z)$ values. The memory is a FIFO buffer with maximum size 4096. We normalize by the standard deviation of stored $t_\theta$ values, use a Gaussian kernel, and bandwidth 0.5. Stored indices are the detached values computed when they enter the buffer; they are not recomputed at later iterates. We clip $\hat g$ to $[10^{-4},1-10^{-4}]$ and mix with $0.05\cdot\sigma(t_\theta)$, where $\sigma$ is the logistic sigmoid, to avoid flat gradients. After training, we align the sign of $t_\theta$ using the sign of the empirical covariance between $t_\theta$ and $z$.

The 5\% sigmoid mixtures in Qwen3 PSPO and OSPO are large-scale optimization stabilizers; they are not used in the synthetic PSPO/OSPO experiments.

\subparagraph*{RSPO (PoP-DPO).}
We implement the pairs-of-pairs DPO variant by forming all pairwise sums within a minibatch:
\[
\ell = \frac{1}{B^2}\sum_{i,j} \log\left(1+\exp\left(-\left(s_i+s_j\right)\right)\right),
\quad
s_i = \log\frac{\pi_\theta(y_{i,\mathrm{pref}}\mid x_i)}{\pi_{\mathrm{ref}}(y_{i,\mathrm{pref}}\mid x_i)} - \log\frac{\pi_\theta(y_{i,\mathrm{other}}\mid x_i)}{\pi_{\mathrm{ref}}(y_{i,\mathrm{other}}\mid x_i)}.
\]

\subparagraph*{Evaluation and KL interpolation.}
For each trained policy and each $\beta$ on a grid, we evaluate on 64 prompts from the same shuffled UltraFeedback prefix used to construct the preference data, generating fresh responses from the aligned policy. We use the token-level geometric-mixture approximation to the KL-calibrated policy discussed in \cref{sec:realizing-calibrated-policy}. Concretely, at every generation step we mix token logits as
\[
\ell_\beta = (1 - s\beta^{-1})\ell_{\mathrm{ref}} + s\beta^{-1}\ell_\theta
=\ell_{\mathrm{ref}}+s\beta^{-1}(\ell_\theta-\ell_{\mathrm{ref}}),
\]
where $s=1$ except for OSPO, for which $s$ is the empirical sign alignment. Thus a negative OSPO sign uses $(1+\beta^{-1})\ell_{\mathrm{ref}}-\beta^{-1}\ell_\theta$, corresponding to negating $\ell_\theta-\ell_{\mathrm{ref}}$. Sampling uses temperature 0.7 and top-$p$ 0.95 with max 128 new tokens. We compute the average reward using the reward model and estimate $D_{\mathrm{KL}}(\pi_\beta\|\pi_{\mathrm{ref}})$ as the mean difference of log-probabilities between $\pi_\beta$ and $\pi_{\mathrm{ref}}$ over generated tokens.

For convergence curves at a target KL $\kappa$, we linearly interpolate the reward between the two neighboring $\beta$ values that bracket $\kappa$ on the KL axis. 
If $\kappa$ is outside the observed range, we use the nearest point.

\subparagraph*{Computer resources.}
The experiments from this subsection are run on a p4d.24xlarge AWS instance. 

\subparagraph*{Licenses.}
Qwen3-0.6B and Skywork-Reward-V2-Qwen3-1.7B are distributed under the Apache-2.0 license, and UltraFeedback is distributed under the MIT license, according to their Hugging Face model and dataset cards.

\end{document}